\documentclass{article}

\usepackage{amsmath, amsthm}
\usepackage{amsfonts, amssymb}
\usepackage{cancel}
\usepackage[svgnames]{xcolor}
\definecolor{mydarkblue}{rgb}{0,0.08,0.45}
\usepackage{etoolbox}
\usepackage{dsfont} 
\usepackage{fontawesome} 
\usepackage[para]{footmisc}
\usepackage{graphicx}
\usepackage{hyperref}
\hypersetup{
    breaklinks=true,
    colorlinks=true,
    linkcolor=mydarkblue,
    citecolor=mydarkblue,
    filecolor=mydarkblue,
    urlcolor=mydarkblue
}
\usepackage{float}
\usepackage{nameref,zref-xr}
\zxrsetup{toltxlabel}

\usepackage[utf8x]{inputenc}
\usepackage{multirow}
\usepackage{natbib}
\usepackage{pgffor}
\usepackage{subfigure}

\usepackage[shortlabels]{enumitem}

\usepackage{mathtools}

\newcommand{\footDPPMC}{\href{https://github.com/rbardenet/dppmc}{\textsf{github.com/rbardenet/dppmc}}}

\newcommand{\footGitHubDPPy}{\href{https://github.com/guilgautier/DPPy}{\textsf{github.com/guilgautier/DPPy}}}

\newcommand{\footReadTheDocs}{\href{https://dppy.readthedocs.io}{\textsf{dppy.readthedocs.io}}}

\def\twofig{0.49\textwidth}

\def\fourfig{0.23\textwidth}

\makeatletter
\let\c@corollary\relax
\let\c@definition\relax
\let\c@example\relax
\let\c@lemma\relax
\let\c@theorem\relax
\let\c@proposition\relax

\newtheorem{corollary}{Corollary}
\newtheorem{definition}{Definition}
\newtheorem{lemma}{Lemma}
\newtheorem{proposition}{Proposition}
\newtheorem{theorem}{Theorem}

\makeatother
\newcommand{\Iauth}[1]{\hat{I}^{\,\text{#1}}}
\newcommand{\DPP}{\operatorname{DPP}}
\newcommand{\weq}{\omega_{\text{eq}}}

\newcommand{\Vol}{\operatorname{volume}}

\def\beq{\begin{equation}}
\def\eeq{\end{equation}}
\newcommand{\bv}{\mathbf}
\def\bx{{\bv x}}

\makeatletter
\newcommand{\leqnomode}{\tagsleft@true\let\veqno\@@leqno}
\newcommand{\reqnomode}{\tagsleft@false\let\veqno\@@eqno}
\makeatother

\newcommand{\CommaBin}{\mathbin{\raisebox{0.5ex}{,}}}
\newcommand{\PeriodBin}{\mathbin{\raisebox{0.5ex}{.}}}

\let\originalleft\left
\let\originalright\right
\renewcommand{\left}{\mathopen{}\mathclose\bgroup\originalleft}
\renewcommand{\right}{\aftergroup\egroup\originalright}

\makeatletter
\@addtoreset{footnote}{page}
\makeatother

\makeatletter
\newcommand{\footnoteref}[1]{\protected@xdef\@thefnmark{\ref{#1}}\@footnotemark}
\makeatother

\newcommand*{\Appref}[1]{Appendix~\ref{#1}}

\newcommand*{\Figref}[1]{Figure~\ref{#1}}
\newcommand*{\Lemref}[1]{Lemma~\ref{#1}}

\newcommand*{\Secref}[1]{Section~\ref{#1}}

\newcommand*{\Thref}[1]{Theorem~\ref{#1}}

\makeatletter
\patchcmd{\NAT@test}{\else \NAT@nm}{\else \NAT@nmfmt{\NAT@nm}}{}{}
\DeclareRobustCommand\citepos
  {\begingroup
   \let\NAT@nmfmt\NAT@posfmt
   \NAT@swafalse\let\NAT@ctype\z@\NAT@partrue
   \@ifstar{\NAT@fulltrue\NAT@citetp}{\NAT@fullfalse\NAT@citetp}}
\let\NAT@orig@nmfmt\NAT@nmfmt
\def\NAT@posfmt#1{\NAT@orig@nmfmt{#1's}}
\makeatother

\foreach \x in {a,...,z}{
	\expandafter\xdef\csname bf\x
		\endcsname{
			\noexpand\ensuremath{\noexpand\mathbf{\x}}
		}
}

\foreach \x in {A,...,Z}{
	\expandafter\xdef\csname bf\x
		\endcsname{
			\noexpand\ensuremath{\noexpand\mathbf{\x}}
			}
	\expandafter\xdef\csname cal\x
		\endcsname{
			\noexpand\ensuremath{\noexpand\mathcal{\x}}
			}
	\expandafter\xdef\csname bb\x
		\endcsname{
			\noexpand\ensuremath{\noexpand\mathbb{\x}}
			}
}

\foreach \x/\y in {alpha/\alpha, beta/\beta, gamma/\gamma, delta/\delta, epsilon/\epsilon, zeta/\zeta, eta/\eta, theta/\theta, iota/\iota, kappa/\kappa, lambda/\lambda, mu/\mu, nu/\nu, xi/\xi, pi/\pi, rho/\rho, sigma/\sigma, tau/\tau, upsilon/\upsilon, phi/\phi, chi/\chi, psi/\psi, omega/\omega }{
	\expandafter\xdef\csname bf\x
		\endcsname{
			\noexpand\ensuremath{\noexpand\boldsymbol{\y}}
			}
}
\foreach \x/\y in {Gamma/\Gamma, Delta/\Delta, Theta/\Theta, Lambda/\Lambda, Xi/\Xi, Pi/\Pi, Sigma/\Sigma, Upsilon/\Upsilon, Phi/\Phi, Psi/\Psi, Omega/\Omega }{
  \expandafter\xdef\csname bf\x
    \endcsname{
      \noexpand\ensuremath{\noexpand\boldsymbol{\y}}
      }
}

\renewcommand{\hat}{\widehat}

\newcommand{\lsum}[2]{\sum\limits_{\substack{#1}}^{\substack{#2}}}
\newcommand{\lprod}[2]{\prod\limits_{\substack{#1}}^{\substack{#2}}}

\newcommand{\argmax}{\operatornamewithlimits{argmax}}

\newcommand{\lequal}[2]{
  \operatornamewithlimits{=}\limits_{\substack{#1}}^{\substack{#2}}
}

\newcommand{\dist}{\operatorname{distance}}

\newcommand*{\diff}{\mathop{}\!\mathrm{d}}

\newcommand{\lrb}[1]{\left[ #1 \right]}
\newcommand{\lrp}[1]{\left( #1 \right)}
\newcommand{\lrcb}[1]{\left\{ #1 \right\}}

\newcommand{\lrsp}[1]{\left\langle #1 \right\rangle}
\newcommand{\lrabs}[1]{\left| #1 \right|}
\newcommand{\lrnorm}[1]{\left\| #1 \right\|}

\makeatletter
\newcommand{\Proba}[1][\@nil]{%
  \def\tmp{#1}%
   \ifx\tmp\@nnil%
   		\bbP%
    \else%
      \bbP\lrb{#1}%
    \fi}
\newcommand{\Expe}[1][\@nil]{%
  \def\tmp{#1}%
   \ifx\tmp\@nnil%
   		\bbE%
    \else%
      \bbE\lrb{#1}%
    \fi}
\newcommand{\Var}[1][\@nil]{%
  \def\tmp{#1}%
   \ifx\tmp\@nnil%
   		\operatorname{\bbV{ar}}%
    \else%
      \operatorname{\bbV{ar}}\lrb{#1}%
    \fi}
\newcommand{\Cov}[1][\@nil]{%
  \def\tmp{#1}%
   \ifx\tmp\@nnil%
   		\operatorname{\bbC{ov}}%
    \else%
      \operatorname{\bbC{ov}}\lrb{#1}%
    \fi}
\makeatother

\makeatletter

\newcommand{\indic}{\mathds{1}}

\newcommand{\Ber}[1][\@nil]{%
  \def\tmp{#1}%
   \ifx\tmp\@nnil%
   		\calB%
    \else%
      \calB\!\lrp{#1}%
    \fi}

\newcommand{\Geom}[1][\@nil]{%
  \def\tmp{#1}%
   \ifx\tmp\@nnil%
   		\calG%
    \else%
      \calG\!\lrp{#1}%
    \fi}
\newcommand{\Poiss}[1][\@nil]{%
  \def\tmp{#1}%
   \ifx\tmp\@nnil%
   		\calP%
    \else%
      \calP\!\lrp{#1}%
    \fi}

\newcommand{\Unif}[1][\@nil]{%
  \def\tmp{#1}%
   \ifx\tmp\@nnil%
   		\calU%
    \else%
      \calU_{#1}%
    \fi}
\newcommand{\Beta}[1][\@nil]{%
  \def\tmp{#1}%
   \ifx\tmp\@nnil%
   		\operatorname{Beta}%
    \else%
      \operatorname{Beta}\lrp{#1}%
    \fi}

\newcommand{\Exp}[1][\@nil]{%
  \def\tmp{#1}%
   \ifx\tmp\@nnil%
   		\calE%
    \else%
      \calE\!\lrp{#1}%
    \fi}
\newcommand{\Gama}[1][\@nil]{%
  \def\tmp{#1}%
   \ifx\tmp\@nnil%
   		\Gamma%
    \else%
      \Gamma\!\lrp{#1}%
    \fi}

\newcommand{\Cauchy}[1][\@nil]{%
  \def\tmp{#1}%
   \ifx\tmp\@nnil%
   		\operatorname{Cauchy}%
    \else%
      \operatorname{Cauchy}\!\lrp{#1}%
    \fi}
\newcommand{\Gauss}[1][\@nil]{%
  \def\tmp{#1}%
   \ifx\tmp\@nnil%
   		\calN%
    \else%
      \calN\!\lrp{#1}%
    \fi}
\newcommand{\Dir}[1][\@nil]{%
  \def\tmp{#1}%
   \ifx\tmp\@nnil%
   		\operatorname{Dir}%
    \else%
      \operatorname{Dir}\!\lrp{#1}%
    \fi}

\makeatother

\renewcommand{\top}{\mathsf{\scriptscriptstyle T}}

\makeatletter
\newcommand{\Span}[1][\@nil]{%
  \def\tmp{#1}%
   \ifx\tmp\@nnil%
   		\operatorname{span}%
    \else%
      \operatorname{span}\!\lrcb{#1}%
    \fi}
\newcommand{\Tr}[1][\@nil]{%
  \def\tmp{#1}%
   \ifx\tmp\@nnil%
   		\operatorname{Tr}%
    \else%
      \operatorname{Tr}\!\lrb{#1}%
    \fi}
\newcommand{\rank}[1][\@nil]{%
  \def\tmp{#1}%
   \ifx\tmp\@nnil%
   		\operatorname{rank}%
    \else%
      \operatorname{rank}\!\lrp{#1}%
    \fi}
\newcommand{\Sp}[1][\@nil]{%
  \def\tmp{#1}%
   \ifx\tmp\@nnil%
   		\operatorname{Sp}%
    \else%
      \operatorname{Sp}\!\lrp{#1}%
    \fi}
\newcommand{\diag}[1][\@nil]{%
  \def\tmp{#1}%
   \ifx\tmp\@nnil%
   		\operatorname{diag}%
    \else%
      \operatorname{diag}\!\lrp{#1}%
    \fi}
\makeatother

\usepackage[final]{neurips_2019}

\usepackage[utf8x]{inputenc} 
\usepackage[T1]{fontenc}    
\usepackage{hyperref}       
\usepackage{url}            
\usepackage{booktabs}       
\usepackage{amsfonts}       
\usepackage{nicefrac}       
\usepackage{microtype}      

\title{On two ways to use determinantal point processes\\ for Monte Carlo integration}

\author{%
  Guillaume Gautier$^{\dagger*}$ \\
  \texttt{g.gautier@inria.fr}
  \And
  Rémi Bardenet$^\dagger$ \\
  \texttt{remi.bardenet@gmail.com}
  \And
  Michal Valko$^{\ddag*\dagger}$ \\
  \texttt{valkom@deepmind.com}
  \AND
  \text{ }\\[-2.5em]
  $^\dagger$Univ.\,Lille, CNRS, Centrale Lille, UMR 9189\,--\,CRIStAL,  59651 Villeneuve d'Ascq, France\\
  $^*$Inria Lille-Nord Europe, 40 avenue Halley 59650 Villeneuve d'Ascq, France\\
  $^\ddag$DeepMind Paris, 14 Rue de Londres, 75009  Paris, France
  }

\begin{document}

\maketitle
\vspace{-2em}
\begin{abstract}
    When approximating an integral by a weighted sum of function evaluations, determinantal point processes (DPPs) provide a way to enforce repulsion between the evaluation points.
    This negative dependence is encoded by a kernel.
    Fifteen years before the discovery of DPPs, \citet[EZ,][]{ErZo60} had the intuition of sampling a DPP and solving a linear system to compute an unbiased Monte Carlo estimator of the integral.
    In the absence of DPP machinery to derive an efficient sampler and analyze their estimator, the idea of Monte Carlo integration with DPPs was stored in the cellar of numerical integration.
    Recently, \citet[BH,][]{BaHa19} came up with a more natural estimator with a fast central limit theorem (CLT).
    In this paper, we first take the EZ estimator out of the cellar, and analyze it using modern arguments.
    Second, we provide an efficient implementation\footnote{\label{fn:dppy}\footGitHubDPPy} to sample exactly a particular multidimensional DPP called \emph{multivariate Jacobi ensemble}.
    The latter satisfies the assumptions of the aforementioned CLT.
    Third, our new implementation lets us investigate the behavior of the two unbiased Monte Carlo estimators in yet unexplored regimes.
    We demonstrate experimentally good properties when the kernel is adapted to basis of functions in which the integrand is sparse or has fast-decaying coefficients.
    If such a basis and the level of sparsity are known (e.g., we integrate a linear combination of kernel eigenfunctions), the EZ estimator can be the right choice, but otherwise it can display an erratic behavior.
\end{abstract}
\vspace{-1em}
\section{Introduction}


    Numerical integration is a core task of many machine learning applications, including most Bayesian methods \citep{Rob07}.
    Both deterministic \citep{DaRa84,DiPi10} and random \citep{RoCa04} algorithms have been proposed; see also \citep{EvSc00} for a survey.
    All methods require evaluating the integrand at  carefully chosen points, called \emph{quadrature nodes}, and combining these evaluations
    to minimize the approximation error.

    Recently, a stream of work has made use of prior knowledge on the smoothness of the integrand using kernels.
    \citet{OaGiCh17} and \citet{LiLe17} used kernel-based control variates, splitting the computational budget into regressing the integrand and integrating the residual.
    \citet{Bac17} looked for the best way to sample i.i.d.\,nodes and combine the resulting evaluations.
    Finally, Bayesian quadrature \citep{Hag91,HuDu12,BOGO15}, herding \citep{ChWeSm10,BaLaOb12}, or the biased importance sampling estimate of \citet{DePo16} all favor \emph{dissimilar} nodes, where dissimilarity is measured by a kernel.
    Our work falls in this last cluster.

    We build on the particular approach of \citet{BaHa19} for Monte Carlo integration based on projection \emph{determinantal point processes} \citep[DPPs,][]{HKPV06,KuTa12}.
    DPPs are a repulsive distribution over configurations of points, where repulsion is again parametrized by a kernel.
    In a sense, DPPs are the kernel machines of point processes.

    Fifteen years before \citet{Mac75} even formalized DPPs, \citet[EZ,][]{ErZo60} had the intuition to use a determinantal structure to sample quadrature nodes, followed by solving a linear system to aggregate the evaluations of the integrand into an unbiased estimator.
    This linear system yields a simple and interpretable characterization of the variance of their estimator.
    \citeauthor{ErZo60}'s result did not diffuse much\footnote{Many thanks to
    Mathieu Gerber of Univ.\,Bristol, UK, for digging up this result from his
    human memory.}
    in the Monte Carlo community, partly because the mathematical and computational machinery to analyze and implement it was not available.
    Seemingly unaware of this previous work, \citet{BaHa19} came up with a more natural estimator of the integral of interest, and they could build upon the thorough study of DPPs in random matrix theory \citep{Joh06} to obtain a fast central limit theorem (CLT).
    Since then, DPPs with stationary kernels have also been used by \citet{MaCoAm19} for Monte Carlo integration.
    In any case, these DPP-based Monte Carlo estimators crucially depend on efficient sampling procedures for the corresponding (potentially multidimensional) DPP.

    \paragraph{Our contributions.}
    First, we reveal the close link between DPPs and the approach of \citet{ErZo60}.
    Second, we provide a simple proof of their result and survey the properties of the EZ estimator with modern arguments.
    In particular, when the integrand is a linear combination of the eigenfunctions of the kernel of the underlying DPP, the corresponding Fourier-like coefficients can be estimated with zero variance.
    In other words, one sample of the corresponding DPP yields perfect interpolation of the underlying integrand, by solving a linear system.
    Third, we propose an efficient Python implementation for exact sampling of a particular DPP, called \emph{multivariate Jacobi ensemble}.\setcounter{footnote}{1}
    The code\footnoteref{fn:dppy} is available in the DPPy toolbox of \citet{GPBV19}.
    This implementation allows to numerically investigate the behavior of the two Monte Carlo estimators derived by \citet{BaHa19} and \citet{ErZo60}, in regimes yet unexplored for any of the two.
    Fourth, important theoretical properties of both estimators, like the CLT of \citep{BaHa19}, are technically involved.
    A CLT for EZ promises to be even more difficult to establish.
    The current empirical investigation provides a motivation and guidelines for more theoretical work.
    Our point is not to compare DPP-based Monte Carlo estimators to the wide choice of numerical integration algorithms, but to get a fine understanding of their properties so as to fine-tune their design and guide theoretical developments.

\section{Quadrature, DPPs, and the multivariate Jacobi ensemble}
\label{sec:background}


    In this section, we quickly survey classical quadrature rules.
    Then, we define DPPs and give the key properties that make them useful for Monte Carlo integration.
    Finally, among so-called \emph{projection} DPPs, we introduce the multivariate Jacobi ensemble used by \citet{BaHa19} to generate quadrature nodes, and on which we base our experimental work.

    \subsection{Standard quadrature} 
    \label{sub:quadrature_rules}

        Following \citet[Section 2.1]{BOGO15}, let $\mu(\diff x) = \omega(x) \diff x$ be a positive Borel measure on $\bbX \subset \bbR^d$ with finite mass and density $\omega$ w.r.t.\,the Lebesgue measure.
        This paper aims to compute integrals of the form
        $\int f(x) \mu(\diff x)$ for some test function $f: \bbX \to \bbR$.
        A quadrature rule approximates such integrals as a weighted sum of evaluations of $f$ at some \emph{nodes} $\lrcb{x_{1}, \dots, x_{N}}\subset \bbX,$
        \begin{equation}
            \label{eq:quadrature}
            \int f(x) \mu(\diff x)
                \approx \sum_{n=1}^{N} \omega_n f(x_n),
        \end{equation}
        where the weights $\omega_n \triangleq \omega_n(x_1, \dots, x_N)$ do not need to be non-negative nor sum to one.

        Among the many quadrature designs mentioned in the introduction \citep[Section~5]{EvSc00}, we pay special attention to the textbook example of the (deterministic) Gauss-Jacobi rule.
        This scheme applies to dimension $d=1$, for $\bbX\triangleq[-1, 1]$ and $\omega(x) \triangleq(1-x)^a(1+x)^b$ with $a, b > -1$.
        In this case, the nodes $\lrcb{x_{1}, \dots, x_{N}}$ are taken to be the zeros of $p_N$, the orthonormal Jacobi polynomial of degree $N$, and the weights $\omega_n \triangleq 1/K(x_n, x_n)$ with $K(x, x) \triangleq \sum_{k=0}^{N-1} p_k(x)^2$.
        In particular, this specific quadrature rule allows to perfectly integrate polynomials up to degree $2N-1$ \citep[Section 2.7]{DaRa84}.
        In a sense, the DPPs of \citet{BaHa19} are a random, multivariate generalization of Gauss-Jacobi quadrature, as we shall see in \Secref{sub:BH_estimator}.

        Monte Carlo integration can be defined as random choices of nodes in \eqref{eq:quadrature}.
        Importance sampling, for instance, corresponds to i.i.d.\,nodes, while Markov chain Monte Carlo corresponds to nodes drawn from a carefully chosen Markov chain; see, e.g., \citet{RoCa04} for more details.
        Finally, quasi-Monte Carlo \citep[QMC,][]{DiPi10} applies to $\mu$ uniform over a compact subset of $\mathbb{R}^d$, and constructs deterministic nodes that spread uniformly, as measured by their \emph{discrepancy}.


    \subsection{Projection DPPs}
    \label{sub:projection_DPPs}

        DPPs can be understood as a parametric class of point processes, specified by a base measure~$\mu$ and a kernel $K:\bbX\times\bbX \to \bbC$.
        The latter is commonly assumed to be Hermitian and trace-class.
        For the resulting process to be well defined, it is necessary and sufficient that the kernel $K$ is positive semi-definite with eigenvalues in $[0,1]$, see, e.g., \citet[Theorem 3]{Sos00}.
        When the eigenvalues further belong to $\lrcb{0,1}$, we speak of a \emph{projection} kernel and a \emph{projection} DPP.
        One practical feature of projection DPPs is that they almost surely produce samples with fixed cardinality, equal to the rank $N$ of the kernel.
        More generally, they are the building blocks of DPPs.
        Indeed, under general assumptions, all DPPs are mixtures of projection DPPs \citep[Theorem 7]{HKPV06}.
        Hereafter, unless specifically stated, $K$ is assumed to be a real-valued, symmetric, projection kernel.

        One way to define a projection DPP with $N$ points is to take $N$ functions $\phi_0,\dots,\phi_{N-1}$ orthonormal w.r.t.\,$\mu$, i.e.,
            $\lrsp{\phi_k, \phi_{\ell}}
                \triangleq \int \phi_k(x) \phi_{\ell}(x) \mu(\diff x)
                = \delta_{k\ell}$,
        and consider the kernel $K_N$ associated to the orthogonal projector onto $\calH_N \triangleq \Span \lrcb{ \phi_k, \; 0\leq k\leq N-1}$, i.e.,
        \begin{equation}
        \label{eq:kernel_projection_DPP}
            K_N(x,y)
                \triangleq
                    \sum_{k=0}^{N-1}\phi_k(x)\phi_k(y).
        \end{equation}
        We say that the set $\lrcb{\bfx_1,\dots,\bfx_N}\subset \bbX$ is drawn from the projection DPP with base measure $\mu$ and kernel $K_N$, denoted by $\lrcb{\bfx_1, \dots, \bfx_N} \sim \DPP(\mu,K_N)$, when $\lrp{\bfx_1,\dots,\bfx_N}$ has joint distribution
        \begin{equation}
        \label{eq:joint_distribution_projection_dpp}
            \frac1{N!}
                \det\lrp{K_N(x_p,x_n)}_{p,n=1}^N
                \, \mu^{\otimes N}(\diff x).
        \end{equation}
        $\DPP(\mu, K_N)$ indeed defines a probability measure over sets since \eqref{eq:joint_distribution_projection_dpp} is invariant by permutation and the orthonormality of the $\phi_k$s yields the normalization.
        See also \Appref{sub:normalization_cauchy_binet} for more details on the construction of projection DPPs from sets of linearly independent functions.

        The repulsion of projection DPPs may be understood geometrically by considering the Gram formulation of the kernel \eqref{eq:kernel_projection_DPP}, namely
        \begin{equation}
        \label{eq:kernel_projection_DPP_Gram_formulation}
            K_N(x,y)
                = \Phi(x)^{\top} \Phi(y),
            \quad\text{where}\quad
            \Phi(x) \triangleq \lrp{\phi_{0}(x),\dots,\phi_{N-1}(x)}^{\top}.
        \end{equation}
        This allows to rewrite the joint distribution \eqref{eq:joint_distribution_projection_dpp} as
        \begin{equation}
        \label{eq:joint_distribution_projection_dpp_geometric}
        \!
        \frac1{N!}
            \underbrace{\det\bfPhi(x_{1:N})\bfPhi(x_{1:N})^{\top}}_{=(\det\bfPhi(x_{1:N}))^2}
            \, \mu^{\otimes N}(\diff x),
        \quad\text{where}\quad
            \bfPhi(x_{1:N})
            \triangleq
            \begin{pmatrix}
                \phi_0(x_1) & \dots  & \phi_{N-1}(x_1)\\
                \vdots     &        & \vdots\\
                \phi_0(x_N) & \dots  & \phi_{N-1}(x_N)
            \end{pmatrix}.
        \end{equation}
        Thus, the larger the determinant of the \emph{feature matrix} $\bfPhi(x_{1:N})$, i.e., the larger the volume of the parallelotope spanned by the \emph{feature vectors} $\Phi(x_{1}), \dots, \Phi(x_{N})$, the more likely $x_1, \dots, x_N$ co-occur.

    \subsection{The multivariate Jacobi ensemble}
    \label{sub:multivariate_jacobi_ensemble}

        In this part, we specify a projection kernel.
        We follow \citet{BaHa19} and take its eigenfunctions to be multivariate orthonormal polynomials.
        In dimension $d=1$, letting $\lrp{\phi_k}_{k\geq0}$ in \eqref{eq:kernel_projection_DPP} be the orthonormal polynomials w.r.t.\,$\mu$ results in a projection DPP called an \emph{orthogonal polynomial ensemble} \citep[OPE,][]{Kon05}.
        When $d>1$, orthonormal polynomials can still be uniquely defined by applying the Gram-Schmidt procedure to a set of monomials, provided the base measure is not pathological.
        However, there is no natural order on multivariate monomials: an ordering $\mathfrak{b}:\bbN^d\rightarrow \bbN$ must be picked before we apply Gram-Schmidt to the monomials in $L^2(\mu)$.
        We follow \citet[Section 2.1.3]{BaHa19} and consider multi-indices $k\triangleq(k^1,\dots,k^d)\in \bbN^d$ ordered by their maximum degree $\max_i k^i$, and for constant maximum degree, by the usual lexicographic order.
        We still denote the corresponding multivariate orthonormal polynomials by $\lrp{\phi_{k}}_{k\in \bbN^d}$.

        By multivariate OPE we mean the projection DPP with base measure
        $\mu(\diff x)\triangleq\omega(x)\diff x$
        and orthogonal projection kernel
        $K_N(x,y)
            \triangleq \sum_{\mathfrak{b}(k)=0}^{N-1}
                \phi_k(x)\phi_k(y)$.
        When the base measure is separable,~i.e.,
            $\omega(x) = \omega^1(x^1) \times \cdots \times \omega^d(x^d)$,
        multivariate orthonormal polynomials are products of univariate ones,
        and the kernel \eqref{eq:kernel_projection_DPP} reads
        \begin{equation}
        \label{eq:kernel_multivariate_separable_OPE}
            K_N(x,y)
            = \sum_{\mathfrak{b}(k)=0}^{N-1}
                \prod_{i=1}^d
                    \phi_{k^i}^i(x^i) \phi_{k^i}^i(y^i),
        \end{equation}
        where $(\phi_\ell^i)_{\ell\geq 0}$ are the orthonormal polynomials w.r.t.\,$\omega^i(z) \diff z$.
        For $\bbX=[-1,1]^d$ and $\omega^i(z) = (1-z)^{a^i}(1+z)^{b^i}$, with $a^i, b^i > -1$, the resulting DPP is called a \emph{multivariate Jacobi ensemble}.



\section{Monte Carlo integration with projection DPPs}
\label{sec:MCI_with_projection_DPPs}

    Our goal is to design random quadrature rules \eqref{eq:quadrature} on $\bbX \triangleq [-1, 1]^d$ with desirable properties.
    We focus on computing $\int f(x) \mu(\diff x)$ with the two unbiased DPP-based Monte Carlo estimators of \citet[BH,][]{BaHa19} and \citet[EZ,][]{ErZo60}.
    We start by presenting the natural BH estimator which, when associated to the multivariate Jacobi ensemble, comes with a CLT with a faster rate than classical Monte Carlo.
    Then, we survey the properties of the less obvious EZ estimator.
    Using a generalization of the Cauchy-Binet formula we provide a slight improvement of the key result of EZ.
    Despite the lack of result illustrating a fast convergence rate, the EZ estimator has a practical and interpretable variance.
    In particular, this estimator turns a single DPP sample into a perfect integrator as well as a perfect interpolator of functions that are linear combinations of eigenfunctions of the associated kernel.
    Finally, we detail our exact sampling procedure for multivariate Jacobi ensemble, which allows to exploit the best of both the BH and EZ estimators.

    \subsection{A natural estimator}
    \label{sub:BH_estimator}

        For $f\in L^1(\mu)$, \citet{BaHa19} consider
        \begin{equation}
        \label{eq:BH_estimator}
            \Iauth{BH}_N(f)
                \triangleq
                    \sum_{n=1}^{N}
                        \frac{f(\bfx_n)}{K_N(\bfx_n,\bfx_n)}\CommaBin
        \end{equation}
        as an unbiased estimator of $\int f(x) \mu(\diff x)$, with variance (see, e.g., \citealp[Section 2.1]{LaMoRu15})
        \begin{equation}
        \label{eq:var_BH}
            \Var[\Iauth{BH}_N(f)]
            =
            \frac12
                \int \lrp{\frac{f(x)}{K_N(x,x)} - \frac{f(y)}{K_N(y,y)}}^2
                        {K_N(x, y)}^2
                            \mu(\diff x) \mu(\diff y),
        \end{equation}
        which clearly captures a notion of smoothness of $f$ w.r.t.\,$K_N$ but its interpretation is not obvious.

        For $\bbX=[-1, 1]^d$, the interest in multivariate Jacobi ensemble among DPPs comes from the fact that \eqref{eq:BH_estimator} can be understood as a (randomized) multivariate counterpart of the Gauss-Jacobi quadrature introduced in \Secref{sub:quadrature_rules}.
        Moreover, for $f$ essentially $\calC^1$, \citet[Theorem 2.7]{BaHa19} proved a CLT with faster-than-classical-Monte-Carlo decay,
        \begin{equation}
        \label{eq:BH_CLT}
            \sqrt{N^{1+1/d}}
            \lrp{\Iauth{BH}_N(f) - \int f(x)\mu(\diff x)}
                \xrightarrow[N\to\infty]{\text{law}}
                    \Gauss[0,\Omega_{f,\omega}^2],
        \end{equation}
        with $\Omega_{f,\omega}^2
                \triangleq \frac{1}{2}
                    \sum_{k \in \bbN^d}
                        (k^1+\cdots+k^d)
                        \calF_{\frac{f \omega}{\weq}}
                        (k)^2$,
        where $\calF_g$ denotes the Fourier transform of $g$, and
        $\weq(x) \triangleq 1/\prod_{i=1}^d \pi\sqrt{1-(x^i)^2}$.
        In the fast CLT \eqref{eq:BH_CLT}, the asymptotic variance is governed by the smoothness of $f$ since $\Omega_{f,\omega}$ is a measure of the decay of the Fourier coefficients of the integrand.

    \subsection{The Ermakov-Zolotukhin estimator}
    \label{sub:EZ_estimator}

        We start by stating the main finding of \citet{ErZo60}, see also \citet[Section 6.4.3]{EvSc00} and references therein.
        To the best of our knowledge, we are the first to make the connection with projection DPPs, as defined in \Secref{sub:projection_DPPs}.
        This allows us to give a slight improvement and provide a simpler proof of the original result, based on a generalization of the Cauchy-Binet formula \citep{Joh06}.
        Finally, we apply \citepos{ErZo60} technique to build an unbiased estimator of $\int f(x) \mu(\diff x)$, which comes with a practical and interpretable variance.
        \begin{theorem}
            \label{th:ermakov_zolotukhin_estimators}
            Consider $f\in L^2(\mu)$ and $N$ functions $\phi_0, \dots, \phi_{N-1}\in L^2(\mu)$ orthonormal w.r.t.\,$\mu$.
            Let $\{\bfx_1,\ldots,\bfx_N\} \sim \DPP(\mu, K_N)$, with $K_N(x,y)=\sum_{k=0}^{N-1} \phi_k(x)\phi_k(y)$. Consider the linear system
            \begin{equation}
                \label{eq:EZ_linear_system}
                \begin{pmatrix}
                    \phi_0(\bfx_1) & \dots  & \phi_{N-1}(\bfx_1)\\
                        \vdots     &        & \vdots\\
                    \phi_0(\bfx_N) & \dots  & \phi_{N-1}(\bfx_N)
                \end{pmatrix}
                \begin{pmatrix}
                    y_1\\
                    \vdots\\
                    y_N
                \end{pmatrix}
                =
                \begin{pmatrix}
                    f(\bfx_1)\\
                    \vdots\\
                    f(\bfx_N)
                \end{pmatrix}.
            \end{equation}
            Then, the solution of \eqref{eq:EZ_linear_system} is unique, $\mu$-almost surely, with coordinates
            $
                y_k
                =
                \frac{\det \bfPhi_{\phi_{k-1}, f}(\bfx_{1:N})}
                     {\det \bfPhi(\bfx_{1:N})}
                \CommaBin
            $
            where $\bfPhi_{\phi_{k-1}, f}(\bfx_{1:N})$ is the matrix obtained by replacing the $k$-th column of $\bfPhi(\bfx_{1:N})$ by $f(\bfx_{1:N})$.
            Moreover, for all $1\leq k\leq N$, the coordinate $y_k$ of the solution vector satisfies
            \begin{equation}
                \label{eq:expe_var_EZ}
                \Expe[y_k]
                    = \lrsp{f, \phi_{k-1}},
                \quad\text{and}\quad
                \Var[y_k]
                    = \lrnorm{f}^2
                       - \sum_{\ell=0}^{N-1}
                         \lrsp{f, \phi_{\ell}}^2.
            \end{equation}
            We improved the original result by showing that $\Cov[y_j, y_k]= 0$, for all $1\leq j\neq k\leq N$.
        \end{theorem}
        Before we provide the proof, also detailed in \Appref{sub:app_proof_ErZo_th}, several remarks are in order.
        We start by considering a function $f\triangleq\sum_{k=0}^{M-1} \lrsp{f, \phi_k} \phi_k$, $1\leq M \leq \infty$, where $(\phi_k)_{k\geq 0}$ forms an orthonormal basis of $L^2(\mu)$, e.g., the Fourier basis or wavelet bases \citep{MaPe09}.
        Next, we build the orthogonal projection kernel $K_N$ onto $\calH_N \triangleq \Span \lrcb{\phi_0, \dots, \phi_{N-1}}$ as in \eqref{eq:kernel_projection_DPP}.
        Then, \Thref{th:ermakov_zolotukhin_estimators} shows that solving \eqref{eq:EZ_linear_system}, with points $\lrcb{\bfx_1, \dots, \bfx_N}\sim\DPP(\mu, K_N)$, provides unbiased estimates of the $N$ Fourier-like coefficients $\lrp{\lrsp{f, \phi_k}}_{k=0}^{N-1}$.
        Remarkably, these estimates are uncorrelated and have the same variance \eqref{eq:expe_var_EZ} equal to the residual $\sum_{k=N}^{\infty} \lrsp{f, \phi_k}^2$.
        The faster the decay of the coefficients, the smaller the variance.
        In particular, for $M\leq N$, i.e., $f\in \mathcal{H}_N$, the estimators have zero variance.
        Put differently, $f$ can be reconstructed perfectly from only one sample of $\DPP(\mu, K_N)$.
        \begin{proof}
            First, the joint distribution \eqref{eq:joint_distribution_projection_dpp_geometric} of $\lrp{\bfx_{1}, \dots, \bfx_{N}}$ is proportional to $\lrp{\det \bfPhi(x_{1:N})}^2 \mu^{\otimes N}(x)$.
            Thus, the matrix $\bfPhi(\bfx_{1:N})$ defining the linear system \eqref{eq:EZ_linear_system} is invertible, $\mu$-almost surely, and the expression of the coordinates follows from Cramer's rule.
            Then, we treat the case $k=1$, the others follow the same lines.
            The proof relies on the orthonormality of the $\phi_k$s and a generalization of the Cauchy-Binet formula \eqref{eq:app_cauchy_binet}, cf.\,\Lemref{lem:app_cauchy_binet}.
            The expectation in \eqref{eq:expe_var_EZ} reads
            \begin{align}
                \!\!\Expe[\frac{\det \bfPhi_{\phi_0, f}(\bfx_{1:N})}
                            {\det \bfPhi(\bfx_{1:N})}]
                &\lequal{}{\eqref{eq:joint_distribution_projection_dpp_geometric}}
                \frac{1}{N!}
                    \int
                    \det \bfPhi_{\phi_0, f} (x_{1:N})
                    \det \bfPhi (x_{1:N})
                        \, \mu^{\otimes N}( \diff x) \nonumber\\
               & \lequal{}{\eqref{eq:app_cauchy_binet}}
                \det
                    \begin{psmallmatrix}
                        \lrsp{f, \phi_0}
                            & \lrp{\lrsp{f, \phi_{\ell}}}_{\ell=1}^{N-1}
                            \\
                        0_{N-1, 1}
                            & I_{N-1}
                    \end{psmallmatrix}
                = \lrsp{f, \phi_0}.
                \label{eq:proof_1st_moment}
            \end{align}
            Similarly, the second moment reads
            \begin{align}
                \!\!\!\!
                \Expe[
                    \lrp{
                     \frac{\det \bfPhi_{\phi_0, f}(\bfx_{1:N})}
                      {\det \bfPhi(\bfx_{1:N})}}^2
                    ]
                &\lequal{}{\eqref{eq:joint_distribution_projection_dpp_geometric}}
                    \frac{1}{N!}
                    \int
                        \det \bfPhi_{\phi_0, f}(x_{1:N})
                        \det \bfPhi_{\phi_0, f}(x_{1:N})
                        \, \mu^{\otimes N}( \diff x)
                        \nonumber\\
             &\lequal{}{\eqref{eq:app_cauchy_binet}}
                \det
                \begin{psmallmatrix}
                \lrnorm{f}^2
                & \lrp{ \lrsp{f, \phi_{\ell}} }_{\ell=1}^{N-1}
                \\
                \lrp{ \lrsp{f, \phi_k} }_{k=1}^{N-1}
                & I_{N-1}
                \end{psmallmatrix}
                = \lrnorm{f}^2
                    - \sum_{k=1}^{N-1}
                        \lrsp{f, \phi_{k}}^2.
                \!\!\!
                \label{eq:proof_2nd_moment}
            \end{align}
            Finally, the variance in \eqref{eq:expe_var_EZ} $=$ \eqref{eq:proof_2nd_moment} - \eqref{eq:proof_1st_moment}$^2$.
            The covariance is treated in \Appref{sub:app_proof_ErZo_th}.
        \end{proof}

        In the setting of the multivariate Jacobi ensemble described in \Secref{sub:multivariate_jacobi_ensemble}, the first
        orthonormal polynomial $\phi_0$  is constant, equal to $\mu\big(\lrb{-1, 1}^d\big)^{-1/2}$.
        Hence, a direct application of \Thref{th:ermakov_zolotukhin_estimators} yields
        \begin{equation}
        \label{eq:EZ_estimator}
            \Iauth{EZ}_N(f)
                \triangleq \frac{y_1}{\phi_0}
                =
                    \mu\big(\lrb{-1, 1}^d\big)^{1/2}
                    ~
                    \frac{\det \bfPhi_{\phi_0, f}(\bfx_{1:N})}
                         {\det \bfPhi(\bfx_{1:N})}\CommaBin
        \end{equation}
        as an unbiased estimator of $\int_{\lrb{-1,1}^d} f(x) \mu(\diff x)$, see \Appref{sub:app_EZ_as_quadrature}.
        We also show that \eqref{eq:EZ_estimator} can be viewed as a quadrature rule \eqref{eq:quadrature} with weights summing to $\mu(\lrb{-1, 1}^d)$.
        Unlike the variance of $\Iauth{BH}_N(f)$ in \eqref{eq:var_BH}, the variance of $\Iauth{EZ}_N(f)$ clearly reflects the accuracy of the approximation of $f$ by its projection onto $\calH_N$.
        In particular, it allows to integrate and interpolate polynomials up to ``degree'' $\mathfrak{b}^{-1}(N-1)$, perfectly.
        Nonetheless, its limiting theoretical properties, like a CLT, look hard to establish.
        In particular, the dependence of each quadrature weight on all quadrature nodes makes the estimator a peculiar object that doesn't fit the assumptions of traditional CLTs for DPPs \citep{Sos00}.


    \subsection{How to sample from projection DPPs and the multivariate Jacobi ensemble} 
    \label{sub:sampling}

        To perform Monte Carlo integration with DPPs, it is crucial to sample the points and evaluate the weights efficiently.
        However, sampling from continuous DPPs remains a challenge.
        In this part, we review briefly the main technique for projection DPP sampling before we develop our method to generate samples from the multivariate Jacobi ensemble.\setcounter{footnote}{2}
        The code\footnoteref{fn:dppy} is available in the DPPy toolbox of \citet{GPBV19}, the associated documentation\footnote{\label{fn:readthedocs}\footReadTheDocs} contains a lot more details on DPP sampling.

        In both finite and continuous cases, except for some specific instances, exact DPP sampling usually requires the spectral decomposition of the underlying kernel \citep[Section 2.4]{LaMoRu15}.
        However, for projection DPPs, prior knowledge of the eigenfunctions is not necessary, only an oracle to evaluate the kernel is required.
        Next, we describe the generic exact sampler of \citet[Algorithm 18]{HKPV06}.
        It is based on the chain rule and valid exclusively for projection DPPs.

        For simplicity, consider a projection $\DPP(\mu, K_N)$ with $\mu(\diff x) = \omega(x) \diff x$ and $K_N$ as in \eqref{eq:kernel_projection_DPP}.
        This DPP has exactly $N$ points, $\mu$-almost surely \citep[Lemma 17]{HKPV06}.
        To get a valid sample $\lrcb{\bfx_{1}, \dots, \bfx_{N}}$, it is enough to apply the chain rule to sample $(\bfx_1,\dots,\bfx_N)$ and forget the order the points were selected.
        The chain rule scheme can be derived from two different perspectives.
        Either using Schur complements to express the determinant in the joint distribution \eqref{eq:joint_distribution_projection_dpp},
        \begin{equation}
        \label{eq:chain_rule_schur}
        \frac{K_N(x_1, x_1)}{N}
            \omega(x_1) \diff x_1
            \prod_{n=2}^{N}
                \frac{K_N(x_n,x_n)
                -   {\bfK}_{n-1}(x_n)^{\top}
                    {\bfK}_{n-1}^{-1}
                    {\bfK}_{n-1}(x_n)}{N-(n-1)}
                \omega(x_n) \diff x_n,
        \end{equation}
        where
        $\bfK_{n-1}(\cdot)
          = \lrp{K_N(x_1,\cdot), \dots, K_N(x_{n-1},\cdot)}^{\top}$, and
        $\bfK_{n-1}
          = \lrp{K_N(x_p,x_q)}_{p,q=1}^{n-1}$.
        Or geometrically using the base$\times$height formula to express the squared volume in the joint distribution \eqref{eq:joint_distribution_projection_dpp_geometric},
        \begin{equation}
        \label{eq:chain_rule_geometric}
        \frac{\|\Phi(x_1)\|^2}{N}
                \omega(x_1) \diff x_1
                \prod_{n=2}^{N}
                    \frac{\dist^2\!\big(\Phi(x_n), \Span\lrcb{\Phi(x_p)}_{p=1}^{n-1}\big)}
                    {N-(n-1)}
                \omega(x_n) \diff x_n.
        \end{equation}
        Note that the numerators in \eqref{eq:chain_rule_schur} correspond to the incremental posterior variances of a noise-free Gaussian process model with kernel $K_N$ \citep{RaWi06}, giving yet another intuition for repulsion.
        When using the chain rule, the practical challenge is twofold: find efficient ways to
        (i) evaluate the conditional densities,
        (ii) sample exactly from the conditionals.

        In this work, we take $\bbX=[-1, 1]^d$ and focus on sampling the multivariate Jacobi ensemble with parameters $\lrabs{a^i}, \lrabs{b^i} \leq 1/2$, cf.\,\Secref{sub:multivariate_jacobi_ensemble}.\setcounter{footnote}{3}
        We remodeled the original implementation\footnote{\footDPPMC} of the multivariate Jacobi ensemble sampler accompanying the work of \citet[BH,][]{BaHa19} in a more Python\emph{ic} way.
        In particular, we address the previous challenges in the following way:
        \begin{enumerate}[(i), wide=0pt]
            \item contrary to BH, we leverage the Gram structure to vectorize the computations and consider \eqref{eq:chain_rule_geometric} instead of \eqref{eq:chain_rule_schur}, and evaluate $K_N(x,y)$ via \eqref{eq:kernel_projection_DPP_Gram_formulation} instead of \eqref{eq:kernel_multivariate_separable_OPE}.
            The overall procedure is akin to a sequential Gram-Schmidt orthogonalization of the feature vectors $\Phi(x_{1}), \dots, \Phi(x_{N})$.
            Moreover we pay special attention to avoiding unnecessary evaluations of orthogonal polynomials (OP) when computing a feature vector $\Phi(x)$.
            This is done by coupling the slicing feature of the Python language with the dedicated method \texttt{scipy.special.eval\_jacobi}, used to evaluate the three-term recurrence relations satisfied by each of the univariate Jacobi polynomials.
            Given the chosen ordering $\mathfrak{b}$, the computation of $\Phi(x)$ requires the evaluation of $d$ recurrence relations up to depth $\sqrt[d]{N}$.

            \item like BH, we sample each conditional in turn using a rejection sampling mechanism with the same proposal distribution.
            But BH take as proposal $\weq(x) \diff x$, which corresponds to the limiting marginal of the multivariate Jacobi ensemble as $N$ goes to infinity; see \cite[Section 3.11]{Sim11}.
            On our side, we use a two-layer rejection sampling scheme.
            We rather sample from the $n$-th conditional using the marginal distribution $N^{-1} K_N(x, x) \omega(x) \diff x$ as proposal and rejection constant $N/(N-(n-1))$.
            This allows us to reduce the number of (costly) evaluations of
            the acceptance ratio.
            The marginal distribution itself is sampled using the same proposal $\weq(x) \diff x$ and rejection constant as BH.
            The rejection constant, of order $2^d$, is derived from \citet{ChGaWo94} and \citet{Gau09}.
            We further reduced the number of OP evaluations by considering $N^{-1} K_N(x, x) \omega(x) \diff x$ as a mixture, where each component in \eqref{eq:kernel_multivariate_separable_OPE} involves only one OP.
            In the end, the expected total number of rejections is of order $2^d N\log N$.
            Finally, we implemented a specific rejection free method for the univariate Jacobi ensemble; a special continuous projection DPP which can be sampled exactly in $\calO(N^2)$, by computing the eigenvalues of a random tridiagonal matrix \citep[Theorem 2]{KiNe04}.
        \end{enumerate}
        All these improvements resulted in dramatic speedups.
        For example, on a modern laptop, sampling a $2D$ Jacobi ensemble with $N=1000$ points, see \Figref{fig:sample_2D_raw}, takes less than a minute, compared to hours previously.
        For more details on the sampling procedure, we refer to \Appref{sub:app_sampling}.\\[-1.5em]

    \begin{figure*}[ht]
        \centering
        \subfigure[
            \textcolor{red}{$\propto\omega^1$},
            \textcolor{Green}{$\propto\omega^2$},
            \textcolor{Orange}{$\weq$}
        ]{
        \label{fig:sample_2D_raw}
        \includegraphics[width=0.25\textwidth]
            {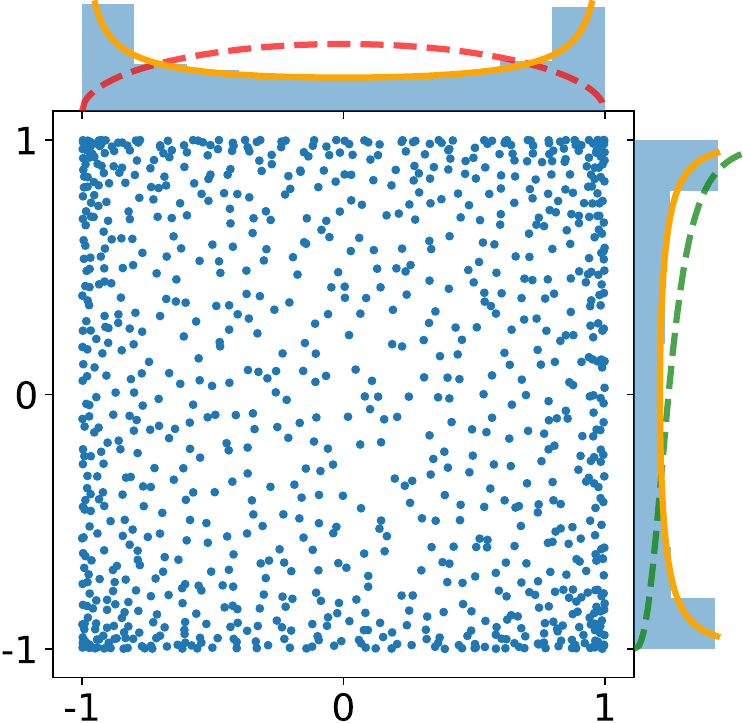}
        }
        \hspace{-1em}
        \subfigure[$\lrsp{\text{time}}$ to get one sample]{
        \label{fig:timings}
        \includegraphics[width=0.36\textwidth]
            {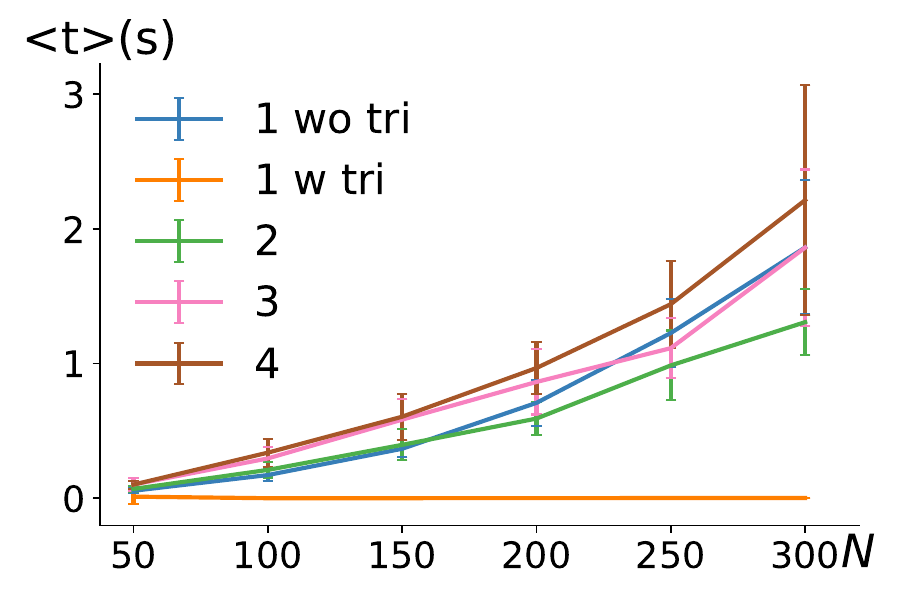}
        }
        \hspace{-1em}
        \subfigure[$\lrsp{\#\text{rejections}}$ to get one sample]{
        \label{fig:rejections}
        \includegraphics[width=0.36\textwidth]
            {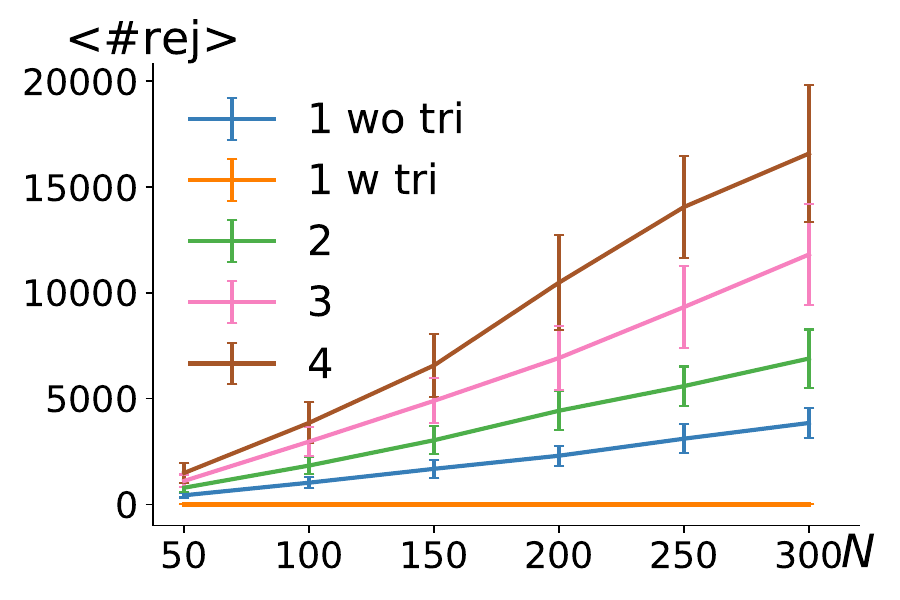}
        }
        \vspace{-0.5em}
        \caption{
        (a) A sample from a $2D$ Jacobi ensemble with $N=1000$ points.
        (b)-(c) $a^i, b^i=-1/2$, the colors and numbers correspond to the dimension.
        For $d=1$, the tridiagonal model (tri) of \citeauthor{KiNe04} offers tremendous time savings.
        (c) The total number of rejections grows as $2^d N \log(N)$.
        }
        \label{fig:timings_rejections}
        \vspace{-1.2em}
    \end{figure*}


\section{Empirical investigation}
\label{sec:XP}

    We perform three main sets of experiments to investigate the properties of the BH \eqref{eq:BH_estimator} and EZ \eqref{eq:EZ_estimator} estimators of the integral $\int f(x) \mu(\diff x)$.
    We add the baseline vanilla Monte Carlo, where points are drawn i.i.d.\,proportionally to $\mu$.
    The two estimators are built from the multivariate Jacobi ensemble, cf.\,\Secref{sub:multivariate_jacobi_ensemble}.
    First, we extend, for larger $N$, the experiments of \citet{BaHa19} illustrating their fast CLT \eqref{eq:BH_CLT} on a smooth function.
    Then, we illustrate \Thref{th:ermakov_zolotukhin_estimators} by considering polynomial functions that can be either fully or partially decomposed in the eigenbasis of the DPP kernel.
    Finally, we compare the potential of both estimators on various non smooth functions.
    \vspace{-1em}
    \begin{figure*}[ht]
        \subfigure[$d=1$]{
        \label{fig:bump_1D}
        \includegraphics[width=\fourfig]
            {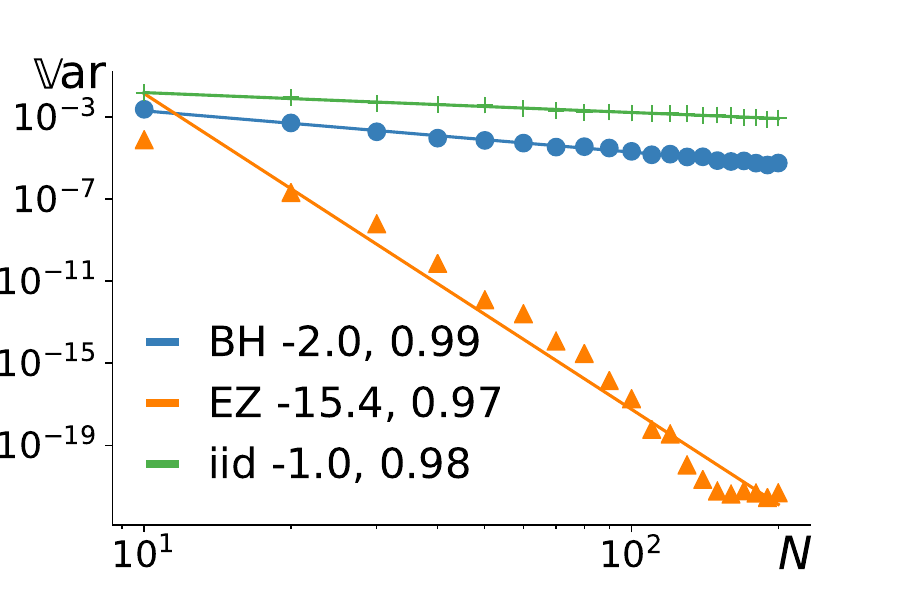}
        }
        \subfigure[$d=2$]{
        \label{fig:bump_2D}
        \includegraphics[width=\fourfig]
            {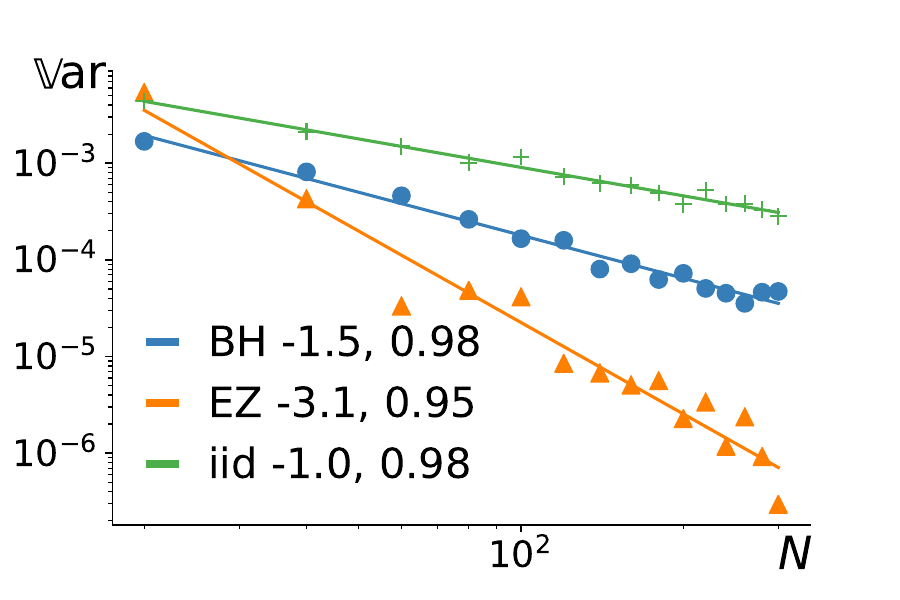}
        }
        \subfigure[$d=3$]{
        \label{fig:bump_3D}
        \includegraphics[width=\fourfig]
            {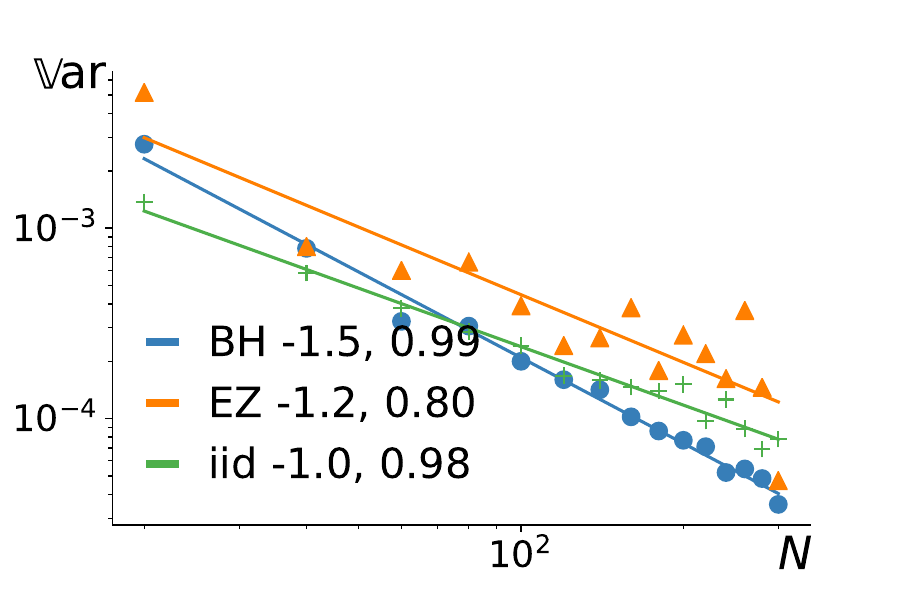}
        }
        \subfigure[$d=4$]{
        \label{fig:bump_4D}
        \includegraphics[width=\fourfig]
            {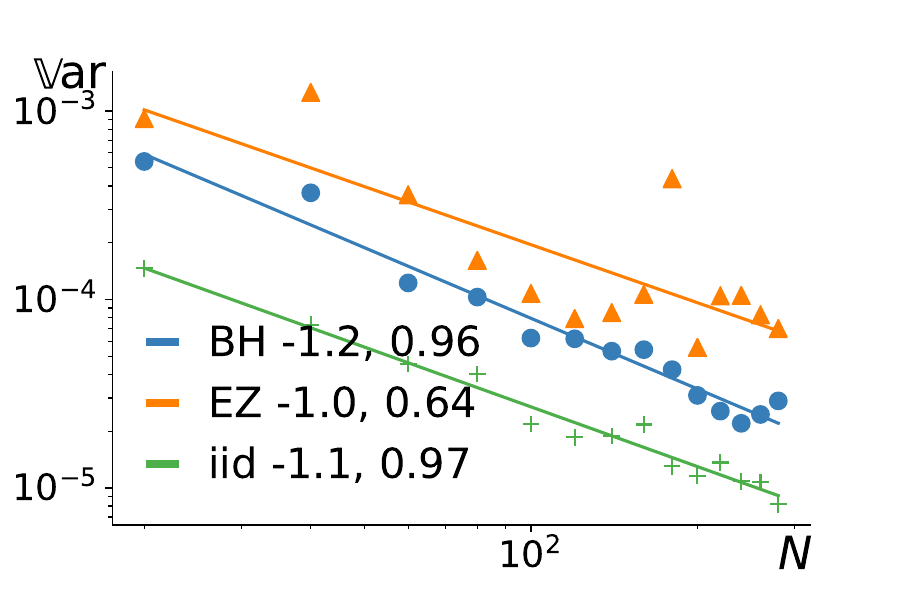}
        }
        \\[-0.4cm]
        \subfigure[$d=1$]{
        \label{fig:decay_1/k_N=70_Ndpp_1D}
        \includegraphics[width=\fourfig]
            {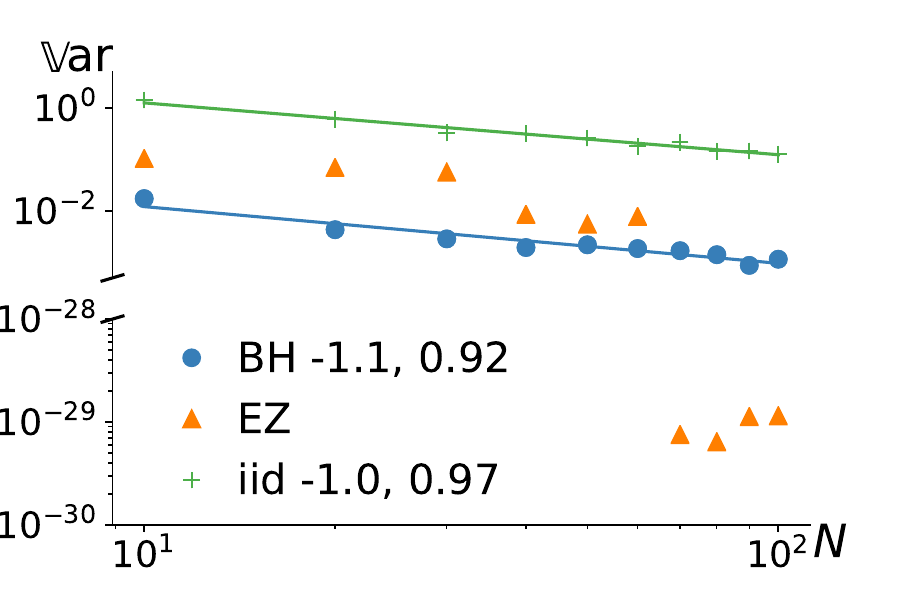}
        }
        \subfigure[$d=2$]{
        \label{fig:decay_1/k_N=70_Ndpp_2D}
        \includegraphics[width=\fourfig]
            {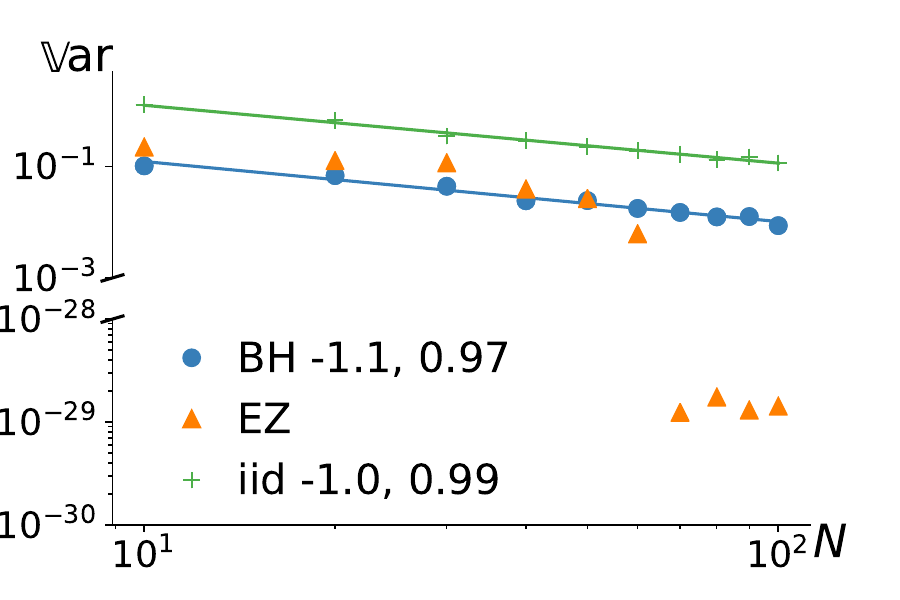}
        }
        \subfigure[$d=3$]{
        \label{fig:decay_1/k_N=70_Ndpp_3D}
        \includegraphics[width=\fourfig]
            {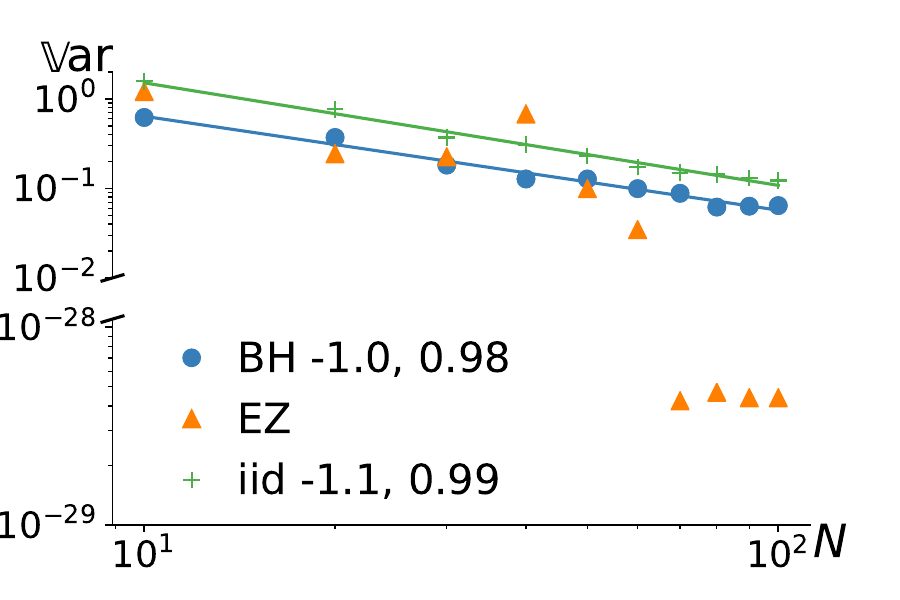}
        }
        \subfigure[$d=4$]{
        \label{fig:decay_1/k_N=70_Ndpp_4D}
        \includegraphics[width=\fourfig]
            {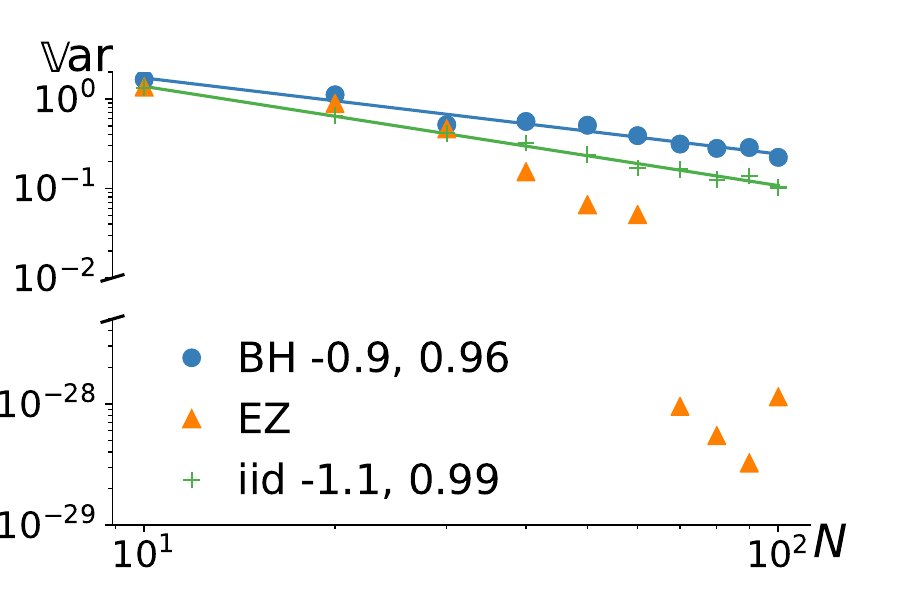}
        }
        \vspace{-0.5em}
        \caption{
            (a)-(d) cf.\,\Secref{sub:XP_bump}, the numbers in the legend are the slope and $R^2$
            (e)-(h) cf.\,\Secref{sub:XP_polynomials}.
            }
        \label{fig:bump_and_drop}
        \vspace{-1em}
    \end{figure*}

    \subsection{The bump experiment}
    \label{sub:XP_bump}

        \citet[Section 3]{BaHa19} illustrate the behavior of $\Iauth{BH}_N$ and its CLT \eqref{eq:BH_CLT} on a unimodal, smooth \emph{bump} function; see Appendix~\ref{sub:app_XP_bump}.
        The expected variance decay is of order $1/N^{1+1/d}$.
        We reproduce their experiment in \Figref{fig:bump_and_drop} for larger $N$, and compare with the behavior of $\Iauth{EZ}_N$.
        In short, $\Iauth{EZ}_N$ dramatically outperforms $\Iauth{BH}_N$ in $d\leq 2$, with surprisingly fast empirical convergence rates.
        When $d\geq 3$, performance decreases, and $\Iauth{BH}_N$ shows both faster and more regular variance decay.

        To know whether we can hope for a CLT for $\Iauth{EZ}_N$, we performed Kolmogorov-Smirnov tests for $N=300$, which yielded small $p$-values across dimensions, from 0.03 to 0.24.
        This is compared to the same $p$-values for $\Iauth{BH}_N$, which range from 0.60 to 0.99.
        The results are presented in \Appref{sub:app_XP_bump}.
        The lack of normality of $\Iauth{EZ}_N$ is partly due to a few outliers.
        Where these outliers come from is left for future work; ill-conditioning of the linear system \eqref{eq:EZ_linear_system} is an obvious candidate.
        Besides, contrary to $\Iauth{BH}_N$, the estimator $\Iauth{EZ}_N$ has no guarantee to preserve the sign of integrands having constant sign.

    \subsection{Integrating sums of eigenfunctions} 
    \label{sub:XP_polynomials}

        In the next series of experiments, we are mainly interested in testing the variance decay of $\Iauth{EZ}_N(f)$ prescribed by \Thref{th:ermakov_zolotukhin_estimators}.
        To that end, we consider functions of the form
        \begin{equation}
        \label{eq:f_xps_polys}
            f(x)
                = \sum\nolimits_{\mathfrak{b}(k)=0}^{M-1}
                    \frac{1}{\mathfrak{b}(k)+1}
                        \phi_{k}(x),
        \end{equation}
        whose integral w.r.t.\,$\mu$ is to be estimated based on realizations of the multivariate Jacobi ensemble with kernel
        $K_N(x,y)=\sum_{\mathfrak{b}(k)=0}^{N-1} \phi_{k}(x) \phi_{k}(y)$, where $N\neq M$ a priori.
        This means that the function~$f$ can be either fully ($M\leq N$) or partially ($M> N$) decomposed in the eigenbasis of the kernel.
        In both cases, we let the number of points $N$ used to build the two estimators vary from $10$ to $100$ in dimensions $d=1$ to $4$.
        In the first setting, we set $M = 70$.
        Thus, $N$ eventually reaches the number of functions used to build $f$ in \eqref{eq:f_xps_polys}, after what $\Iauth{EZ}_N$ is an exact estimator, see \Figref{fig:bump_and_drop}(e)-(h).
        The second setting has $M = N+1$, so that the number of points $N$ is never enough for the variance in \eqref{eq:expe_var_EZ} to be zero.
        The results of both settings are to be found in \Appref{sub:app_XP_polynomials}.

        In the first case, for each dimension $d$, we indeed observe a drop in the variance of $\Iauth{EZ}_N$ once the number of points of the DPP hits the threshold $N=M$.
        This is in perfect agreement with \Thref{th:ermakov_zolotukhin_estimators}: once $f\in \calH_{M} \subseteq \calH_{N}$, the variance in \eqref{eq:expe_var_EZ} is zero.
        In the second setting, as $N$ increases the contribution of the extra mode $\phi_{\mathfrak{b}^{-1}(N)}$ in \eqref{eq:f_xps_polys} decreases as $\frac{1}{N}$.
        Hence, from \Thref{th:ermakov_zolotukhin_estimators} we expect a variance decay of order $\frac{1}{N^2}$, which we observe in practice.

    \subsection{Further experiments} 
    \label{sub:further_experiments}

        In Appendices~\ref{sub:app_XP_absolute}-\ref{sub:integrating_mix} we test the robustness of both BH and EZ estimators, when applied to functions presenting discontinuities or which do not belong to the span of the eigenfunctions of the kernel.
        Although the conditions of the CLT \eqref{eq:BH_CLT} associated to $\Iauth{BH}$ are violated, the corresponding variance decay is smooth but not as fast.
        For $\Iauth{EZ}$, the performance deteriorates with the dimension.
        Indeed, the cross terms arising from the Taylor expansion of the different functions introduce monomials, associated to large coefficients, that do not belong to $\calH_N$.
        Sampling more points would reduce the variance \eqref{eq:expe_var_EZ}.
        But more importantly, for EZ to excel, this suggests to adapt the kernel  to the basis where the integrand is known to be sparse or to have fast-decaying coefficients.
        In regimes where BH and EZ do not shine, vanilla Monte Carlo becomes competitive for small values of $N$.


\section{Conclusion}

    \citet[EZ,][]{ErZo60} proposed a non-obvious unbiased Monte Carlo estimator using projection DPPs.
    It requires solving a linear system, which in turn involves evaluating both the $N$ eigenfunctions of the corresponding kernel and the integrand at the $N$ points of the DPP sample.
    This is yet another connection between DPPs and linear algebra.
    In fact, solving this linear system provides unbiased estimates of the Fourier-like coefficients of the integrand $f$ with each of the $N$ eigenfunctions of the DPP kernel.
    Remarkably, these estimators have identical variance, and this variance measures the accuracy of the approximation of $f$ by its projection onto these eigenfunctions.
    With modern arguments, we have provided a much shorter proof of these properties than in the original work of \citep{ErZo60}.
    Beyond this, little is known on the EZ estimator.
    While coming with a less interpretable variance, the more direct estimator proposed by \citet[BH,][]{BaHa19} has an intrinsic connection with the classical Gauss quadrature and further enjoys stronger theoretical properties when using multivariate Jacobi ensemble.

    Our experiments highlight the key features of both estimators when the underlying DPP is a multivariate Jacobi ensemble, and further demonstrate the known properties of the BH estimator in yet unexplored regimes.
    Although EZ shows a \emph{surprisingly fast} empirical convergence rate for $d\leq 2$, its behavior is more erratic for $d\geq3$.
    Ill-conditioning of the linear system is a potential source of outliers in the distribution of the estimator.
    Regularization may help but would introduce a stability/bias trade-off.
    More generally, EZ seems worth investigating for integration or even interpolation tasks where the function is known to be decomposable in the eigenbasis of the kernel, i.e., in a setting similar to the one of \citet{Bac17}.
    Finally, the new implementation of an exact sampler for multivariate Jacobi ensemble unlocks more large-scale empirical investigations and asks for more theoretical work.
    The associated code\footnoteref{fn:dppy} is available in the DPPy toolbox of \citet{GPBV19}.



\clearpage
\bibliography{biblio} 

\begin{thebibliography}{31}
\providecommand{\natexlab}[1]{#1}
\providecommand{\url}[1]{\texttt{#1}}
\expandafter\ifx\csname urlstyle\endcsname\relax
  \providecommand{\doi}[1]{doi: #1}\else
  \providecommand{\doi}{doi: \begingroup \urlstyle{rm}\Url}\fi

\bibitem[Bach(2017)]{Bac17}
Francis Bach.
\newblock {On the Equivalence between Kernel Quadrature Rules and Random
  Feature Expansions}.
\newblock \emph{Journal of Machine Learning Research}, 18\penalty0
  (21):\penalty0 1--38, 2017.
\newblock URL \url{http://jmlr.org/papers/v18/15-178.html}.

\bibitem[Bach et~al.(2012)Bach, Lacoste-Julien, and Obozinski]{BaLaOb12}
Francis Bach, Simon Lacoste-Julien, and Guillaume Obozinski.
\newblock {On the Equivalence between Herding and Conditional Gradient
  Algorithms}.
\newblock In \emph{International Conference on Machine Learning (ICML)}, 2012.
\newblock URL \url{https://icml.cc/2012/papers/683.pdf}.

\bibitem[Bardenet and Hardy(2019)]{BaHa19}
R{\'{e}}mi Bardenet and Adrien Hardy.
\newblock {Monte Carlo with Determinantal Point Processes}.
\newblock \emph{Annals of Applied Probability, in press}, 2019.
\newblock URL \url{http://arxiv.org/abs/1605.00361}.

\bibitem[Briol et~al.(2015)Briol, Oates, Girolami, and Osborne]{BOGO15}
Fran{\c{c}}ois-Xavier Briol, Chris~J. Oates, Mark Girolami, and Michael~A.
  Osborne.
\newblock {Frank-Wolfe Bayesian Quadrature: Probabilistic Integration with
  Theoretical Guarantees}.
\newblock In \emph{Advances in Neural Information Processing Systems
  (NeurIPS)}, pages 1162--1170, jun 2015.
\newblock URL
  \url{https://papers.nips.cc/paper/5749-frank-wolfe-bayesian-quadrature-probabilistic-integration-with-theoretical-guarantees}.

\bibitem[Chen et~al.(2010)Chen, Welling, and Smola]{ChWeSm10}
Yutian Chen, Max Welling, and Alex Smola.
\newblock {Super-Samples from Kernel Herding}.
\newblock In \emph{Conference on Uncertainty in Artificial Intelligence (UAI)},
  2010.
\newblock ISBN 9780974903965.
\newblock URL \url{https://dl.acm.org/citation.cfm?id=3023562}.

\bibitem[Chow et~al.(1994)Chow, Gatteschi, and Wong]{ChGaWo94}
Yunshyong Chow, L.~Gatteschi, and R.~Wong.
\newblock {A Bernstein-type inequality for the Jacobi polynomial}.
\newblock \emph{Proceedings of the American Mathematical Society}, 121\penalty0
  (3):\penalty0 703--703, 1994.
\newblock ISSN 0002-9939.
\newblock \doi{10.1090/S0002-9939-1994-1209419-X}.
\newblock URL
  \url{http://www.ams.org/jourcgi/jour-getitem?pii=S0002-9939-1994-1209419-X}.

\bibitem[Davis and Rabinowitz(1984)]{DaRa84}
Philip~J. Davis and Philip. Rabinowitz.
\newblock \emph{{Methods of numerical integration}}.
\newblock Academic Press, 1984.
\newblock ISBN 9780122063602.
\newblock URL \url{https://doi.org/10.1016/C2013-0-10566-1}.

\bibitem[Delyon and Portier(2016)]{DePo16}
Bernard Delyon and Fran{\c{c}}ois Portier.
\newblock {Integral approximation by kernel smoothing}.
\newblock \emph{Bernoulli}, 22\penalty0 (4):\penalty0 2177--2208, nov 2016.
\newblock \doi{10.3150/15-BEJ725}.
\newblock URL \url{http://projecteuclid.org/euclid.bj/1462297679}.

\bibitem[Dick and Pillichshammer(2010)]{DiPi10}
Jospeh Dick and Friedrich Pillichshammer.
\newblock \emph{{Digital nets and sequences : discrepancy and quasi-Monte Carlo
  integration}}.
\newblock Cambridge University Press, 2010.
\newblock ISBN 9780521191593.
\newblock URL
  \url{https://www.cambridge.org/vi/academic/subjects/computer-science/algorithmics-complexity-computer-algebra-and-computational-g/digital-nets-and-sequences-discrepancy-theory-and-quasimonte-carlo-integration?format=HB}.

\bibitem[Ermakov and Zolotukhin(1960)]{ErZo60}
S.~M. Ermakov and V.~G. Zolotukhin.
\newblock {Polynomial Approximations and the Monte-Carlo Method}.
\newblock \emph{Theory of Probability {\&} Its Applications}, 5\penalty0
  (4):\penalty0 428--431, jan 1960.
\newblock ISSN 0040-585X.
\newblock \doi{10.1137/1105046}.
\newblock URL \url{http://epubs.siam.org/doi/10.1137/1105046}.

\bibitem[Evans and Swartz(2000)]{EvSc00}
Michael Evans and Tim Swartz.
\newblock \emph{{Approximating integrals via Monte Carlo and deterministic
  methods}}.
\newblock Oxford University Press, 2000.
\newblock ISBN 9780198502784.
\newblock URL
  \url{https://global.oup.com/academic/product/approximating-integrals-via-monte-carlo-and-deterministic-methods-9780198502784}.

\bibitem[Gautier et~al.(2019)Gautier, Polito, Bardenet, and Valko]{GPBV19}
Guillaume Gautier, Guillermo Polito, R{\'{e}}mi Bardenet, and Michal Valko.
\newblock {DPPy: DPP Sampling with Python}.
\newblock \emph{Journal of Machine Learning Research - Machine Learning Open
  Source Software (JMLR-MLOSS), in press}, 2019.
\newblock URL \url{http://arxiv.org/abs/1809.07258}.

\bibitem[Gautschi(2009)]{Gau09}
Walter Gautschi.
\newblock {How sharp is Bernstein's Inequality for Jacobi polynomials?}
\newblock \emph{Electronic Transactions on Numerical Analysis}, 36:\penalty0
  1--8, 2009.
\newblock URL
  \url{http://emis.ams.org/journals/ETNA/vol.36.2009-2010/pp1-8.dir/pp1-8.pdf}.

\bibitem[Hough et~al.(2006)Hough, Krishnapur, Peres, and Vir{\'{a}}g]{HKPV06}
J.~Ben Hough, Manjunath Krishnapur, Yuval Peres, and B{\'{a}}lint Vir{\'{a}}g.
\newblock {Determinantal Processes and Independence}.
\newblock In \emph{Probability Surveys}, volume~3, pages 206--229. The
  Institute of Mathematical Statistics and the Bernoulli Society, 2006.
\newblock \doi{10.1214/154957806000000078}.
\newblock URL \url{http://arxiv.org/abs/math/0503110}.

\bibitem[Husz{\'{a}}r and Duvenaud(2012)]{HuDu12}
Ferenc Husz{\'{a}}r and David Duvenaud.
\newblock {Optimally-Weighted Herding is Bayesian Quadrature}.
\newblock In \emph{Conference on Uncertainty in Artificial Intelligence (UAI)},
  2012.
\newblock ISBN 9780974903989.
\newblock URL \url{https://dl.acm.org/citation.cfm?id=3020694}.

\bibitem[Johansson(2006)]{Joh06}
Kurt Johansson.
\newblock {Random matrices and determinantal processes}.
\newblock \emph{Les Houches Summer School Proceedings}, 83\penalty0
  (C):\penalty0 1--56, 2006.
\newblock ISSN 09248099.
\newblock \doi{10.1016/S0924-8099(06)80038-7}.

\bibitem[Killip and Nenciu(2004)]{KiNe04}
Rowan Killip and Irina Nenciu.
\newblock {Matrix models for circular ensembles}.
\newblock \emph{International Mathematics Research Notices}, 2004\penalty0
  (50):\penalty0 2665, 2004.
\newblock ISSN 1073-7928.
\newblock \doi{10.1155/S1073792804141597}.
\newblock URL
  \url{https://academic.oup.com/imrn/article-lookup/doi/10.1155/S1073792804141597}.

\bibitem[K{\"{o}}nig(2004)]{Kon05}
Wolfgang K{\"{o}}nig.
\newblock {Orthogonal polynomial ensembles in probability theory}.
\newblock \emph{Probability Surveys}, 2:\penalty0 385--447, 2004.
\newblock ISSN 1549-5787.
\newblock \doi{10.1214/154957805100000177}.
\newblock URL \url{http://arxiv.org/abs/math/0403090}.

\bibitem[Kulesza and Taskar(2012)]{KuTa12}
Alex Kulesza and Ben Taskar.
\newblock {Determinantal Point Processes for Machine Learning}.
\newblock \emph{Foundations and Trends in Machine Learning}, 5\penalty0
  (2-3):\penalty0 123--286, 2012.
\newblock ISSN 1935-8237.
\newblock \doi{10.1561/2200000044}.
\newblock URL \url{http://arxiv.org/abs/1207.6083}.

\bibitem[Lavancier et~al.(2012)Lavancier, M{\o}ller, and Rubak]{LaMoRu15}
Fr{\'{e}}d{\'{e}}ric Lavancier, Jesper M{\o}ller, and Ege Rubak.
\newblock {Determinantal point process models and statistical inference :
  Extended version}.
\newblock \emph{Journal of the Royal Statistical Society. Series B: Statistical
  Methodology}, 77\penalty0 (4):\penalty0 853--877, may 2012.
\newblock \doi{10.1111/rssb.12096}.
\newblock URL
  \url{https://rss.onlinelibrary.wiley.com/doi/abs/10.1111/rssb.12096}.

\bibitem[Liu and Lee(2017)]{LiLe17}
Qiang Liu and Jason~D Lee.
\newblock {Black-Box Importance Sampling}.
\newblock In \emph{Internation Conference on Artificial Intelligence and
  Statistics (AISTATS)}, 2017.
\newblock URL \url{http://proceedings.mlr.press/v54/liu17b.html}.

\bibitem[Macchi(1975)]{Mac75}
Odile Macchi.
\newblock {The coincidence approach to stochastic point processes}.
\newblock \emph{Advances in Applied Probability}, 7\penalty0 (01):\penalty0
  83--122, mar 1975.
\newblock ISSN 0001-8678.
\newblock \doi{10.2307/1425855}.
\newblock URL
  \url{https://www.cambridge.org/core/product/identifier/S0001867800040313/type/journal{\_}article}.

\bibitem[Mallat and Peyr{\'{e}}(2009)]{MaPe09}
St{\'{e}}phane Mallat and Gabriel Peyr{\'{e}}.
\newblock \emph{{A wavelet tour of signal processing : the sparse way}}.
\newblock Elsevier/Academic Press, 2009.
\newblock ISBN 9780123743701.
\newblock URL
  \url{https://www.sciencedirect.com/book/9780123743701/a-wavelet-tour-of-signal-processing}.

\bibitem[Mazoyer et~al.(2019)Mazoyer, Coeurjolly, and Amblard]{MaCoAm19}
Adrien Mazoyer, Jean-Fran{\c{c}}ois Coeurjolly, and Pierre-Olivier Amblard.
\newblock {Projections of determinantal point processes}.
\newblock \emph{ArXiv e-prints}, 2019.
\newblock URL \url{https://arxiv.org/pdf/1901.02099.pdf}.

\bibitem[Oates et~al.(2017)Oates, Girolami, and Chopin]{OaGiCh17}
Chris~J. Oates, Mark Girolami, and Nicolas Chopin.
\newblock {Control functionals for Monte Carlo integration}.
\newblock \emph{Journal of the Royal Statistical Society: Series B (Statistical
  Methodology)}, 79\penalty0 (3):\penalty0 695--718, jun 2017.
\newblock \doi{10.1111/rssb.12185}.
\newblock URL \url{http://doi.wiley.com/10.1111/rssb.12185}.

\bibitem[O'Hagan(1991)]{Hag91}
A.~O'Hagan.
\newblock {Bayes–Hermite quadrature}.
\newblock \emph{Journal of Statistical Planning and Inference}, 29\penalty0
  (3):\penalty0 245--260, nov 1991.
\newblock \doi{10.1016/0378-3758(91)90002-V}.
\newblock URL
  \url{https://www.sciencedirect.com/science/article/pii/037837589190002V}.

\bibitem[Rasmussen and Williams(2006)]{RaWi06}
Carl~Edward. Rasmussen and Christopher K.~I. Williams.
\newblock \emph{{Gaussian processes for machine learning}}.
\newblock MIT Press, 2006.
\newblock ISBN 026218253X.
\newblock URL \url{http://www.gaussianprocess.org/gpml/}.

\bibitem[Robert(2007)]{Rob07}
Christian~P. Robert.
\newblock \emph{{The Bayesian choice : from decision-theoretic foundations to
  computational implementation}}.
\newblock Springer, 2007.
\newblock ISBN 9780387952314.
\newblock \doi{10.1007/0-387-71599-1}.
\newblock URL \url{https://www.springer.com/la/book/9780387952314}.

\bibitem[Robert and Casella(2004)]{RoCa04}
Christian~P. Robert and George. Casella.
\newblock \emph{{Monte Carlo statistical methods}}.
\newblock Springer-Verlag New York, 2004.
\newblock ISBN 9781441919397.
\newblock \doi{10.1007/978-1-4757-4145-2}.

\bibitem[Simon(2011)]{Sim11}
B.~Simon.
\newblock \emph{Szegő's Theorem and its Descendants: Spectral Theory for $L^2$
  Perturbations of Orthogonal Polynomials}.
\newblock M. B. Porter Lecture Series, Princeton Univ. Press, Princeton, NJ,
  2011.
\newblock URL
  \url{https://press.princeton.edu/books/hardcover/9780691147048/szegos-theorem-and-its-descendants}.

\bibitem[Soshnikov(2000)]{Sos00}
Alexander Soshnikov.
\newblock {Determinantal random point fields}.
\newblock \emph{Russian Mathematical Surveys}, 55\penalty0 (5):\penalty0
  923--975, feb 2000.
\newblock ISSN 0042-1316.
\newblock \doi{10.1070/RM2000v055n05ABEH000321}.
\newblock URL \url{http://dx.doi.org/10.1070/RM2000v055n05ABEH000321}.

\end{thebibliography}
\bibliographystyle{plainnat}

\clearpage

\appendix

\renewcommand{\thefigure}{\Alph{figure}}
\renewcommand{\thetheorem}{\Alph{theorem}}
\renewcommand{\thelemma}{\Alph{lemma}}
\numberwithin{equation}{section}
\numberwithin{figure}{section}

\section{Methodology} 
\label{sec:method}

    \subsection{The generalized Cauchy-Binet formula: a modern argument} 
    \label{sub:normalization_cauchy_binet}

        \citet[Section 2.2]{Joh06} developed a natural way to build DPPs associated to projection (potentially non-Hermitian) kernels.
        In this part, we focus on the generalization of the Cauchy-Binet formula \citep[Proposition 2.10]{Joh06}.
        Its usefulness is twofold for our purpose.
        First, it serves to justify the fact that the normalization constant of the joint distribution \eqref{eq:joint_distribution_projection_dpp}  is one, i.e., it is indeed a probability distribution.
        Second, we use it as a modern and simple argument to prove a slight improvement of the result of \citet{ErZo60}, cf.\,\Thref{th:ermakov_zolotukhin_estimators}.
        An extended version of the proof is given in \Appref{sub:app_proof_ErZo_th}.

        \begin{lemma}\citep[Proposition 2.10]{Joh06}
        \label{lem:app_cauchy_binet}
            Let $(\bbX, \calB, \mu)$ be a measurable space and consider measurable functions $\phi_{0}, \dots, \phi_{N-1}$ and $\psi_{0}, \dots, \psi_{N-1}$, such that $\phi_k \psi_{\ell} \in L^1(\mu)$.
            Then,
            \begin{equation}
            \label{eq:app_cauchy_binet}
                \det
                    \lrp{
                        \lrsp{\phi_k, \psi_{\ell}}
                        }_{k,\ell=0}^{N-1}
                    = \frac{1}{N!}
                        \int
                        \det \bfPhi(x_{1:N})
                        \det \bfPsi(x_{1:N}) \, \mu^{\otimes N}( \diff x),
            \end{equation}
            where
            \begin{equation*}
                \bfPhi(x_{1:N})
                =
                \begin{pmatrix}
                    \phi_0(x_1) & \dots  & \phi_{N-1}(x_1)\\
                        \vdots     &        & \vdots\\
                    \phi_0(x_N) & \dots  & \phi_{N-1}(x_N)
                \end{pmatrix}
                \quad\text{and}\quad
                \bfPsi(x_{1:N})
                =
                \begin{pmatrix}
                    \psi_0(x_1) & \dots  & \psi_{N-1}(x_1)\\
                        \vdots     &        & \vdots\\
                    \psi_0(x_N) & \dots  & \psi_{N-1}(x_N)
                \end{pmatrix}
            \end{equation*}
        \end{lemma}


    \subsection{Proof of \Thref{th:ermakov_zolotukhin_estimators}}
    \label{sub:app_proof_ErZo_th}

        First, we recall the result of \citet{ErZo60}, cf.\,\Thref{th:ermakov_zolotukhin_estimators}.
        Then, we provide a modern proof based on the generalization of the Cauchy-Binet formula, cf.\,\Lemref{lem:app_cauchy_binet}, where we exploit the orthonormality of the eigenfunctions of the kernel.
        \begin{theorem}
        \label{th:app_ermakov_zolotukhin_estimators}
            Consider $f\in L^2(\mu)$ and $N$ functions $\phi_0, \dots, \phi_{N-1}\in L^2(\mu)$ orthonormal w.r.t.\,$\mu$, i.e.,
            \begin{equation}
                \label{eq:app_orthonormality}
                \lrsp{\phi_k, \phi_{\ell}}
                    \triangleq \int \phi_k(x) \phi_{\ell}(x) \mu(\diff x)
                    = \delta_{k\ell},
                    \quad \forall 0\leq k, \ell \leq N-1.
            \end{equation}
            Let $\{\bfx_1,\ldots,\bfx_N\} \sim \DPP(\mu, K_N)$, with projection kernel $K_N(x,y)=\sum_{k=0}^{N-1} \phi_k(x)\phi_k(y)$.
            That is to say $\lrp{\bfx_{1}, \dots, \bfx_{N}}$ has joint distribution
            \begin{equation}
            \label{eq:app_joint_distribution_projection_dpp}
                \frac1{N!}
                    \det\lrp{K_N(x_p, x_q)}_{p, q=1}^N
                    \, \mu^{\otimes N}(\diff x)
                = \frac1{N!}
                    \lrp{\det\bfPhi(x_{1:N})}^2
                    \, \mu^{\otimes N}(\diff x).
            \end{equation}
            Consider the linear system $\bfPhi(\bfx_{1:N}) y = f(\bfx_{1:N})$, that is,
            \begin{equation}
                \label{eq:app_EZ_linear_system}
                \begin{pmatrix}
                    \phi_0(\bfx_1) & \dots  & \phi_{N-1}(\bfx_1)\\
                        \vdots     &        & \vdots\\
                    \phi_0(\bfx_N) & \dots  & \phi_{N-1}(\bfx_N)
                \end{pmatrix}
                \begin{pmatrix}
                    y_1\\
                    \vdots\\
                    y_N
                \end{pmatrix}
                =
                \begin{pmatrix}
                    f(\bfx_1)\\
                    \vdots\\
                    f(\bfx_N)
                \end{pmatrix}.
            \end{equation}
            Then, the solution of \eqref{eq:app_EZ_linear_system} is unique, $\mu$-almost surely, with coordinates
            \begin{equation}
                \label{eq:app_coordinates_EZ_theorem}
                y_k
                =
                \frac{\det \bfPhi_{\phi_{k-1}, f}(\bfx_{1:N})}
                     {\det \bfPhi(\bfx_{1:N})}
                \CommaBin
            \end{equation}
            where $\bfPhi_{\phi_{k-1}, f}(\bfx_{1:N})$ is the matrix obtained by replacing the $k$-th column of $\bfPhi(\bfx_{1:N})$ by $f(\bfx_{1:N})$.
            Moreover, for all $1\leq k\leq N$, the coordinate $y_k$ of the solution vector satisfies
            \begin{equation}
                \label{eq:app_expe_var_EZ}
                \Expe[y_k]
                    = \lrsp{f, \phi_{k-1}},
                \quad\text{and}\quad
                \Var[y_k]
                    = \lrnorm{f}^2
                       - \sum_{\ell=0}^{N-1}
                         \lrsp{f, \phi_{\ell}}^2.
            \end{equation}
            We improved the original result by showing that $\Cov[y_j, y_k]= 0$, for all $1\leq j\neq k\leq N$.
        \end{theorem}

        \begin{proof}[Proof of \Thref{th:app_ermakov_zolotukhin_estimators}]
            First, the joint distribution \eqref{eq:app_joint_distribution_projection_dpp} of $\lrp{\bfx_{1}, \dots, \bfx_{N}}$ is proportional to $\lrp{\det \bfPhi(x_{1:N})}^2 \mu^{\otimes N}(x)$.
            Thus, $\det \bfPhi(\bfx_{1:N})\neq 0$, $\mu$-almost surely.
            Hence, the matrix $\bfPhi(\bfx_{1:N})$ defining the linear system \eqref{eq:app_EZ_linear_system} is invertible, $\mu$-almost surely.

            The expression of the coordinates \eqref{eq:app_coordinates_EZ_theorem} follows from Cramer's rule.

            Then, we treat the case $k=1$, the others follow the same lines.
            The proof relies on \Lemref{lem:app_cauchy_binet} where we exploit the orthonormality of the $\phi_k$s \eqref{eq:app_orthonormality}.
            The expectation \eqref{eq:app_expe_var_EZ} reads
            \begin{align}
                \Expe[\frac{\det \bfPhi_{\phi_0, f}(\bfx_{1:N})}
                            {\det \bfPhi(\bfx_{1:N})}]
                &\lequal{}{\eqref{eq:app_joint_distribution_projection_dpp}}
                \frac{1}{N!}
                    \int
                    \det \bfPhi_{\phi_0, f} (x_{1:N})
                    \det \bfPhi (x_{1:N})
                        \, \mu^{\otimes N}( \diff x) \nonumber\\
                &\lequal{}{\eqref{eq:app_cauchy_binet}}
                    \det
                    \begin{pmatrix}
                        \lrsp{f, \phi_{0}}
                        & \lrp{
                            \lrsp{f, \phi_{\ell}}
                            }_{\ell=1}^{N-1}
                        \\[1em]
                        \lrp{
                            \lrsp{\phi_{k}, \phi_{0}}
                            }_{k=1}^{N-1}
                        & \lrp{
                            \lrsp{\phi_k, \phi_{\ell}}
                            }_{k,\ell=1}^{N-1}
                    \end{pmatrix}
                    \nonumber\\
                &\lequal{}{\eqref{eq:app_orthonormality}}
                    \det
                    \begin{pmatrix}
                        \lrsp{f, \phi_{0}}
                        & \lrp{
                            \lrsp{f, \phi_{\ell}}
                            }_{\ell=1}^{N-1}
                        \\
                        0_{N-1,1}
                        & I_{N-1}
                    \end{pmatrix}
                    \nonumber\\
                &= \lrsp{f, \phi_{0}}.
                \label{eq:app_proof_1st_moment}
            \end{align}
            Similarly, the second moment reads
            \begin{align}
                \Expe[
                    \lrp{
                     \frac{\det \bfPhi_{\phi_0, f}(\bfx_{1:N})}
                      {\det \bfPhi(\bfx_{1:N})}}
                    ]
                &\lequal{}{\eqref{eq:app_joint_distribution_projection_dpp}}
                    \frac{1}{N!}
                    \int
                    \det \bfPhi_{\phi_0, f}(x_{1:N})
                    \det \bfPhi_{\phi_0, f}(x_{1:N})
                    \, \mu^{\otimes N}( \diff x)
                    \nonumber\\
                &\lequal{}{\eqref{eq:app_cauchy_binet}}
                    \det
                    \begin{pmatrix}
                        \lrsp{f, f}
                        & \lrp{
                            \lrsp{f, \phi_{\ell}}
                            }_{\ell=1}^{N-1}\\[1em]
                        \lrp{
                            \lrsp{\phi_{k}, f}
                            }_{k=1}^{N-1}
                        & \lrp{
                            \lrsp{\phi_k, \phi_{\ell}}
                            }_{k,\ell=1}^{N-1}
                    \end{pmatrix}
                    \nonumber\\
                &\lequal{}{\eqref{eq:app_orthonormality}}
                    \det
                    \begin{pmatrix}
                        \lrnorm{f}^2
                        & \lrp{
                            \lrsp{f, \phi_{\ell}}
                            }_{\ell=1}^{N-1}
                        \\
                        \lrp{
                            \lrsp{f, \phi_{k}}
                            }_{k=1}^{N-1}
                        & I_{N-1}
                    \end{pmatrix}
                    \nonumber\\
                &= \lrnorm{f}^2
                    - \sum_{k=1}^{N-1}
                        \lrsp{f, \phi_{k}}^2.
                    \label{eq:app_proof_2nd_moment}
            \end{align}
            Finally, the variance in \eqref{eq:app_expe_var_EZ} $=$ \eqref{eq:app_proof_2nd_moment} - \eqref{eq:app_proof_1st_moment}$^2$.

            With the same arguments, for $j\neq k$, we can compute the covariance $\Cov[y_j, y_k]$.
            For simplicity, we treat only the case $j=1, k=2$, the general case follows the same lines.
            \begin{align*}
                \Cov[y_1, y_2]
                &=
                \Expe[
                    \frac{\det \bfPhi_{\phi_0, f}(\bfx_{1:N})}
                         {\det \bfPhi(\bfx_{1:N})}
                    \frac{\det \bfPhi_{\phi_1, f}(\bfx_{1:N})}
                         {\det \bfPhi(\bfx_{1:N})}
                    ]
                - \Expe[
                    \frac{\det \bfPhi_{\phi_0, f}(\bfx_{1:N})}
                         {\det \bfPhi(\bfx_{1:N})}
                    ]
                  \Expe[
                    \frac{\det \bfPhi_{\phi_1, f}(\bfx_{1:N})}
                         {\det \bfPhi(\bfx_{1:N})}
                    ]
                    \\
                &\lequal{}{
                    \eqref{eq:app_joint_distribution_projection_dpp},
                    \eqref{eq:app_proof_1st_moment}}
                    \frac{1}{N!}
                    \int
                    \det \bfPhi_{\phi_0, f}(x_{1:N})
                    \det \bfPhi_{\phi_1, f}(x_{1:N})
                    \, \mu^{\otimes N}( \diff x)
                    - \lrsp{f, \phi_{0}} \lrsp{f, \phi_{1}}
                    \\
                &\lequal{}{\eqref{eq:app_cauchy_binet}}
                    \det
                    \begin{pmatrix}
                        \lrsp{f, \phi_0}
                            & \lrsp{f, f}
                            & \lrp{
                                \lrsp{f, \phi_{\ell}}
                                }_{\ell=2}^{N-1}\\[1em]
                        \lrp{
                            \lrsp{\phi_{k}, \phi_{0}}
                            }_{k=1}^{N-1}
                            & \lrp{
                                \lrsp{\phi_k, f}
                                }_{k=1}^{N-1}
                            & \lrp{
                                \lrsp{\phi_k, \phi_{\ell}}
                                }_{k=1,\ell=2}^{N-1}
                    \end{pmatrix}
                    - \lrsp{f, \phi_{0}} \lrsp{f, \phi_{1}}
                    \\
                &\lequal{}{\eqref{eq:app_orthonormality}}
                    \det
                    \begin{pmatrix}
                        \lrsp{f, \phi_0}
                            & \lrnorm{f}^2
                            & \lrp{
                                \lrsp{f, \phi_{\ell}}
                                }_{\ell=2}^{N-1}\\[1em]
                        0
                            & \lrsp{\phi_1, f}
                            & 0\\[1em]
                        0_{N-2, 1}
                            & \lrp{
                                \lrsp{\phi_k, f}
                                }_{k=2}^{N-1}
                            & I_{N-2}
                    \end{pmatrix}
                    - \lrsp{f, \phi_{0}} \lrsp{f, \phi_{1}}
                    \\
                &= \lrsp{f, \phi_{0}} \lrsp{f, \phi_{1}}
                    - \lrsp{f, \phi_{0}} \lrsp{f, \phi_{1}}
                 = 0.
            \end{align*}
        \end{proof}

    \subsection{The EZ estimator as a quadrature rule} 
    \label{sub:app_EZ_as_quadrature}

        In this part, we consider \Thref{th:app_ermakov_zolotukhin_estimators} in the setting where one of the eigenfunctions of the kernel, say $\phi_0$ is constant.
        In this case, we show that $\Iauth{EZ}_N(f)$ defined by \eqref{eq:EZ_estimator} provides an unbiased estimate of $\int_{\bbX} f(x) \mu(\diff x)$ with known variance.
        In addition, it can be seen as a quadrature rule in the sense of \eqref{eq:quadrature}, with weights a priori non negative weights $\omega_n$ that sum to $\mu(\bbX)$.
        This is a non obvious fact, judging from the expression \eqref{eq:EZ_estimator} of the estimator.
        \begin{proposition}
            Consider $\phi_0$ constant in \Thref{th:app_ermakov_zolotukhin_estimators}.
            Then, solving the corresponding linear system \eqref{eq:app_EZ_linear_system} allows to construct
            \begin{equation}
            \label{eq:app_EZ_estimator}
                \Iauth{EZ}_N(f)
                    \triangleq \frac{y_1}{\phi_0}
                    = \mu(\bbX)^{1/2}
                        ~
                        \frac{\det \bfPhi_{\phi_0, f}(\bfx_{1:N})}
                             {\det \bfPhi(\bfx_{1:N})}\CommaBin
            \end{equation}
            as an unbiased estimator of $\int_{\bbX} f(x) \mu(\diff x)$, with variance equal to $\mu(\bbX)\times$\eqref{eq:app_expe_var_EZ}.
            In addition, it can be seen as a random quadrature rule \eqref{eq:quadrature} with weights summing to $\mu(\bbX)$.
        \end{proposition}
        \begin{proof}
            Since $\phi_0$ is constant with unit norm we have $\phi_0 = \mu(\bbX)^{-1/2}$, so that
            \begin{align*}
                \Expe[\Iauth{EZ}_N(f)]
                    &= \frac{1}{\phi_0} \Expe[y_1]
                    = \frac{1}{\cancel{\phi_0}} \lrsp{f, \cancel{\phi_0}}
                    = \int_{\bbX} f(x) \diff x,
            \intertext{and}
                \Var[\Iauth{EZ}_N(f)]
                    &= \frac{1}{{\phi_0}^2} \Var[y_1]
                    = \mu(\bbX) \times \eqref{eq:app_expe_var_EZ}.
            \end{align*}
            In addition, \eqref{eq:app_EZ_estimator} can be written
            \begin{align}
            \Iauth{EZ}_N(f)
                &= \mu(\bbX)^{1/2}
                    ~
                        \frac{\det \bfPhi_{\phi_0, f}(\bfx_{1:N})}
                             {\phi_0 \det \bfPhi_{\phi_0, 1}(\bfx_{1:N})}
                =
                    \mu(\bbX)
                    ~
                        \frac{\det \bfPhi_{\phi_0, f}(\bfx_{1:N})}
                             {\det \bfPhi_{\phi_0, 1}(\bfx_{1:N})}\CommaBin
            \nonumber
            \intertext{and the expansion of the numerator w.r.t.\,the first column yields}
            \Iauth{EZ}_N(f)
                &=
                \sum_{n=1}^N
                    f(\bfx_n)
                    \underbrace{
                        \frac
                            {\mu(\bbX)}
                            {\det \bfPhi_{\phi_0, 1}(\bfx_{1:N})}
                        (-1)^{1+n}
                        \det
                        \lrp{\phi_k(x_p)}_{k=1, p=1\neq n}^{N-1, N}
                        }_{\triangleq\omega_n(\bfx_{1:N})}\cdot
                \label{eq:app_weights_EZ_estimator}
            \end{align}
            Note that there is a priori no reason for the weights to be nonnegative.
            Finally,
            \begin{equation*}
                \sum_{n=1}^{N}
                    \omega_n(\bfx_{1:N})
                = \frac
                    {\mu(\bbX)}
                    {\cancel{\det \bfPhi_{\phi_0, 1}(\bfx_{1:N})}}
                    \underbrace{
                        \sum_{n=1}^{N}
                            (-1)^{1+n}
                            \det
                            \lrp{\phi_k(x_p)}_{k=1, p=1\neq n}^{N-1, N}
                        }_{=\cancel{\det \bfPhi_{\phi_0, 1}(\bfx_{1:N})}}
                = \mu(\bbX).
            \end{equation*}
            This concludes the proof.
        \end{proof}


    \subsection{Sampling from the multivariate Jacobi ensemble} 
    \label{sub:app_sampling}

        We mention that the code\footnoteref{fn:dppy} and the documentation\footnoteref{fn:readthedocs} associated to this work are available in the DPPy toolbox of \citet{GPBV19}.

        In dimension $d=1$, we implemented the random tridiagonal matrix model of \citet[Theorem 2]{KiNe04} to sample from the univariate Jacobi ensemble, with base measure $\mu(\diff x) = (1-x)^a (1+x)^b \diff x$, where $a, b > -1$.
        That is to say, this one dimensional continuous projection DPP with $N$ points can be sampled in $\calO(N^2)$, by computing the eigenvalues of random tridiagonal matrix with i.i.d. coefficients of size $N\times N$.

        Next, for $d\geq 2$, we detail the procedure described in \Secref{sub:sampling} for sampling exactly from the multivariate Jacobi ensemble with parameters $\lrabs{a^i}, \lrabs{b^i} \leq \frac{1}{2}$, for all $1\leq i \leq d$.

        More specifically, we consider sampling exactly from the projection $\DPP(\mu, K_N)$ where
        \begin{itemize}[wide=0pt]
            \item $\mu(\diff x) = \omega(x) \diff x$, with
            \begin{equation}
            \label{eq:app_base_measure}
                \omega(x)
                = \prod_{i=1}^d
                    \omega^i(x^i),
            \quad\text{where}\quad
                \omega^i(z)
                    = \prod_{i=1}^d
                        (1-z)^{a^i}
                        (1+z)^{b^i},
            \quad\text{and}\quad
            \lrabs{a^i}, \lrabs{b^i} \leq \frac{1}{2}\cdot
            \end{equation}
            \item $K_N(x, y)
                    = \sum_{\mathfrak{b}(b)=0}^{N-1}
                        \phi_k(x) \phi_k(y)$, with
            \begin{equation}
            \label{eq:app_multivariate_orthogonal_polynomials}
                \phi_k(x)
                    = \prod_{i=1}^d
                        \phi_{k^i}^i(x^i),
                \quad\text{where}\quad
                \int_{-1}^{1} \phi_{u}^i(z) \phi_{v}^i(z) \omega^i(z) \diff z
                    = \delta_{uv}.
            \end{equation}
        \end{itemize}
        As an illustration, \Figref{fig:app_sample_1000_points_2D} displays a sample of a $d=2$ Jacobi ensemble with $N=1000$ points.
        \begin{figure*}[ht]
            \centering
            \includegraphics[width=0.3\textwidth]
                {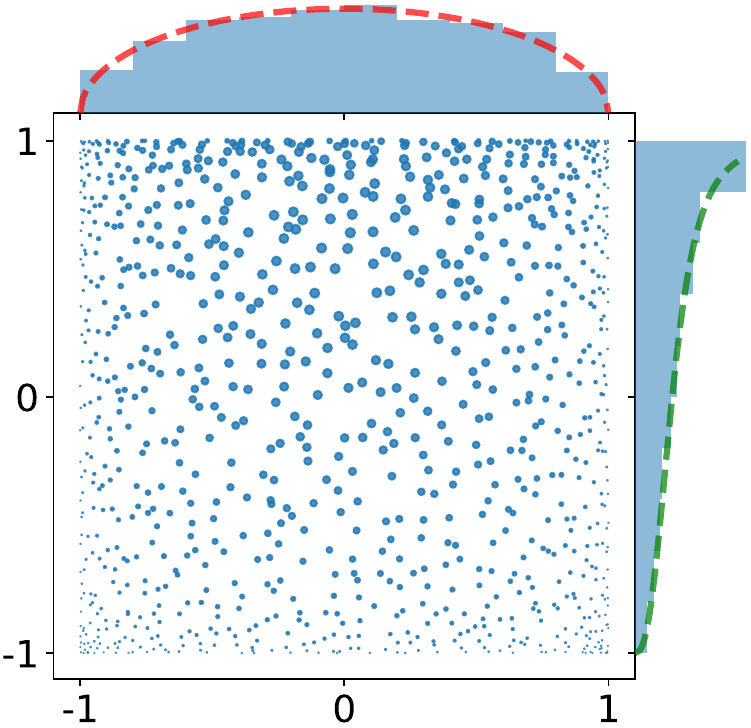}
            \includegraphics[width=0.3\textwidth]
                {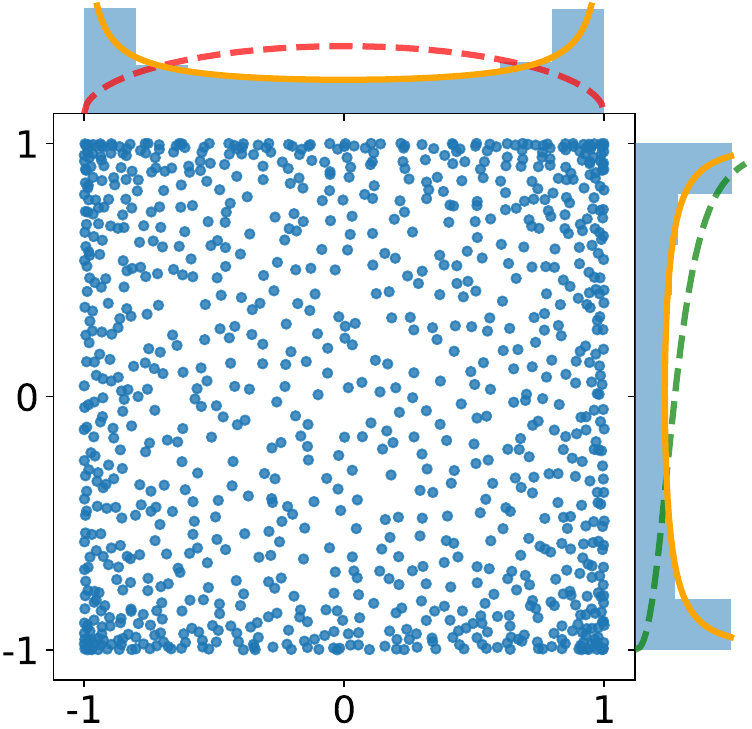}
            \includegraphics[width=0.3\textwidth]
                {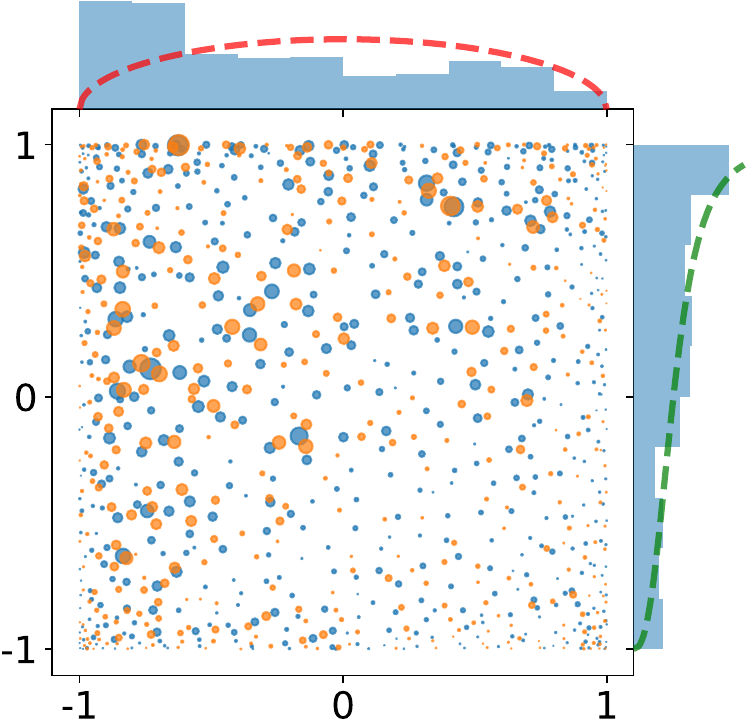}
            \vspace{-1em}
            \caption{Middle: a sample of a 2D Jacobi ensemble with $N=1000$ points.
            The normalized reference densities, proportional to \textcolor{red}{$(1-x)^{a^1} (1+x)^{b^1}$} and \textcolor{Green}{$(1-y)^{a^2} (1+y)^{b^2}$}, are displayed in dashed lines.
            The empirical marginal densities which converges to the arcsine density \textcolor{orange}{$\weq(x)=\frac{1}{\pi\sqrt{1-x^2}}$} is plotted in solid line.
            Left: we plot the same sample where the disk centered at $\bx_n$ now has now an area proportional to the weight $1/K_N(\bx_n,\bx_n)$ in $\Iauth{BH}_N(f)$ in \eqref{eq:BH_estimator}. Observe that these weights serve as a proxy for the reference measure, like Gaussian quadrature.
            Right: the weight in $\Iauth{EZ}_N(f)$ given by \eqref{eq:app_weights_EZ_estimator}; observe that they can be either positive or negative.
            The histogram of the absolute value of the weights is plotted on the marginal axes}
            \label{fig:app_sample_1000_points_2D}
            \vspace{-1em}
        \end{figure*}
        Our sampling scheme is an instance of the generic chain-rule-based procedure of \citet[Algorithm 18]{HKPV06} where the knowledge of the eigenfunctions can be leveraged, see also \citet[Algorithm 1]{LaMoRu15}.
        In our case, sampling $N$ points in dimension $d$, requires an expected total number of rejections of order $2^d N \log(N)$.
        As mentioned in \Secref{sub:sampling}, to sample from $\lrcb{\bfx_{1}, \dots, \bfx_{N} } \sim \DPP(\mu, K_N)$ it is enough to sample $\lrp{\bfx_{1}, \dots, \bfx_{N} }$ and forget the order the points were selected.
        Starting from the two formulations \eqref{eq:joint_distribution_projection_dpp} and \eqref{eq:joint_distribution_projection_dpp_geometric} of the joint distribution, the chain rule scheme can be derived from two different perspectives.
        Either by expressing the determinant $\det\lrp{K_N(x_p,x_n)}_{p,n=1}^N$ using Schur complements\\[-1em]
        \begin{align}
            &\eqref{eq:joint_distribution_projection_dpp}
            =\frac1{N!}
                \det\lrp{K_N(x_p,x_n)}_{p,n=1}^N
                \,\prod_{n=1}^{N}
                    \omega(x_n)
                    \diff x_n
            \label{eq:app_chain_rule_schur}
            \\
            &=
            \frac{K_N(x_1, x_1)}{N}
            \omega(x_1) \diff x_1
            \prod_{n=2}^{N}\omega(x_n) \diff x_n
                \frac{K_N(x_n,x_n)
                                -   {\bfK}_{n-1}(x_n)^{\top}
                                    {\bfK}_{n-1}^{-1}
                                    {\bfK}_{n-1}(x_n)}{N-(n-1)}
                \omega(x_n) \diff x_n,
            \nonumber
        \end{align}
        where
        $\bfK_{n-1}(\cdot)
          = \lrp{K_N(x_1,\cdot), \dots, K_N(x_{n-1},\cdot)}^{\top}$, and
        $\bfK_{n-1}
          = \lrp{K_N(x_p,x_q)}_{p,q=1}^{n-1}$.
        Or geometrically using the base$\times$height formula to express $(\det\bfPhi(x_{1:N}))^2$ as the squared volume of the parallelotope spanned by $\Phi(x_{1}), \dots, \Phi(x_{N})$
        \begin{align}
            \eqref{eq:joint_distribution_projection_dpp_geometric}
            &=\frac1{N!}
            \Vol^2\lrp{\Phi(x_{1}), \dots, \Phi(x_{N}) }
            \prod_{n=1}^{N}
                \omega(x_n) \diff x_n
            \nonumber\\
            &=\frac{\|\Phi(x_1)\|^2}{N}
                \omega(x_1) \diff x_1
                \prod_{n=2}^{N}
                    \frac{\dist^2\!\big(\Phi(x_n), \Span\lrcb{\Phi(x_p)}_{p=1}^{n-1}\big)}
                    {N-(n-1)}
                \omega(x_n) \diff x_n.
            \label{eq:app_chain_rule_geometric}
        \end{align}
        Note that, contrary to \eqref{eq:app_chain_rule_geometric}, the formulation \eqref{eq:app_chain_rule_schur} does not require a priori knowledge of the eigenfunctions of the projection kernel $K_N$.

        Like \citet{BaHa19}, we sample each conditional in turn using rejection sampling with the same proposal distribution and rejection bound.
        But where \citet{BaHa19} use the formulation \eqref{eq:app_chain_rule_schur} of the chain rule we consider the geometrical perspective \eqref{eq:app_chain_rule_geometric}.
        This allows for a implementation that is simpler (no need to update $\bfK_{n-1}^{-1}$), fully vectorized, and more interpretable: akin to a sequential Gram-Schmidt orthogonalization of the feature vectors $\Phi(x_{1}), \dots, \Phi(x_{N})$.

        Moreover, contrary to \citet{BaHa19} who take $\weq(x) \diff x$ as proposal to sample from the each of the conditionals, we use a two-layer rejection sampling scheme.
        We rather sample from the $n$-th conditional using the marginal distribution $N^{-1} K_N(x, x) \omega(x) \diff x$.
        This choice of proposal allows us to reduce the number of (costly) evaluations of
        the acceptance ratio.

        The rejection constant associated to the $n$-th conditional in \eqref{eq:app_chain_rule_schur} reads
        \begin{align}
            &\frac{ \lrp{N-(n-1)}^{-1}
                    \lrp{K_N(x,x)
                        -   {\bfK}_{n-1}(x)^{\top}
                            {\bfK}_{n-1}^{-1}
                            {\bfK}_{n-1}(x)}
                    \cancel{\omega(x)}}
                    {N^{-1} K_N(x,x) \cancel{\omega(x)}}
                \nonumber\\
            &\quad= \frac{N}{N-(n-1)}
                \frac{K_N(x,x)
                        -   {\bfK}_{n-1}(x)^{\top}
                            {\bfK}_{n-1}^{-1}
                            {\bfK}_{n-1}(x)}
                      {K_N(x,x)}
            \leq \frac{N}{N-(n-1)}\PeriodBin
            \label{eq:app_domination_conditionals}
        \end{align}

        The marginal distribution itself is sampled using the same proposal $\weq(x) \diff x$ and rejection constant as \citet{BaHa19}.
        However, we further reduce the number of computations by considering $N^{-1} K_N(x, x) \omega(x) \diff x$ as a mixture, see \Secref{ssub:generate_samples_from_the_marginal_distribution}

        \subsubsection{Generate samples from the marginal distribution} 
        \label{ssub:generate_samples_from_the_marginal_distribution}

            First, observe that the marginal density can be written as a mixture of $N$ probability densities where each component is assigned the same weight $1/N$
            \begin{equation}
            \label{eq:app_marginal_distribution}
                \frac{1}{N}
                    K_N(x, x)
                    \omega(x)
                =
                \frac{1}{N}
                \lsum{\mathfrak{b}(k)=0}{N-1}
                    \phi_{k}(x)^2
                    \omega(x).
            \end{equation}
            Thus, sampling from \eqref{eq:app_marginal_distribution} can be done in two steps:
            \begin{enumerate}[(i)]
                \item select a multi-index $k=\mathfrak{b}^{-1}(n)$ with $n$ drawn uniformly at random in $\lrcb{0, \dots, N-1}$
                \item sample from $\phi_k(x)^2 \omega(x) \diff x$ \label{step:app_sample_marginal_step_2}
            \end{enumerate}
            We perform Step \ref{step:app_sample_marginal_step_2} using rejection sampling with proposal distribution
            \begin{equation}
                \label{eq:app_equilibrium_measure}
                \weq(x) \diff x
                    = \prod_{i=1}^d
                        \frac{1}{\pi\sqrt{1-(x^i)^2}}
                        \diff x^i,
            \end{equation}
            which corresponds to the limiting marginal distribution of the multivariate Jacobi ensemble as $N$ goes to infinity; see \cite[Section 3.11]{Sim11} and \Figref{fig:app_sample_1000_points_2D}.
            The acceptance ratio writes
            \begin{align}
                \frac{\phi_k(x)^2 \omega(x)}{\weq(x)}
                &\lequal{\eqref{eq:app_equilibrium_measure}}
                        {\eqref{eq:app_multivariate_orthogonal_polynomials}
                         \eqref{eq:app_base_measure}}
                    \lprod{i=1}{d}
                    \frac{
                            \phi_{k^i}^i(x^i)^2
                        \times
                            (1-x^i)^{a^i} (1+x^i)^{b^i}
                        }
                        {
                            \pi^{-1}
                            (1-x^i)^{-\frac12} (1+x^i)^{-\frac12}}
                    \nonumber\\
                &=
                    \lprod{i=1}{d}
                        \pi
                        (1-x^i)^{a^i+\frac12} (1+x^i)^{b^i+\frac12}
                        \phi_{k^i}^i(x^i)^2.
                \label{eq:app_domination_marginal}
            \end{align}
            Each of the terms that appear in \eqref{eq:app_domination_marginal} can be bounded using the following recipe:
            \begin{enumerate}[(a)]
                \item For $k^i=0$, $\phi_{0}^i$ is constant and the orthonormality w.r.t.\,$(1-x)^{a^i}(1+x)^{b^i} \diff x$ yields
                \begin{equation}
                    (\phi_{0}^i)^2
                        \int_{-1}^1 (1-x)^{a^i}(1+x)^{b^i} \diff x
                        = 1
                        \Longleftrightarrow
                        (\phi_{0}^i)^2
                        = \frac{1}{2^{a^i+b^i+1} B(a^i+1, b^i+1)}\CommaBin
                \end{equation}
                so that the corresponding term in \eqref{eq:app_domination_marginal} becomes
                \begin{equation}
                    \frac{\pi (1-x)^{a^i+\frac12} (1+x)^{b^i+\frac12}}
                         {2^{a^i+b^i+1} B(a^i+1, b^i+1)}
                    \leq
                        \frac{\pi (1-m)^{a^i+\frac12} (1+m)^{b^i+\frac12}}
                         {2^{a^i+b^i+1} B(a^i+1, b^i+1)}
                    \triangleq C_{k^i=0}
                    \leq 2,
                \end{equation}
                where
                $m
                    = \argmax\limits_{-1\leq x \leq 1}~(1-x)^{a^i+\frac12} (1+x)^{b^i+\frac12}
                    = \begin{cases}
                        0, &\text{ if } a^i=b^i=-\frac12,\\
                        \frac{b^i-a^i}{a^i+b^i+1}\CommaBin &\text{otherwise}.
                    \end{cases}$

                \item For $k^i\geq1$, we use the bound $C_{k^i\geq1}$ \eqref{eq:app_Chow_bound} provided originally by \citet{ChGaWo94}.
                As mentioned by \citet{Gau09}, this bound is probably maximal for $k^i=1$ and parameters $a^i\approx-0.0691, b^i=1/2$, with value $\approx 0.64297807\pi \approx 2.02$.
            \end{enumerate}
            Finally, the expected number of rejections to perform Step \ref{step:app_sample_marginal_step_2} is equal to $\prod_{i=1}^{d} C_{k^i}$ which is of order $2^d$, and the expected total number of rejections of the chain rule \eqref{eq:app_chain_rule_schur} is of order
            \begin{equation}
                \label{eq:expected_total_number_of_rejections}
                \sum_{n=1}^{N} 2^d \frac{N}{N-(n-1)}
                = 2^d N \sum_{n=1}^{N} \frac{1}{n}
                \approx 2^d N \log(N).
            \end{equation}

            \begin{proposition}
            \label{propo:app_Chow_bound}
                \citep[Equation 1.3]{Gau09}
                Let $(\phi_k)_{k\geq 0}$ be the (univariate) orthonormal polynomials w.r.t.\,$(1-x)^a (1+x)^b \diff x$ with $\lrabs{a}\leq \frac12, \lrabs{b}\leq \frac12$.
                Then, for any $x \in [-1, 1]$ and $k\geq 1$,
                \begin{equation}
                \label{eq:app_Chow_bound}
                    \hspace{-1em}
                    \pi
                        (1-x)^{a+\frac{1}{2}}
                        (1+x)^{b+\frac{1}{2}}
                    \phi_{k}(x)^{2}
                    \leq
                        \frac
                            {2~\Gamma(k+a+b+1)~\Gamma(k+\max(a,b)+1)}
                            {k!~(k+\frac{a+b+1}{2})^{2 \max(a,b)}~\Gamma(k+\min(a,b)+1)}\PeriodBin
                \end{equation}
            \end{proposition}


        \subsubsection{Empirical timing and number of rejections} 
        \label{ssub:timings_and_number_of_rejections}

            In \Figref{fig:app_timings_rejections} we illustrate the following observations.
            Computing the acceptance ratio \eqref{eq:app_domination_conditionals} requires to propagate the recurrence relations up to order $\sqrt[d]{N}$.
            Thus, for a given number of points $N$, the larger the dimension, the smaller the depth of the recurrence.
            This could hint that, evaluating the kernel \eqref{eq:kernel_multivariate_separable_OPE} becomes cheaper as $d$ increases.
            However, the rejection rate also increases, so that in practice, it is not cheaper to sample in larger dimensions because the number of rejections dominates.
            In the particular case of dimension $d=1$, samples are generated using the fast and rejection-free tridiagonal matrix model of \citet[Theorem 2]{KiNe04}.
            This grants huge time savings compared to the acceptance-rejection method.

            Finally, some remarks are in order.
            Sampling from the $n$-th conditional distribution using rejection sampling is common practice \citep[Section 2.4.2]{LaMoRu15}.
            However, tailored proposals with tight rejection constants are required \citep[Appendices E-F]{LaMoRu15}.
            Taking the marginal distribution $N^{-1} K_N(x, x) \omega(x) \diff x$ as proposal yields a $N/(N-(n-1))$ rejection constant and applies in the general case.
            Nevertheless, it remains to sample from this marginal distribution
            Rejection sampling might be a first option to sample from $N^{-1} K_N(x, x) \omega(x) \diff x$, but when the eigenfunctions are available it could be another option to see it as a mixture (cf. \Secref{ssub:generate_samples_from_the_marginal_distribution}), where good proposals for each $\phi_k(x)^2 w(x) \diff x$ are required.

            In the case of (multivariate) orthogonal polynomial ensembles (cf. \Secref{sub:multivariate_jacobi_ensemble}), evaluations of $K_N(x, y)$ \eqref{eq:kernel_multivariate_separable_OPE} can be performed using the Gram representation \eqref{eq:kernel_projection_DPP_Gram_formulation}, $K_N(x, y)=\Phi(x)^{\top}\Phi(y)$ and one can leverage the three-term recurrence relations satisfied by each of the univariate Jacobi polynomials $(\phi_{\ell}^i)_{\ell}$.
            This is what we do in our special case, we use the dedicated function \texttt{scipy.special.eval\_jacobi} to evaluate, up to depth $\sqrt[d]{N}$, the three-term recurrence relations satisfied by each of the univariate Jacobi.
            Instead of calling the recursive routine internally to evaluate $\Phi(x)$, the corresponding $d\sqrt[d]{N}$ univariate polynomials or $N$ multivariate polynomials could be stored in some way and evaluated pointwise on the fly.
            The preprocessing time and the memory required would increase but it might accelerate the evaluation of $\Phi(x)$.




        \begin{figure*}[ht]
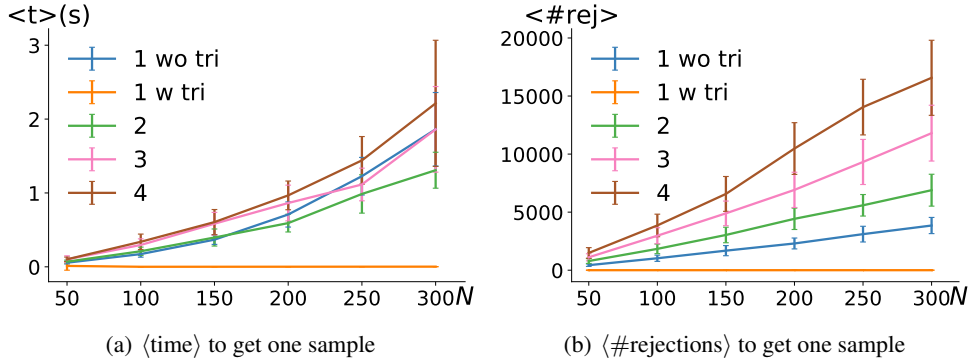

            \centering
            \subfigure[$\lrsp{\text{time}}$ to get one sample]{
            \label{fig:app_timings}
            \includegraphics[width=0.47\textwidth]
                {images/timings_chebychev_d=0-1-2-3-4D_with_legend.pdf}
            }
            \hspace{-1em}
            \subfigure[$\lrsp{\#\text{rejections}}$ to get one sample]{
            \label{fig:app_rejections}
            \includegraphics[width=0.47\textwidth]
                {images/rejections_chebychev_d=0-1-2-3-4D_with_legend.pdf}
            }
            \caption{
            $a^i, b^i=-1/2$, the colors and numbers correspond to the dimension.
            For $d=1$, the tridiagonal model (tri) of \citeauthor{KiNe04} offers tremendous time savings.
            (b) The total number of rejections grows as $2^d N \log(N)$ \eqref{eq:expected_total_number_of_rejections}.
            }
            \label{fig:app_timings_rejections}
        \end{figure*}






\clearpage
\section{Experiments} 
\label{sec:app_XPs}

    \subsection{Reproducing the bump example}
    \label{sub:app_XP_bump}

        In \Secref{sub:XP_bump}, we reproduce the experiment of \citet[Section 3]{BaHa19} where they illustrate the behavior of $\Iauth{BH}_N$ on a unimodal, smooth bump function:
        \begin{equation}
            f(x)
            =
            \prod_{i=1}^d
                \exp\lrp{-\frac{1}{1- \varepsilon - (x^i)^2}}
                \indic_{(-\sqrt{1-\varepsilon}, \sqrt{1- \varepsilon})}(x^i).
        \end{equation}
        We take $\varepsilon = 0.05$. For each value of $N$, we sample $100$ times from the same multivariate Jacobi ensemble with i.i.d. uniform parameters on $[-1/2,1/2]$, compute the resulting $100$ values of each estimator, and plot the two resulting sample variances.
        In addition, in \Figref{fig:app_bump_KS} we test the potential hope for a CLT for $\Iauth{EZ}_N$ and compare with $\Iauth{BH}_N$ for which the CLT \eqref{eq:BH_CLT} holds, in the regime $N=300$.

        \begin{figure*}[ht]
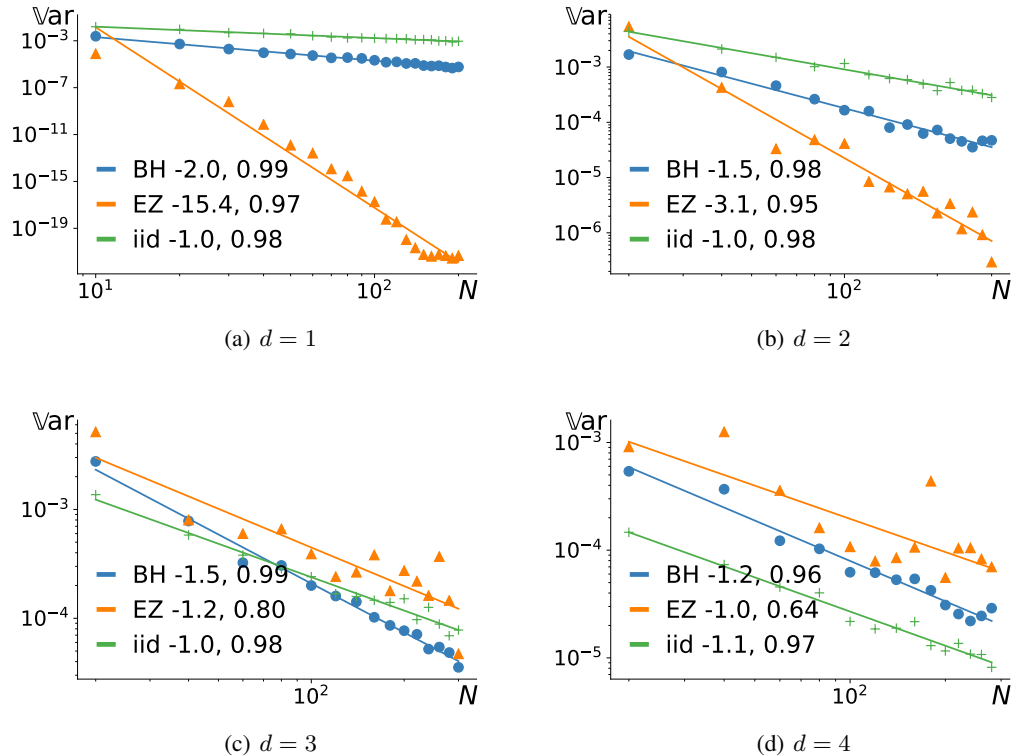

            \subfigure[$d=1$]{
            \label{fig:app_bump_1D}
            \includegraphics[width=.49\textwidth]
                {images/bump_1D_100_repeats.pdf}
            }
            \subfigure[$d=2$]{
            \label{fig:app_bump_2D}
            \includegraphics[width=.49\textwidth]
                {images/bump_2D_100_repeats.pdf}
            }\\
            \subfigure[$d=3$]{
            \label{fig:app_bump_3D}
            \includegraphics[width=.49\textwidth]
                {images/bump_3D_100_repeats.pdf}
            }
            \subfigure[$d=4$]{
            \label{fig:app_bump_4D}
            \includegraphics[width=.49\textwidth]
                {images/bump_4D_100_repeats.pdf}
            }
            \caption{Reproducing the bump function ($\varepsilon=0.05$) experiment of \citet{BaHa19}, cf.\,\Secref{sub:XP_bump}.
            Observe the expected variance decay of order $1/N^{1+1/d}$ for BH.
            Although vanilla Monte Carlo becomes competitive for small $N$ as $d$ increases, its variance decay is of order $1/N \geq 1/N^{1+1/d}$.
            Thus, there will always be meeting point, for some $N^*$, after which the variance of BH will be smaller.
            For $d=1$, EZ has almost no variance for $N\geq 100$: the bump function is extremely well approximated by a polynomials of degree $N\geq 100$.}
            \label{fig:app_bump}
        \end{figure*}
        \begin{figure*}[ht]
            \subfigure[$d=1$]{
            \label{fig:app_bump_1D_KS}
            \includegraphics[width=.49\textwidth]
                {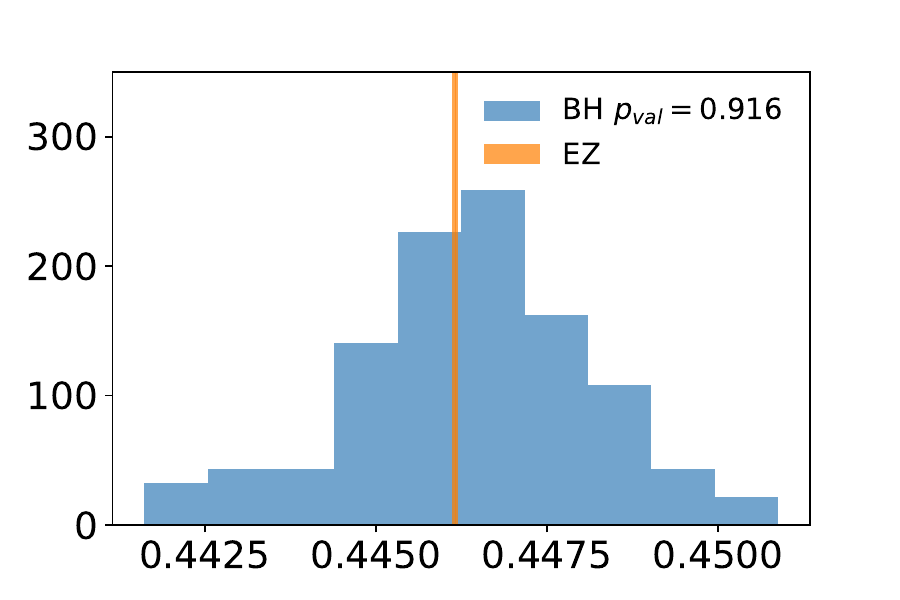}
            }
            \subfigure[$d=2$]{
            \label{fig:app_bump_2D_KS}
            \includegraphics[width=.49\textwidth]
                {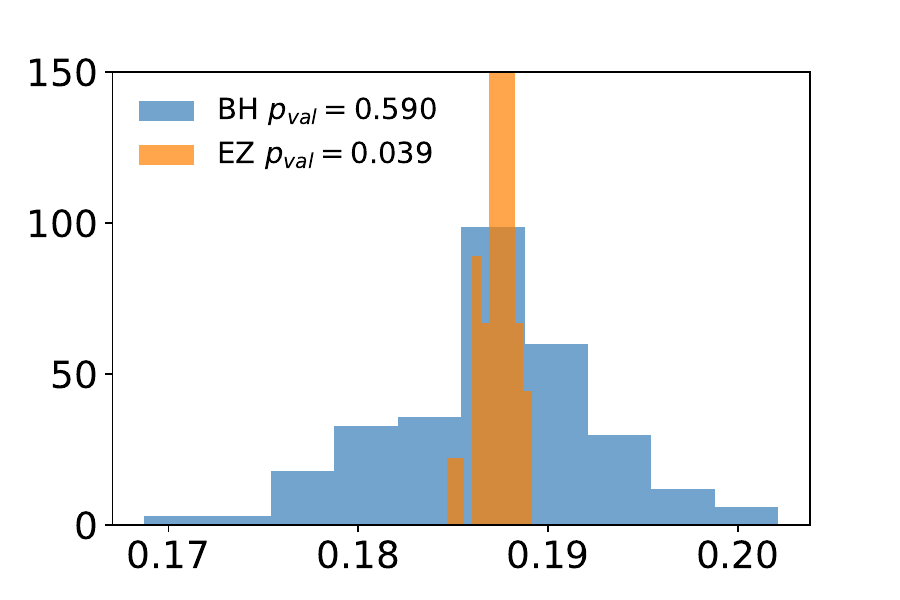}
            }\\
            \subfigure[$d=3$]{
            \label{fig:app_bump_3D_KS}
            \includegraphics[width=.49\textwidth]
                {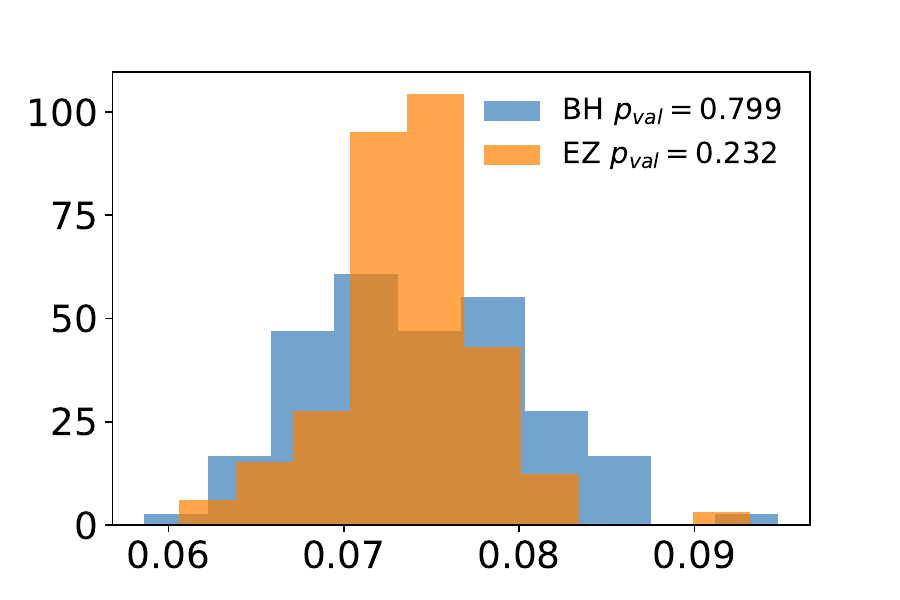}
            }
            \subfigure[$d=4$]{
            \label{fig:app_bump_4D_KS}
            \includegraphics[width=.49\textwidth]
                {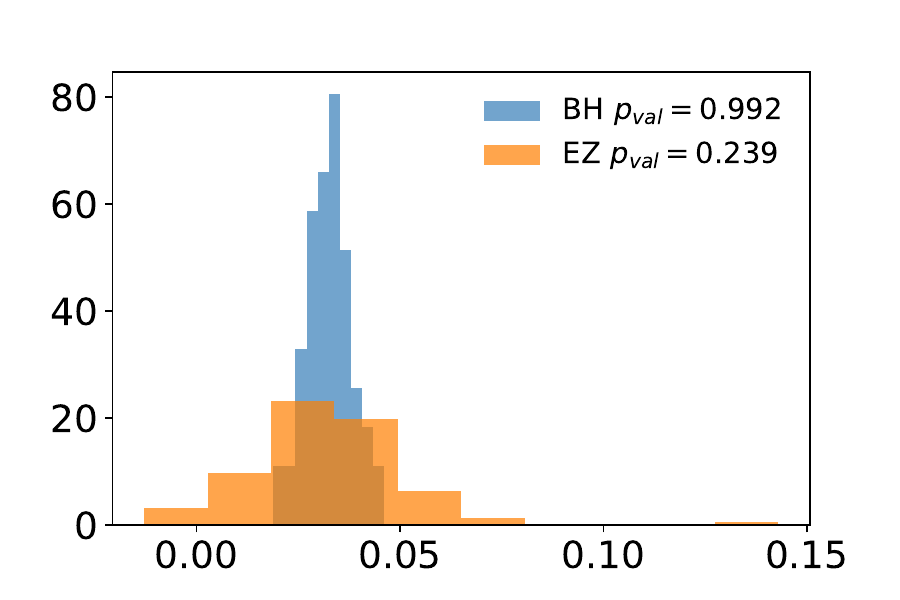}
            }
            \caption{Histogram of $100$ independent estimates $\Iauth{BH}_N$ and $\Iauth{EZ}_N$ of the integral of the bump function ($\varepsilon=0.05$) with $N=300$ and associated p-value of Kolmogorov-Smirnov test, cf.\,\Secref{sub:XP_bump}.
            The fluctuations of BH confirm to be Gaussian (cf.\,CLT \eqref{eq:BH_CLT}).
            (a) the bump function is extremely well approximated by a polynomial of degree $300$ hence $\Iauth{EZ}_N$ has almost no variance.
            (b)-(c)-(d) A few outliers seem to break the potential Gaussianity of $\Iauth{EZ}_N(f)$.
            (d) $\Iauth{EZ}_N(f)$ does not preserve the sign of the integrand.}
            \label{fig:app_bump_KS}
        \end{figure*}


\clearpage

    \subsection{Integrating sums of eigenfunctions}
    \label{sub:app_XP_polynomials}

        \Figref{fig:app_decay_1/k_N=70_Ndpp} gives the results of the first setting set in \Secref{sub:XP_polynomials}, where we integrate a sum of $M=70$ kernel eigenfunctions.
        In this case, $EZ$ has zero variance once $N\geq M$, a performance that can be reached neither by BH nor vanilla Monte Carlo.

        \Figref{fig:app_decay_1/k_pursuit_Ndpp} illustrates the second setting, where the sum always has one more eigenfunction than there are points in the DPP samples.
        In this case, the conditions for the CLT of BH, cf \eqref{eq:BH_CLT}, are not met; there is no $1/N^{1+1/d}$ guarantee on the variance decay for BH estimator.
        The performance of BH and vanilla Monte Carlo are comparable.
        By construction, the variance of EZ decays as $1/N^2\leq 1/N$.
        Thus, there will always be meeting point, for some $N^*$, after which the variance of EZ will be smaller than vanilla Monte Carlo.
        \begin{figure*}[ht]
            \subfigure[$d=1$]{
            \label{fig:app_decay_1/k_N=70_Ndpp_1D}
            \includegraphics[width=\twofig]
                {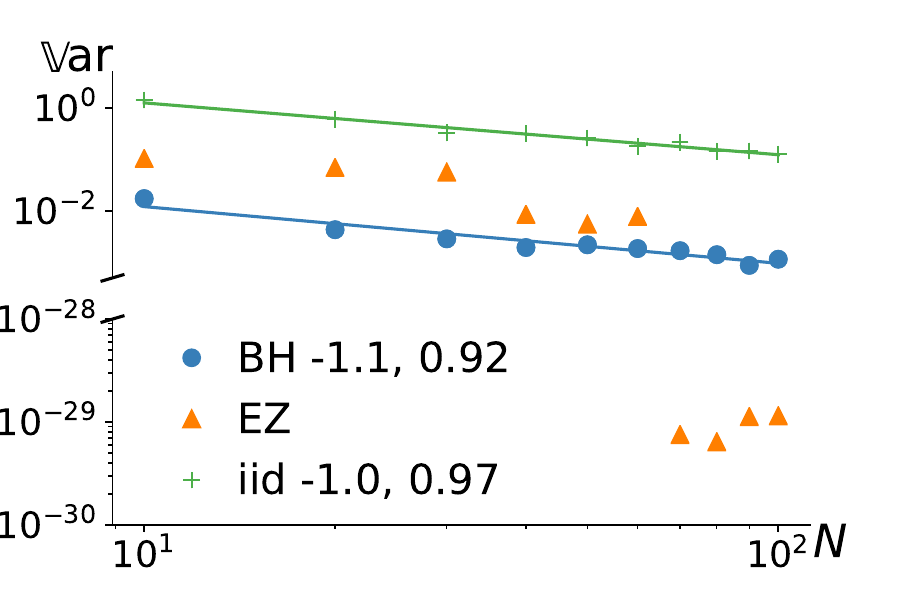}
            }
            \subfigure[$d=2$]{
            \label{fig:app_decay_1/k_N=70_Ndpp_2D}
            \includegraphics[width=\twofig]
                {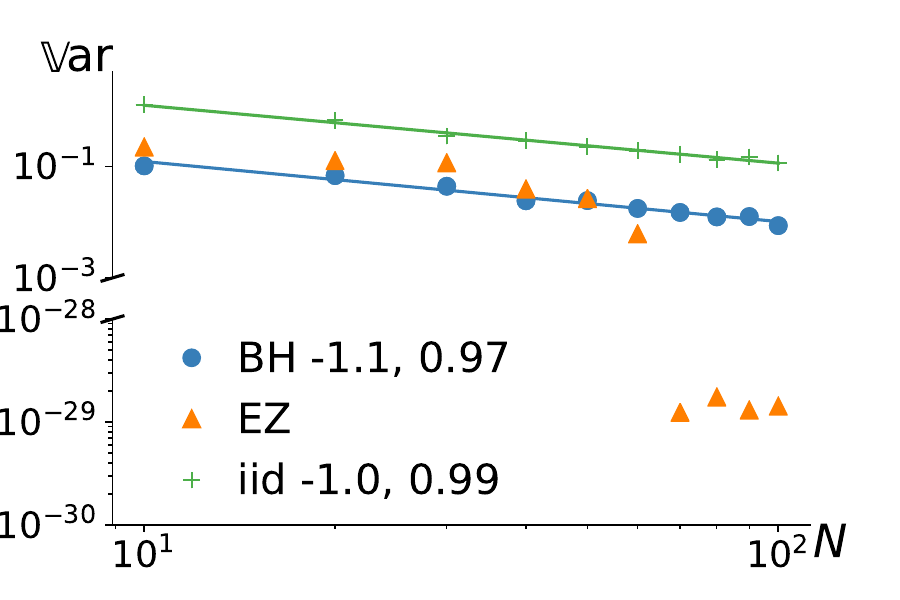}
            }\\
            \subfigure[$d=3$]{
            \label{fig:app_decay_1/k_N=70_Ndpp_3D}
            \includegraphics[width=\twofig]
                {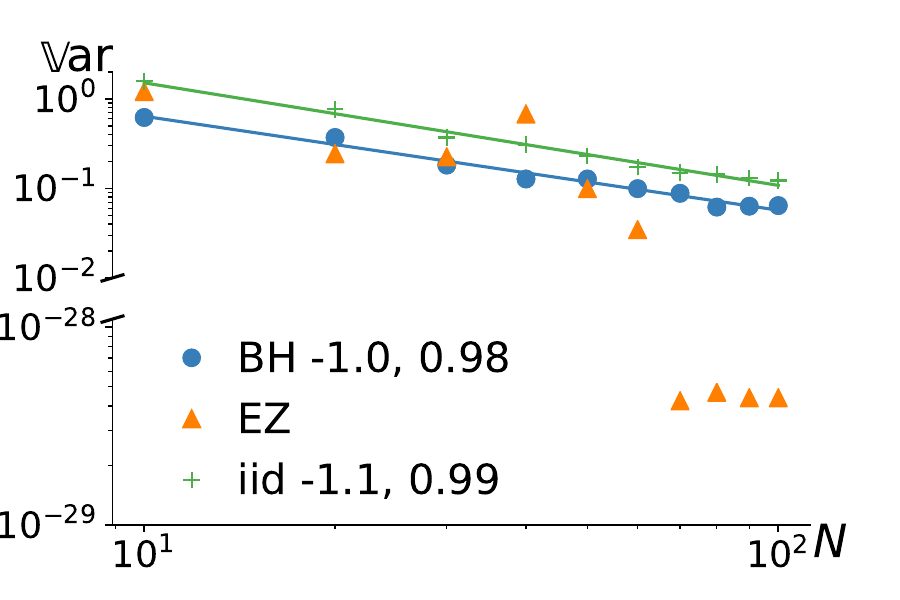}
            }
            \subfigure[$d=4$]{
            \label{fig:app_decay_1/k_N=70_Ndpp_4D}
            \includegraphics[width=\twofig]
                {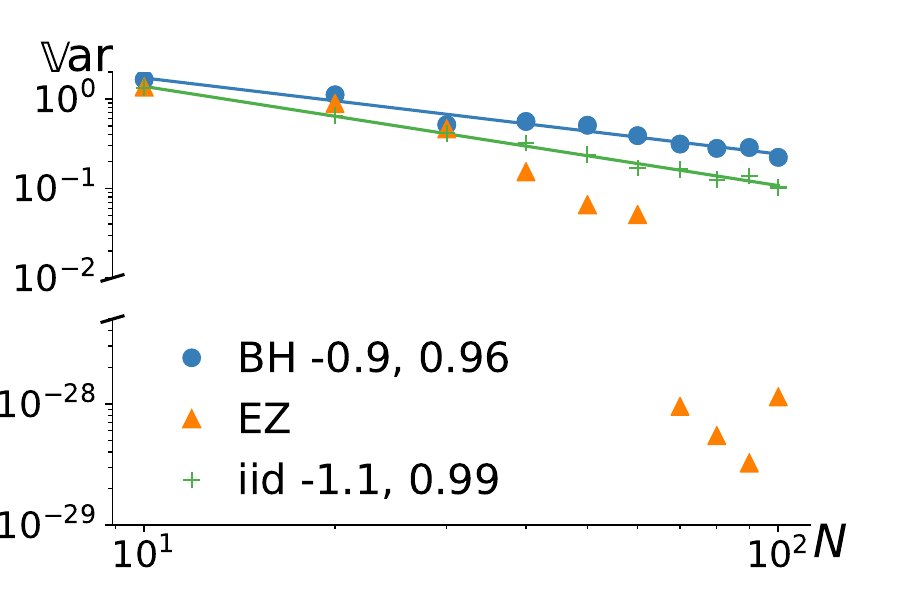}
            }
            \caption{Comparison of $\Iauth{BH}_N$ and $\Iauth{EZ}_N$ integrating a finite sum of $70$ eigenfunctions of the DPP kernel as in \eqref{eq:f_xps_polys}, cf.\,\Secref{sub:XP_polynomials}.}
            \label{fig:app_decay_1/k_N=70_Ndpp}
        \end{figure*}

        \begin{figure*}[ht]
            \subfigure[$d=1$]{
            \label{fig:app_decay_1/k_pursuit_Ndpp_1D}
            \includegraphics[width=\twofig]
                {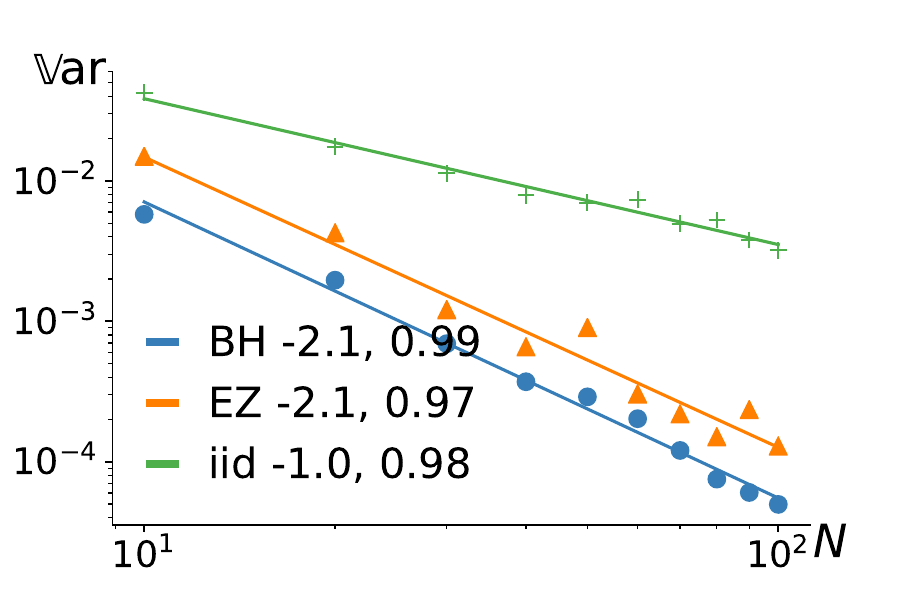}
            }
            \subfigure[$d=2$]{
            \label{fig:app_decay_1/k_pursuit_Ndpp_2D}
            \includegraphics[width=\twofig]
                {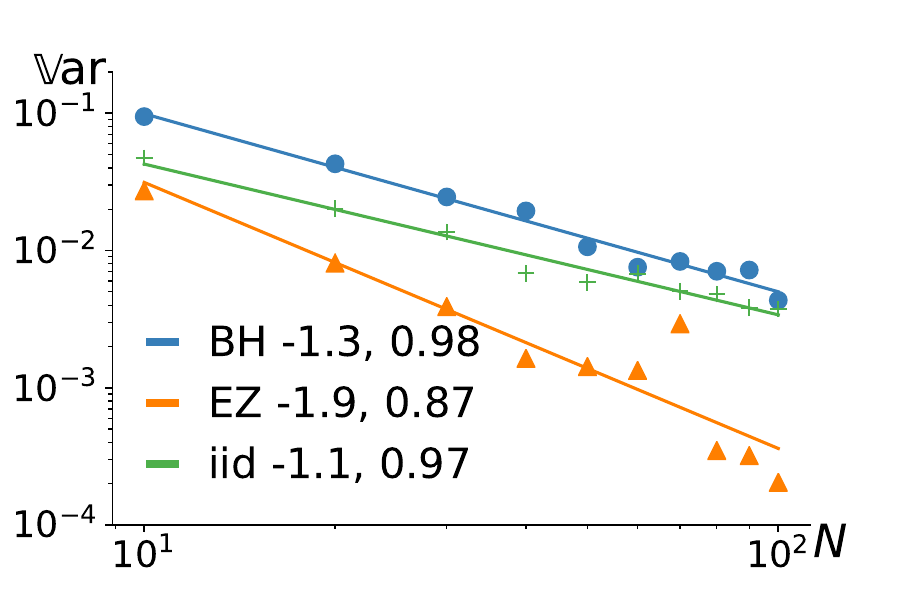}
            }\\
            \subfigure[$d=3$]{
            \label{fig:app_decay_1/k_pursuit_Ndpp_3D}
            \includegraphics[width=\twofig]
                {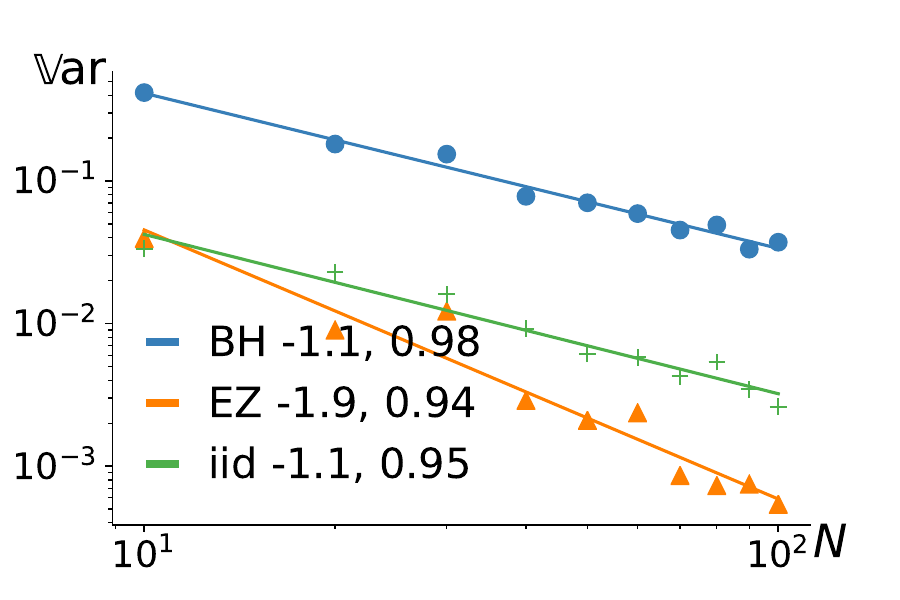}
            }
            \subfigure[$d=4$]{
            \label{fig:app_decay_1/k_pursuit_Ndpp_4D}
            \includegraphics[width=\twofig]
                {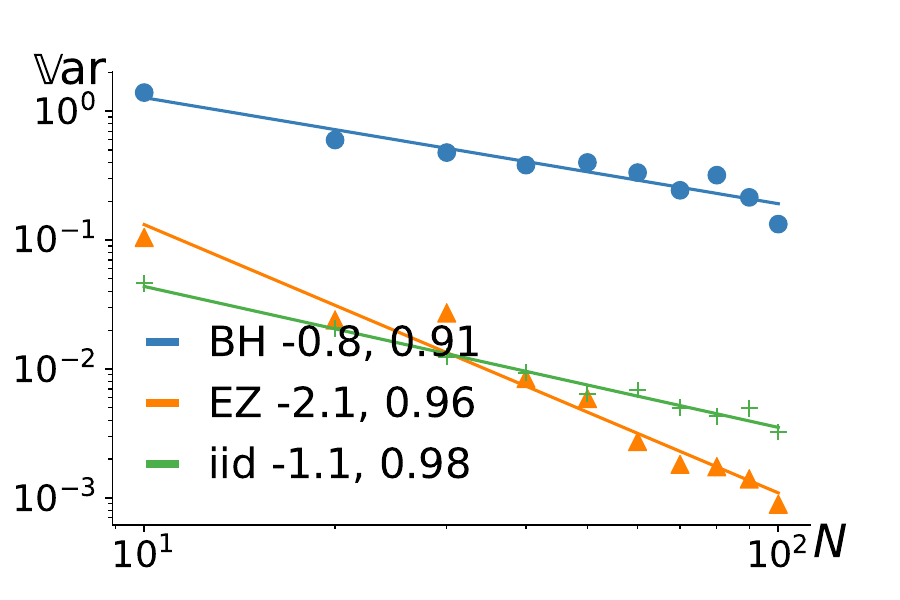}
            }
            \caption{Comparison of $\Iauth{BH}_N$ and $\Iauth{EZ}_N$ for a linear combination of $N+1$ eigenfunctions of the DPP kernel as in \eqref{eq:f_xps_polys}, cf.\,\Secref{sub:XP_polynomials}.}
            \label{fig:app_decay_1/k_pursuit_Ndpp}
        \end{figure*}


\clearpage

    We now consider cases where the guarantees of BH not EZ are unknown.

    \subsection{Integrating absolute value}
    \label{sub:app_XP_absolute}

        We consider estimating the integral
        \begin{equation}
            \int_{[-1, 1]^d}
                \prod_{i=1}^{d}
                    |x^i|
                    (1-x^i)^{a^i} (1+x^i)^{b^i}
                \diff x^i
        \end{equation}
        where $a^1, b^1=-\frac12$ and $a^i, b^i$ i.i.d.\,uniformly in $[-\frac12, \frac12]$, using BH \eqref{eq:BH_estimator} and EZ \eqref{eq:EZ_estimator} estimators.

        Results are given in \Figref{fig:app_absolute_value}.
        In dimension $d=1$, the absolute value is well approximated by its truncated Taylor series of low order and EZ performs very well, but as the dimension increases, its performance is more erratic.
        For $d\leq 2$, the performance of BH is smooth and better that vanilla Monte Carlo.
        In particular, for $d\leq 2$, the rate $1/N^{1+1/d}$ seems to hold for BH while the conditions for the CLT \eqref{eq:BH_CLT} are not satisfied.
        But it seems no longer true in larger dimension.
        \begin{figure*}[ht]
            \subfigure[$d=1$]{
            \label{fig:absolute_value_1D}
            \includegraphics[width=\twofig]
                {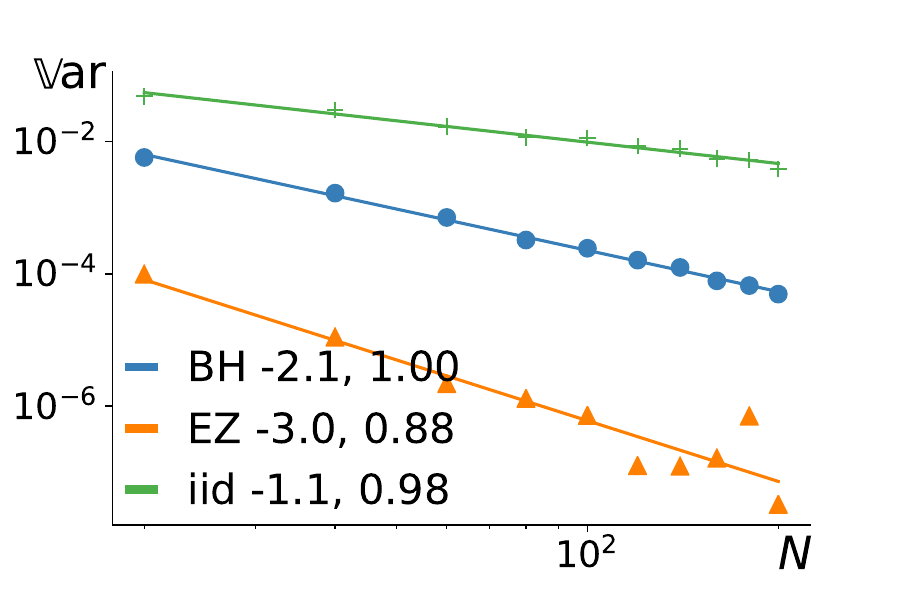}
            }
            \subfigure[$d=2$]{
            \label{fig:absolute_value_2D}
            \includegraphics[width=\twofig]
                {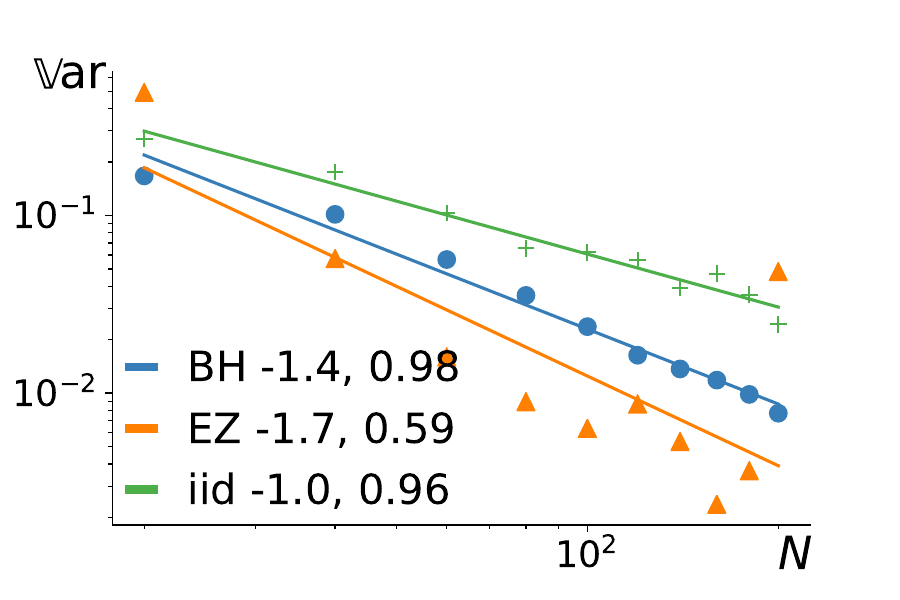}
            }\\
            \subfigure[$d=3$]{
            \label{fig:absolute_value_3D}
            \includegraphics[width=\twofig]
                {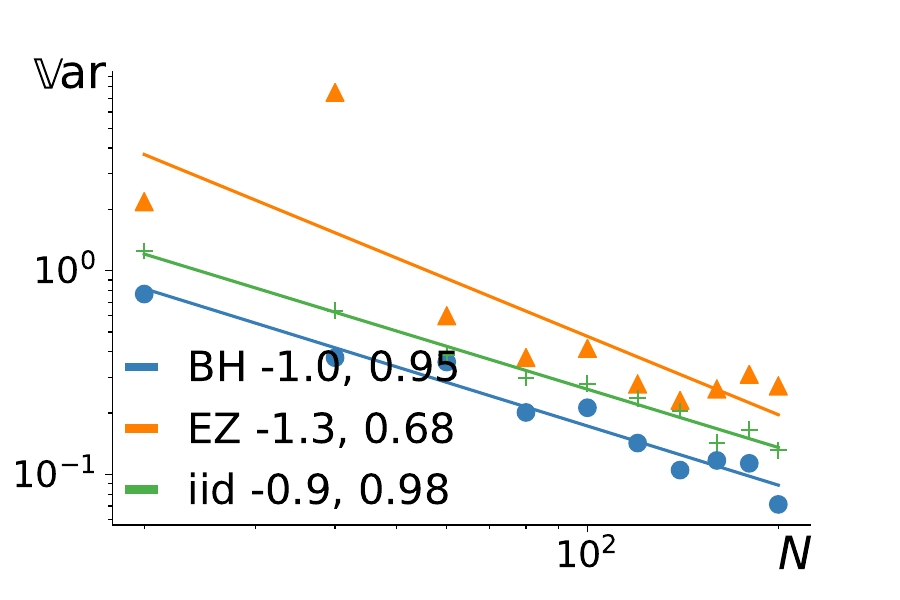}
            }
            \subfigure[$d=4$]{
            \label{fig:absolute_value_4D}
            \includegraphics[width=\twofig]
                {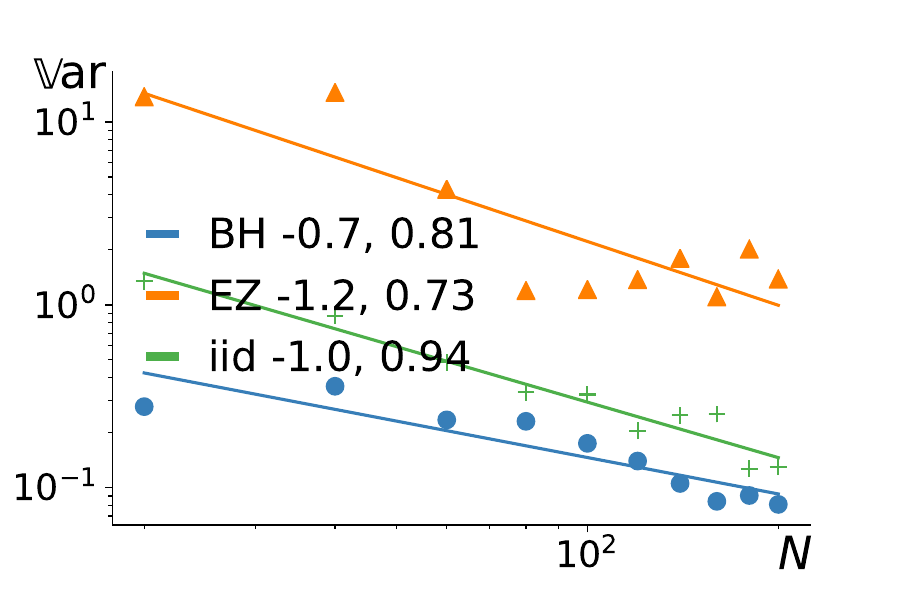}
            }
            \caption{Comparison of $\Iauth{BH}_N$ and $\Iauth{EZ}_N$ for absolute value, cf.\,\Secref{sub:further_experiments}.}
            \label{fig:app_absolute_value}
        \end{figure*}


\clearpage

    \subsection{Integrating Heaviside}
    \label{sub:app_XP_heaviside}

        Let $H(x)=\begin{cases}
            1, &\text{ if }x>0\\
            0, &\text{otherwise}
        \end{cases}$.
        We consider estimating the integral
        \begin{equation}
            \int_{[-1, 1]^d}
                \prod_{i=1}^{d}
                    2\lrp{H(x^i) - \frac{1}{2}}
                    (1-x^i)^{a^i} (1+x^i)^{b^i}
                \diff x^i
        \end{equation}
        where $a^1, b^1=-\frac12$ and $a^i, b^i$ i.i.d.\,uniformly in $[-\frac12, \frac12]$, using BH \eqref{eq:BH_estimator} and EZ \eqref{eq:EZ_estimator} estimators.

        Results are given in \Figref{fig:app_heaviside}.
        The EZ estimator behaves in a very erratic way; it does not seem robust to the discontinuity we have introduced.
        This can be explained by considering
        $H(x)=\frac12 \lim_{\epsilon\to 0} 1+\tanh{\frac{x}{\epsilon}}$ and taking the product of the Taylor series expansions of $\tanh$; the square of the coefficients in front of the monomials in such expansion become very large as $\epsilon\to 0$.
        One could expect better behavior for very large $N$.
        The performance of BH is smooth and the rate $1/N^{1+1/d}$ seems to hold despite the conditions for the CLT \eqref{eq:BH_CLT} are not satisfied.
        \begin{figure*}[ht]
            \subfigure[$d=1$]{
            \label{fig:heaviside_value_1D}
            \includegraphics[width=\twofig]
                {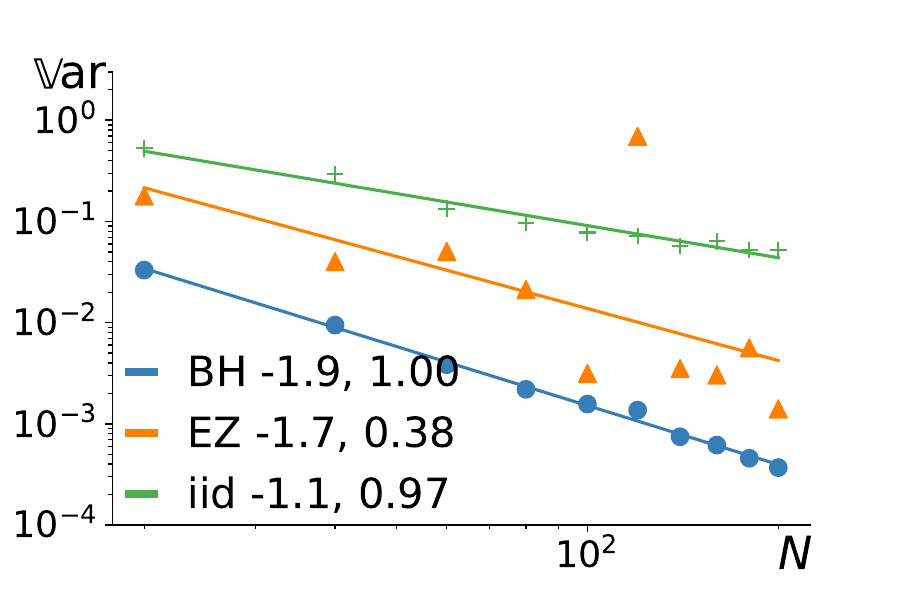}
            }
            \subfigure[$d=2$]{
            \label{fig:heaviside_value_2D}
            \includegraphics[width=\twofig]
                {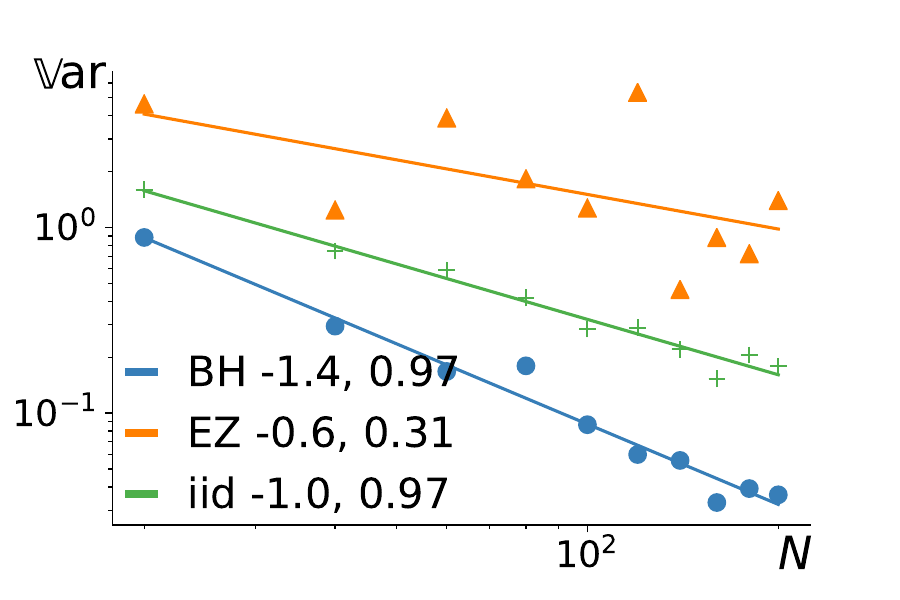}
            }\\
            \subfigure[$d=3$]{
            \label{fig:heaviside_value_3D}
            \includegraphics[width=\twofig]
                {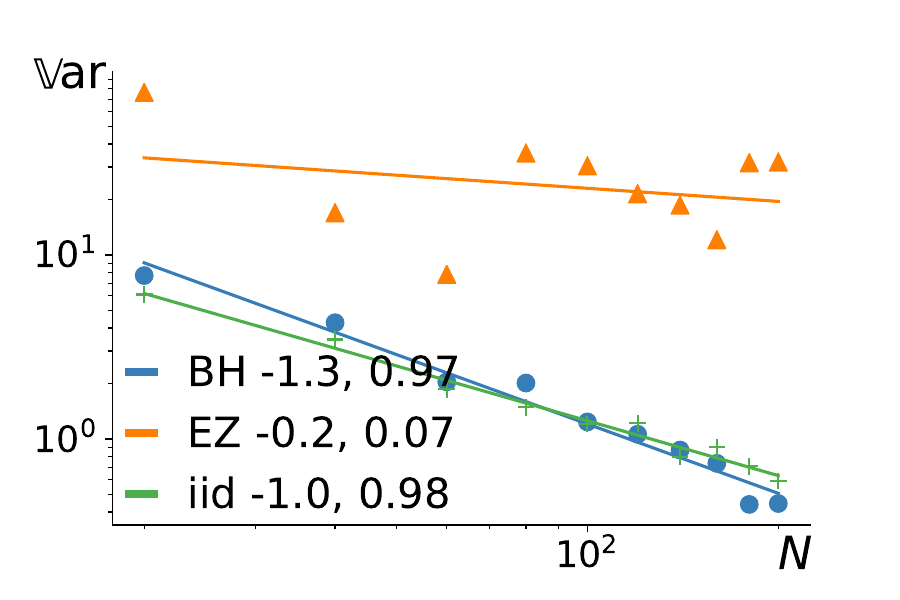}
            }
            \subfigure[$d=4$]{
            \label{fig:heaviside_value_4D}
            \includegraphics[width=\twofig]
                {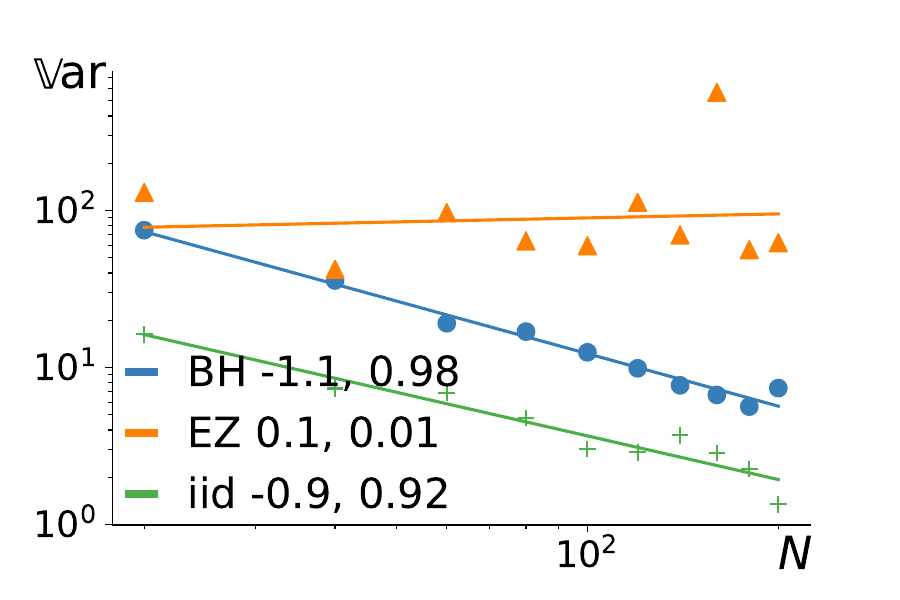}
            }
            \caption{Comparison of $\Iauth{BH}_N$ and $\Iauth{EZ}_N$ for Heaviside function, cf.\,\Secref{sub:further_experiments}.}
            \label{fig:app_heaviside}
        \end{figure*}


\clearpage

    \subsection{Integrating cosine}
    \label{sub:app_XP_cosine}

        We consider estimating the integral
        \begin{equation}
            \int_{[-1, 1]^d}
                \prod_{i=1}^{d}
                    \cos(\pi x^i)
                    (1-x^i)^{a^i} (1+x^i)^{b^i}
                \diff x^i
        \end{equation}
        where $a^1, b^1=-\frac12$ and $a^i, b^i$ i.i.d.\,uniformly in $[-\frac12, \frac12]$, using BH \eqref{eq:BH_estimator} and EZ \eqref{eq:EZ_estimator} estimators.

        Results are given in \Figref{fig:app_cosine}
        The EZ estimator behaves well for $d\leq 2$ but its performance deteriorates for $d\geq 3$.
        Indeed, the cross terms arising from the Taylor expansion of the different $\cos(\pi x^i)$ introduce monomials, associated to large coefficients, that do not belong to $\calH_N$.
        One could expect better behavior for very large $N$.
        For $d\leq 2$, the rate $1/N^{1+1/d}$ for BH seems to hold despite the conditions for the CLT \eqref{eq:BH_CLT} are not satisfied.
        For $d\geq 3$, BH and vanilla Monte Carlo behave similarly.
        \begin{figure*}[ht]
            \subfigure[$d=1$]{
            \label{fig:cosine_1D}
            \includegraphics[width=\twofig]
                {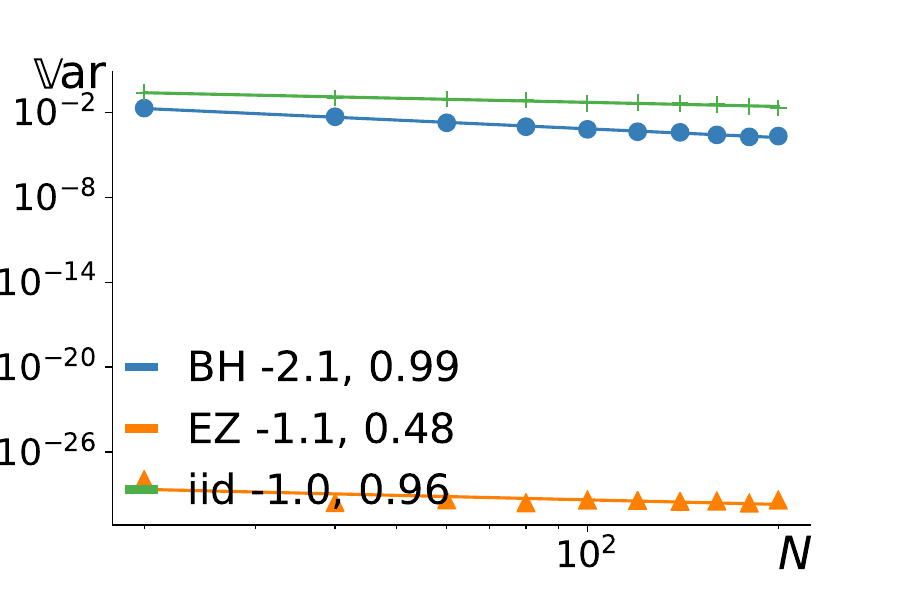}
            }
            \subfigure[$d=2$]{
            \label{fig:cosine_2D}
            \includegraphics[width=\twofig]
                {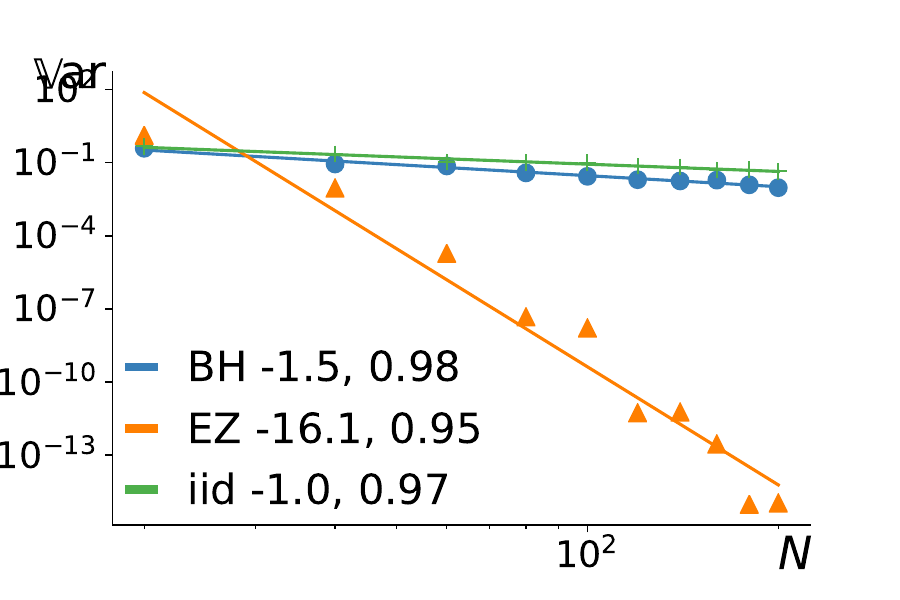}
            }\\
            \subfigure[$d=3$]{
            \label{fig:cosine_3D}
            \includegraphics[width=\twofig]
                {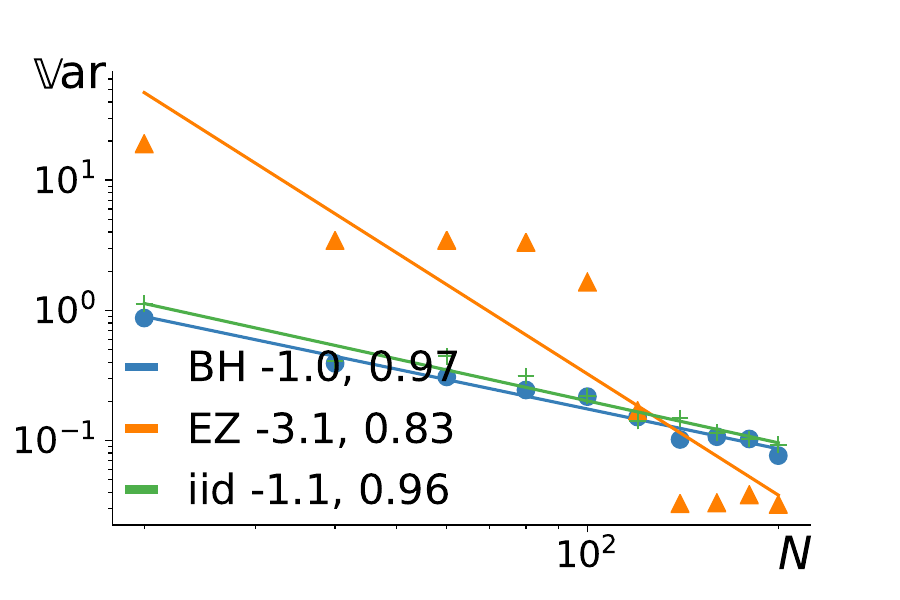}
            }
            \subfigure[$d=4$]{
            \label{fig:cosine_4D}
            \includegraphics[width=\twofig]
                {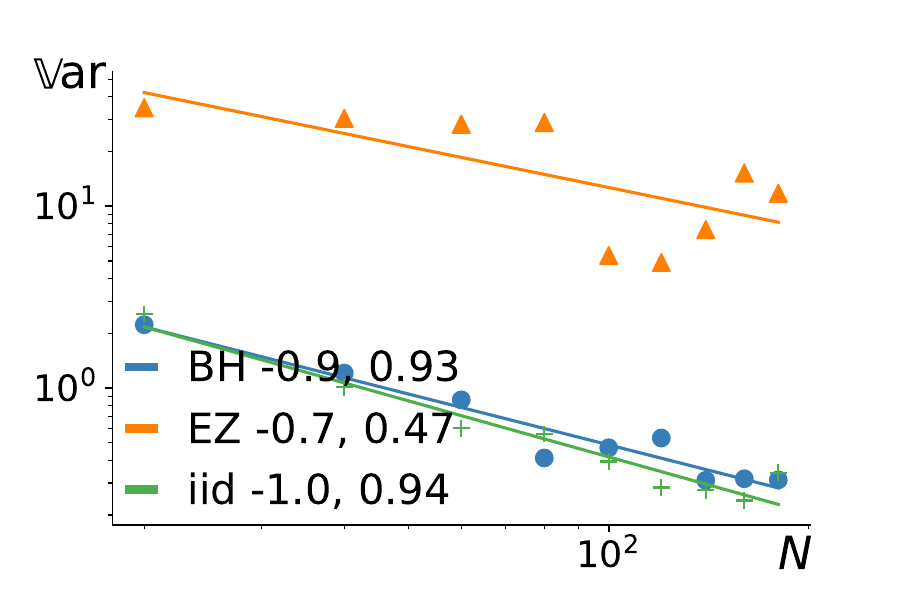}
            }
            \caption{Comparison of $\Iauth{BH}_N$ and $\Iauth{EZ}_N$ for cosine, cf.\,\Secref{sub:further_experiments}.}
            \label{fig:app_cosine}
        \end{figure*}


\clearpage

    \subsection{Integrating a mixture of smooth and non smooth functions} 
    \label{sub:integrating_mix}

        Let $f(x)=H(x)(\cos(\pi x) + \cos(2\pi x) + \sin(5\pi x))$.
        We consider estimating the integral
        \begin{equation}
            \int_{[-1, 1]^d}
                \prod_{i=1}^{d}
                    f(x^i)
                    (1-x^i)^{a^i} (1+x^i)^{b^i}
                \diff x^i
        \end{equation}
        where $a^1, b^1=-\frac12$ and $a^i, b^i$ i.i.d.\,uniformly in $[-\frac12, \frac12]$, using BH \eqref{eq:BH_estimator} and EZ \eqref{eq:EZ_estimator} estimators.

        \begin{figure*}[ht]
            \subfigure[$d=1$]{
            \label{fig:mix_1D}
            \includegraphics[width=\twofig]
                {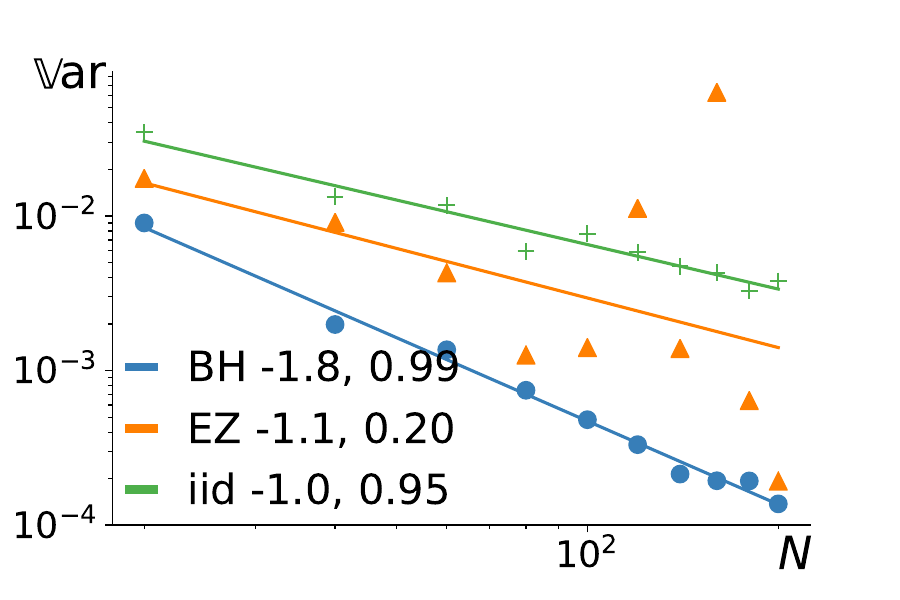}
            }
            \subfigure[$d=2$]{
            \label{fig:mix_2D}
            \includegraphics[width=\twofig]
                {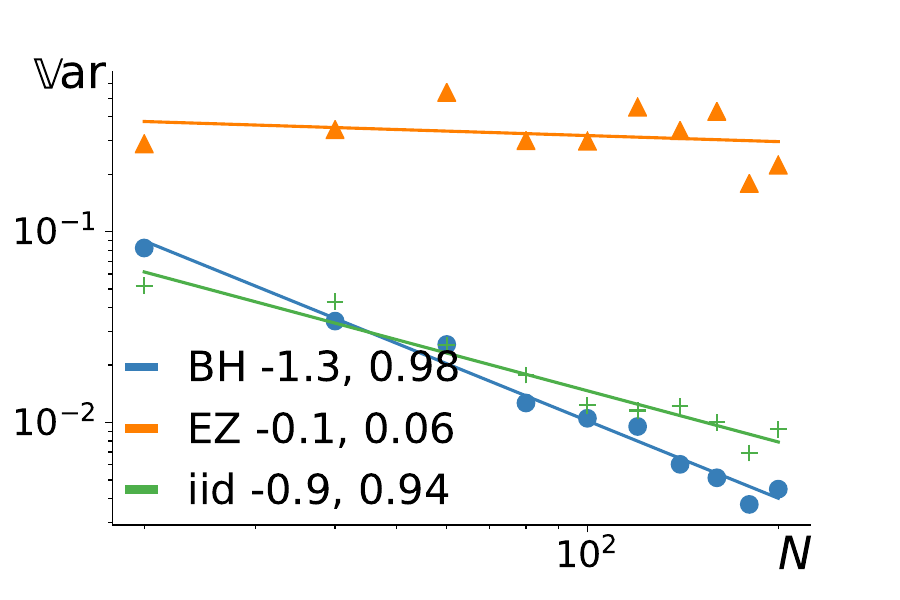}
            }\\
            \subfigure[$d=3$]{
            \label{fig:mix_3D}
            \includegraphics[width=\twofig]
                {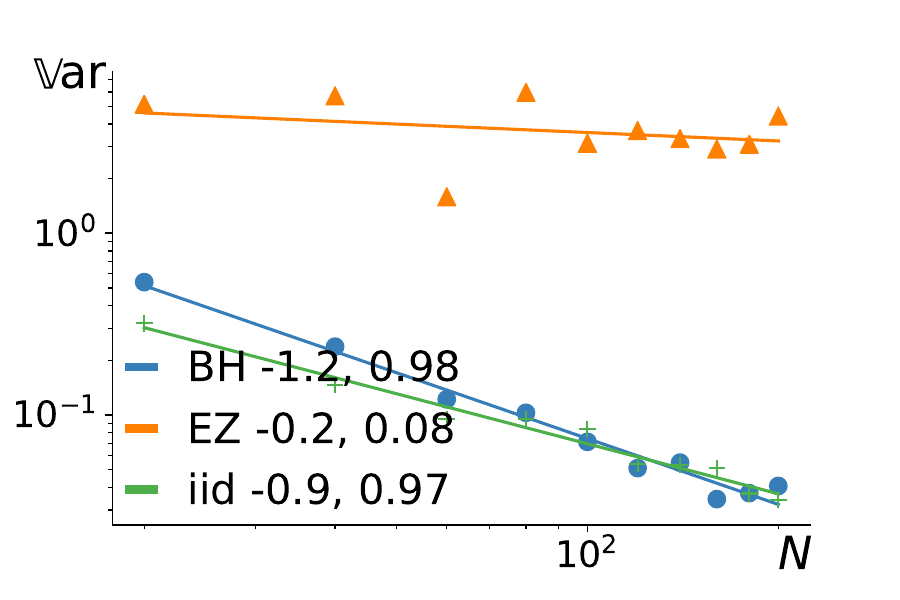}
            }
            \subfigure[$d=4$]{
            \label{fig:mix_4D}
            \includegraphics[width=\twofig]
                {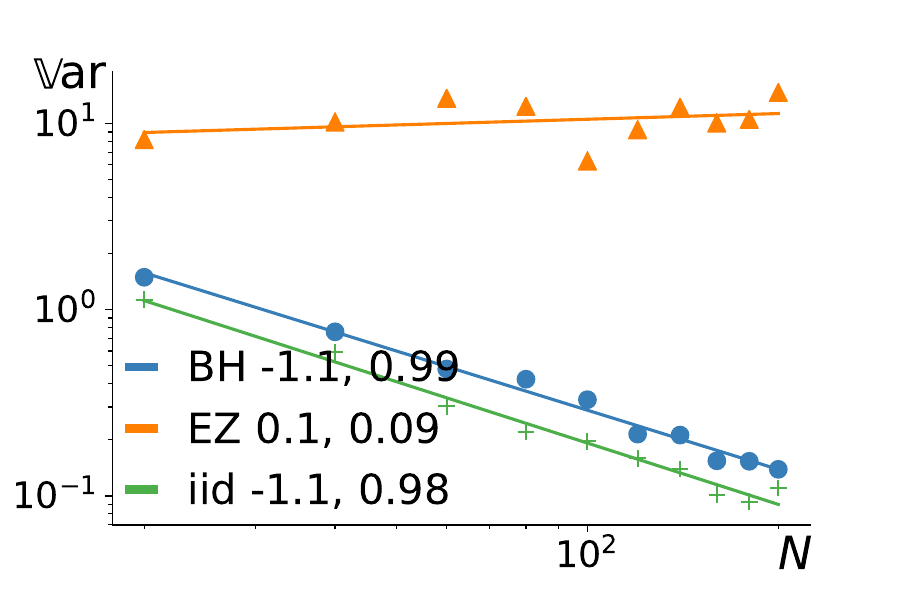}
            }
            \caption{Comparison of $\Iauth{BH}_N$ and $\Iauth{EZ}_N$, cf.\,\Secref{sub:further_experiments}.}
            \label{fig:app_mix}
        \end{figure*}



\end{document}



\appendix

\renewcommand*{\thefigure}{\Alph{figure}}
\renewcommand*{\thetheorem}{\Alph{theorem}}
\renewcommand*{\thelemma}{\Alph{lemma}}
\numberwithin{equation}{section}
\numberwithin{figure}{section}

\section{Methodology} 
\label{sec:method}

    \subsection{The generalized Cauchy-Binet formula: the modern argument} 
    \label{sub:normalization_cauchy_binet}

        \citet[Section 2.2]{Joh06} developed a natural way to build DPPs associated to projection (potentially non hermitian) kernels.
        In this part, we focus on the generalization of the Cauchy-Binet formula \citep[Proposition 2.10]{Joh06}.
        Its usefulness is twofold for our purpose.
        First, it serves to justify the fact that the normalization constant of the joint distribution \eqref{eq:joint_distribution_projection_dpp}  is one, i.e., it is indeed a probability distribution.
        Second, we use it as a modern and simple argument to prove the result of \citet{ErZo60}, cf.\,\Thref{th:ermakov_zolotukhin_estimators}.
        An extended version of the proof is given in \Appref{sub:app_proof_ErZo_th}.

        \begin{lemma}\citep[Proposition 2.10]{Joh06}
        \label{lem:app_cauchy_binet}
            Let $(\bbX, \calB, \mu)$ be a measurable space and consider measurable functions $\phi_{0}, \dots, \phi_{N-1}$ and $\psi_{0}, \dots, \psi_{N-1}$, such that $\phi_k \psi_{\ell} \in L^1(\mu)$.
            Then,
            \begin{equation}
            \label{eq:app_cauchy_binet}
                \det
                    \lrp{
                        \lrsp{\phi_k, \psi_{\ell}}_{L^2(\mu)}
                        }_{k,\ell=1}^{N}
                    = \frac{1}{N!}
                        \int
                        \det \bfPhi(x_{1:N})
                        \det \bfPsi(x_{1:N}) \, \mu^{\otimes N}( \diff x),
            \end{equation}
            where
            \begin{equation*}
                \bfPhi(x_{1:N})
                =
                \begin{pmatrix}
                    \phi_0(x_1) & \dots  & \phi_{N-1}(x_1)\\
                        \vdots     &        & \vdots\\
                    \phi_0(x_N) & \dots  & \phi_{N-1}(x_N)
                \end{pmatrix}
                \quad\text{and}\quad
                \bfPsi(x_{1:N})
                =
                \begin{pmatrix}
                    \psi_0(x_1) & \dots  & \psi_{N-1}(x_1)\\
                        \vdots     &        & \vdots\\
                    \psi_0(x_N) & \dots  & \psi_{N-1}(x_N)
                \end{pmatrix}
            \end{equation*}
        \end{lemma}


    \subsection{Proof of \Thref{th:ermakov_zolotukhin_estimators}}
    \label{sub:app_proof_ErZo_th}

        First, we recall the result of \citet{ErZo60}, cf.\,\Thref{th:ermakov_zolotukhin_estimators}.
        Then, we provide a modern proof based on the generalization of the Cauchy-Binet formula, cf.\,\Lemref{lem:app_cauchy_binet}, where we exploit the orthonormality of the eigenfunctions of the kernel.
        \begin{theorem}
        \label{th:app_ermakov_zolotukhin_estimators}
            Consider $f\in L^2(\mu)$ together with $N$ orthonormal functions $\phi_0, \dots, \phi_{N-1} \in L^2(\mu)$:
            \begin{equation}
                \label{eq:app_orthonormality}
                \lrsp{\phi_k, \phi_{\ell}}_{L^2(\mu)}
                    \triangleq \int \phi_k(x) \phi_{\ell}(x) \mu(\diff x)
                    = \delta_{k\ell},
                    \quad \forall 0\leq k, \ell \leq N-1.
            \end{equation}
            Let $\{\bfx_1,\ldots,\bfx_N\} \sim \DPP(\mu, K_N)$ with $K_N(x,y)=\sum_{k=0}^{N-1} \phi(x)\phi(y)$.
            That is to say $\lrp{\bfx_{1}, \dots, \bfx_{N}}$ has joint distribution
            \begin{equation}
            \label{eq:app_joint_distribution_projection_dpp}
                \frac1{N!}
                    \det\lrb{K_N(x_p, x_q)}_{p, q=1}^N
                    \, \mu^{\otimes N}(\diff x).
            \end{equation}
            Then, the solution of
            \begin{equation}
                \label{eq:app_EZ_linear_system}
                \begin{pmatrix}
                    \phi_0(\bfx_1) & \dots  & \phi_{N-1}(\bfx_1)\\
                        \vdots     &        & \vdots\\
                    \phi_0(\bfx_N) & \dots  & \phi_{N-1}(\bfx_N)
                \end{pmatrix}
                \begin{pmatrix}
                    y^1\\
                    \vdots\\
                    y^N
                \end{pmatrix}
                =
                \begin{pmatrix}
                    f(\bfx_1)\\
                    \vdots\\
                    f(\bfx_N)
                \end{pmatrix}
            \end{equation}
            is unique, $\mu$-almost surely and the coordinates of the solution vector, namely
            \begin{equation}
                \label{eq:app_coordinates}
                y^k
                =
                \frac{\det \bfPhi_{\phi_{k-1}, f}(\bfx_{1:N})}
                     {\det \bfPhi(\bfx_{1:N})},
            \end{equation}
            satisfy
            \begin{equation}
                \label{eq:app_expe_var_EZ}
                \Expe[y^k]
                    = \lrsp{f, \phi_{k-1}}_{L^2(\mu)},
                \quad\text{and}\quad
                \Var[y^k]
                    = \lrnorm{f}_{L^2(\mu)}^2
                       - \sum_{\ell=0}^{N-1}
                         \lrsp{f, \phi_{\ell}}_{L^2(\mu)}^2,
            \end{equation}
            where $\bfPhi(\bfx_{1:N})$ denotes the feature matrix in \eqref{eq:app_EZ_linear_system} and $\bfPhi_{\phi_{k-1}, f}(\bfx_{1:N})$ is defined as the matrix obtained by replacing the $k$-th column of $\bfPhi(\bfx_{1:N})$ by $\lrp{f(\bfx_1), \dots, f(\bfx_N)}^{\top}$.
        \end{theorem}

        \begin{proof}[Proof of \Thref{th:app_ermakov_zolotukhin_estimators}]
            First, the joint distribution \eqref{eq:app_joint_distribution_projection_dpp} of $\lrp{\bfx_{1}, \dots, \bfx_{N}}$ is proportional to $\lrp{\det \bfPhi(\bfx_{1:N})}^2 \mu^{\otimes N}(x)$.
            Thus, $\det \bfPhi(x_{1:N})\neq 0$, $\mu$-almost surely.
            Hence, the matrix $\bfPhi(\bfx_{1:N})$ defining the linear system \eqref{eq:app_EZ_linear_system} is invertible, $\mu$-almost surely.

            The expression of the coordinates \eqref{eq:app_coordinates} follows from Cramer's rule.

            Then, we treat the case $k=1$, the others follow the same lines.
            The proof relies on \Lemref{lem:app_cauchy_binet} where we exploit the orthonormality of the $\phi_k$s \eqref{eq:app_orthonormality}.
            The expectation in \eqref{eq:app_expe_var_EZ} reads
            \begin{align}
                \Expe[\frac{\det \bfPhi_{\phi_0, f}(\bfx_{1:N})}
                            {\det \bfPhi(\bfx_{1:N})}]
                &\lequal{}{\eqref{eq:app_joint_distribution_projection_dpp}}
                \frac{1}{N!}
                    \int
                    \det \bfPhi_{\phi_0, f} (x_{1:N})
                    \det \bfPhi (x_{1:N})
                        \, \mu^{\otimes N}( \diff x) \nonumber\\
                &\lequal{}{\eqref{eq:app_cauchy_binet}}
                    \det
                    \begin{pmatrix}
                        \lrsp{f, \phi_{0}}_{L^2(\mu)}^2
                        & \lrp{
                            \lrsp{f, \phi_{\ell}}_{L^2(\mu)}^2
                            }_{\ell=1}^{N-1}
                        \\[1em]
                        \lrp{
                            \lrsp{f, \phi_{0}}_{L^2(\mu)}^2
                            }_{k=1}^{N-1}
                        & \lrp{
                            \lrsp{\phi_k, \phi_{\ell}}_{L^2(\mu)}^2
                            }_{k,\ell=1}^{N-1}
                    \end{pmatrix}
                    \nonumber\\
                &\lequal{}{\eqref{eq:app_orthonormality}}
                    \det
                    \begin{pmatrix}
                        \lrsp{f, \phi_{0}}_{L^2(\mu)}^2
                        & \lrp{
                            \lrsp{f, \phi_{\ell}}_{L^2(\mu)}^2
                            }_{\ell=1}^{N-1}
                        \\
                        0_{N-1,1}
                        & I_{N-1}
                    \end{pmatrix}
                    \nonumber\\
                &= \lrsp{f, \phi_{0}}_{L^2(\mu)}^2.
                \label{eq:app_proof_1st_moment}
            \end{align}
            Similarly, the second moment reads
            \begin{align}
                \Expe[
                    \lrp{
                     \frac{\det \bfPhi_{\phi_0, f}(\bfx_{1:N})}
                      {\det \bfPhi(\bfx_{1:N})}}^2
                    ]
                &\lequal{}{\eqref{eq:app_joint_distribution_projection_dpp}}
                    \frac{1}{N!}
                    \int
                    \det \bfPhi_{\phi_0, f}(x_{1:N})
                    \det \bfPhi_{\phi_0, f}(x_{1:N})
                    \, \mu^{\otimes N}( \diff x)
                    \nonumber\\
                &\lequal{}{\eqref{eq:app_cauchy_binet}}
                    \det
                    \begin{pmatrix}
                        \lrsp{f, f}_{L^2(\mu)}^2
                        & \lrp{
                            \lrsp{f, \phi_{\ell}}_{L^2(\mu)}^2
                            }_{\ell=1}^{N-1}\\[1em]
                        \lrp{
                            \lrsp{f, \phi_{k}}_{L^2(\mu)}^2
                            }_{k=1}^{N-1}
                        & \lrp{
                            \lrsp{\phi_k, \phi_{\ell}}_{L^2(\mu)}^2
                            }_{k,\ell=1}^{N-1}
                    \end{pmatrix}
                    \nonumber\\
                &\lequal{}{\eqref{eq:app_orthonormality}}
                    \det
                    \begin{pmatrix}
                        \lrnorm{f}_{L^2(\mu)}^2
                        & \lrp{
                            \lrsp{f, \phi_{\ell}}_{L^2(\mu)}^2
                            }_{\ell=1}^{N-1}
                        \\
                        \lrp{
                            \lrsp{f, \phi_{k}}_{L^2(\mu)}^2
                            }_{k=1}^{N-1}
                        & I_{N-1}
                    \end{pmatrix}
                    \nonumber\\
                &= \lrnorm{f}_{L^2(\mu)}^2
                    - \sum_{k=1}^{N-1}
                        \lrsp{f, \phi_{k}}_{L^2(\mu)}^2.
                    \label{eq:app_proof_2nd_moment}
            \end{align}
            Finally, the variance in \eqref{eq:app_expe_var_EZ} $=$ \eqref{eq:app_proof_2nd_moment} - \eqref{eq:app_proof_1st_moment}$^2$.
        \end{proof}

    \subsection{EZ estimator as a quadrature rule with weights summing to one} 
    \label{sub:app_EZ_as_quadrature}

        In this part, we consider \Thref{th:app_ermakov_zolotukhin_estimators} in the setting where one of the eigenfunctions of the kernel, say $\phi_0$ is constant.
        In this case, we show that the EZ estimator defined to estimate $\int f(x) \mu(\diff x)$ can be seen as a quadrature rule in the sense of \eqref{eq:quadrature}, with weights $\omega_n$ that sum to $\mu\big(\lrb{-1, 1}^d\big)$.
        This is a non obvious fact, judging from the expression \eqref{eq:EZ_estimator} of the estimator.
        \begin{proposition}
            Consider $\phi_0$ constant in \Thref{th:app_ermakov_zolotukhin_estimators}.
            Then, solving the corresponding linear system \eqref{eq:app_EZ_linear_system} allows to construct
            \begin{equation}
            \label{eq:app_EZ_estimator}
                \Iauth{EZ}_N(f)
                    \lequal{}{\eqref{eq:EZ_estimator}}
                        \sqrt{\mu\Big(\lrb{-1, 1}^d\Big)}
                        ~
                        \frac{\det \bfPhi_{\phi_0, f}(\bfx_{1:N})}
                             {\det \bfPhi(\bfx_{1:N})}
            \end{equation}
            as an unbiased estimator of $\int f(x) \mu(\diff x)$, with variance equal to the variance in \eqref{eq:app_expe_var_EZ}$\times\mu\Big(\lrb{-1, 1}^d\Big)$.
            In particular it can be seen as a random quadrature rule \eqref{eq:quadrature},
            \begin{equation}
                \Iauth{EZ}_N(f)
                    = \sum_{n=1}^{N} \omega_n(\bfx_{1:N}) f(\bfx_n)
                    \approx \int f(x) \mu(\diff x)
            \end{equation}
            such that
            $\sum_{n=1}^N \omega_n(\bfx_{1:N}) = \mu\Big(\lrb{-1, 1}^d\Big)$.
        \end{proposition}
        \begin{proof}
            Take $\phi_0$ constant.
            Since $\phi_0$ has unit norm, cf.\,\eqref{eq:app_orthonormality}, it is straightforward to see that
            \begin{equation*}
                \phi_0 = \frac{1}{\sqrt{\mu\Big(\lrb{-1, 1}^d\Big)}}
            \end{equation*}
            so that \eqref{eq:app_EZ_estimator} can be written
            \begin{align*}
            \Iauth{EZ}_N(f)
                &=
                    \mu\Big(\lrb{-1, 1}^d\Big)
                    ~
                        \frac{\det \bfPhi_{\phi_0, f}(\bfx_{1:N})}
                             {\det \bfPhi_{\phi_0, 1}(\bfx_{1:N})}.\\
            \intertext{Expanding the numerator w.r.t.\,the first column yields}
            \Iauth{EZ}_N(f)
                &=
                \sum_{n=1}^N
                    f(\bfx_n)
                    \underbrace{
                        (-1)^{1+n}
                        \det
                        \lrp{\phi_k(x_p)}_{k=1, p=1\neq n}^{N-1, N}
                        \frac
                            {\mu\Big(\lrb{-1, 1}^d\Big)}
                            {\det \bfPhi_{\phi_0, 1}(\bfx_{1:N})}
                        }_{\triangleq\omega_n(\bfx_{1:N})}\cdot
            \end{align*}
            In particular, there is a priori no reason for the weights to be nonnegative.
            Finally,
            \begin{equation*}
                \sum_{n=1}^{N}
                    \omega_n(\bfx_{1:N})
                = \frac
                    {\mu\Big(\lrb{-1, 1}^d\Big)}
                    {\cancel{\det \bfPhi_{\phi_0, 1}(\bfx_{1:N})}}
                    \underbrace{
                        \sum_{n=1}^{N}
                            (-1)^{1+n}
                            \det
                            \lrp{\phi_k(x_p)}_{k=1, p=1\neq n}^{N-1, N}
                        }_{=\cancel{\det \bfPhi_{\phi_0, 1}(\bfx_{1:N})}}
                = \mu\Big(\lrb{-1, 1}^d\Big).
            \end{equation*}
            This concludes the proof.
        \end{proof}


    \subsection{Sampling multivariate Jacobi ensembles} 
    \label{sub:app_sampling}

        We consider sampling exactly from the multivariate Jacobi ensemble as defined in \Secref{sub:multivariate_jacobi_ensemble}.

        In dimension $d=1$, to sample the univariate Jacobi ensemble, with base measure $\mu(\diff x) = (1-x)^a (1+x)^b \diff x$ where $a, b > -1$, we use the random tridiagonal matrix model of \citet[Theorem 2]{KiNe04}.
        That is to say, computing the eigenvalues of a properly randomized tridiagonal matrix allows to get a sample of this continuous projection DPP at cost $\calO(N^2)$!

        For $d\geq 2$, we follow \citet[Section 3]{BaHa16} who proposed to use the chain rule \eqref{eq:chain_rule_geometric} to sample from the multivariate Jacobi ensemble with base measure $\mu(\diff x)=\omega(x)\diff x$, where
        \begin{equation}
            \label{eq:app_base_measure}
            \omega(x)
                = \prod_{i=1}^d (1-x^i)^{a^i} (1-x^i)^{b^i},
            \text{with}~
            \lrabs{a^i}, \lrabs{b^i}\leq \frac12\cdot
        \end{equation}
        To do so, we use the same proposal distribution and rejection bound to sample from each of the conditionals \eqref{eq:chain_rule_geometric}.
        The density (w.r.t.\,Lebesgue) of the proposal distribution writes
        \begin{equation}
            \label{eq:app_equilibrium_measure}
            \weq(x)
                = \prod_{i=1}^d \frac{1}{\pi\sqrt{1-(x^i)^2}}\cdot
        \end{equation}
        The rejection constant is derived after by successive applications of the following result on Jacobi polynomials derived by \citet{ChGaWo94}.
        \begin{proposition}
        \label{propo:Chow_bound}
            \citep[Equation 1.3]{Gau09}
            Let $(\phi_k)_{k\geq 0}$ be the orthonormal polynomials w.r.t.\,the measure $(1-x)^a (1+x)^b \diff x$ with $\lrabs{a}\leq \frac12, \lrabs{b}\leq \frac12$.
            Then, for any $x \in [-1, 1]$ and $k\geq 0$,
            \begin{equation}
            \label{eq:Chow_bound}
                \pi
                    (1-x)^{a+\frac{1}{2}}
                    (1+x)^{b+\frac{1}{2}}
                \phi_{k}(x)^{2}
                \leq
                    \frac
                        {2~\Gamma(k+a+b+1)~\Gamma(k+\max(a,b)+1)}
                        {k!~(k+\frac{a+b+1}{2})^{2 \max(a,b)}~\Gamma(k+\min(a,b)+1)}\cdot
            \end{equation}
        \end{proposition}

        The domination of the acceptance ratio, i.e., the ratio of the $n$-th conditional density in \eqref{eq:chain_rule_geometric} over the proposal density \eqref{eq:app_equilibrium_measure} is computed as follows
        \begin{align}
            &\frac{K_N(x,x)
                        -   {\bfK}_{n-1}(x)^{\top}
                            {\bfK}_{n-1}^{-1}
                            {\bfK}_{n-1}(x)
                    }{N-(n-1)}
                \omega(x)\times\frac{1}{\weq(x)}
                \nonumber\\
            &\leq
                \frac{1}{N-(n-1)}
                \frac{K(x,x) \omega(x)}{w_{\text{eq}}(x)}
            \lequal{\eqref{eq:app_equilibrium_measure}}
                   {\eqref{eq:kernel_multivariate_separable_OPE}}
                \frac{1}{N-(n-1)}
                \lsum{\mathfrak{b}(k)=0}{N-1}
                    \lprod{i=1}d
                        \pi(1-x^i)^{a^i+\frac12} (1+x^i)^{b^i+\frac12}
                        \phi_{k^i}^i(x^i)^2.
            \label{eq:domination}
        \end{align}
        Finally, each of the terms that appear in \eqref{eq:domination} can be bounded using the following recipe:
        \begin{enumerate}
            \item For $k^i>0$, we use the bound \eqref{eq:Chow_bound}
            \item For $k^i=0$, the domination of the left hand side (LHS) of \eqref{eq:Chow_bound} is not tight enough ($=2$), so we proceed as follows.
            In this case, $\phi_{0}$ is constant equal to $\lrp{\int (1-x)^a(1+x)^b \diff x}^{-1/2}$ and since $|a|, |b|\in [-1/2, 1/2]$ we upper bound $(1-x)^{a+1/2}(1+x)^{b+1/2}$ by the evaluation at its mode.
        \end{enumerate}

        Point 2. is crucial to tighten the rejection constant.
        Indeed, because of the choice of the ordering $\mathfrak{b}$ (cf.\,\Secref{sub:multivariate_jacobi_ensemble}), the number of times that $\phi_0^i$ appears in \eqref{eq:domination} increases with the dimension.
        Hence, the tighter the bound on the LHS of \eqref{eq:Chow_bound} for $k=0$ the best the rejection constant.

        In \Figref{fig:app_rejection_bounds} we illustrate how the rejection bound increased with the dimension.
        In particular observe how the gap increases between the rejection constants associated low $-\frac12$ and large parameters $\frac12$ of the base measure $\mu(\diff x) = \prod_{i=1}^{d} (1-x)^{a^i}(1+x)^{b^i} \diff x^i$.
        \begin{figure*}[ht]
            \centering
            \subfigure[Larger \Figref{fig:rejection_bound}, $d=1, 2, 3, 4$]{
            \label{fig:app_rejection_bounds_1D}
            \includegraphics[width=0.8\textwidth]
                {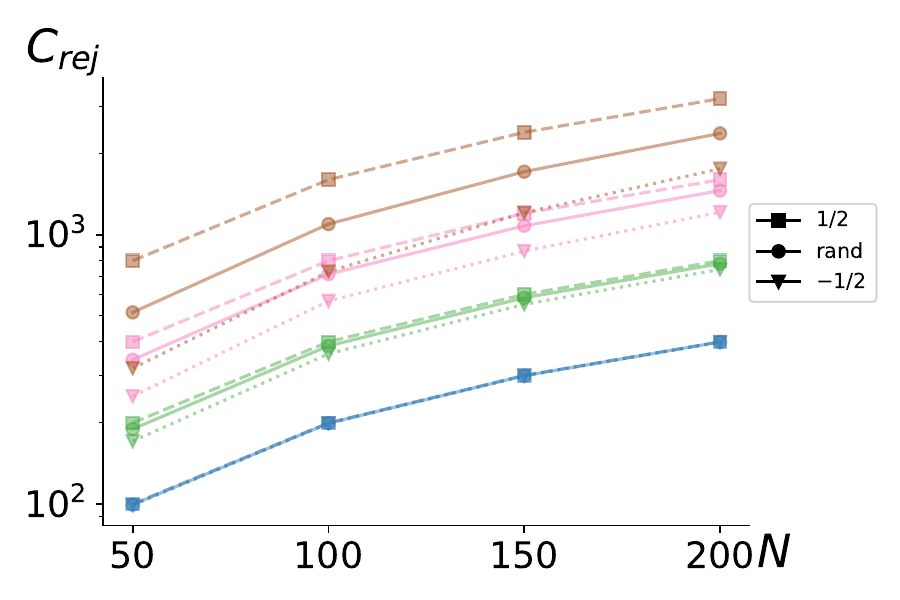}
            }\\
            \subfigure[$d=1, 5, 10, 15$]{
            \label{fig:app_rejection_bounds_2D}
            \includegraphics[width=0.8\textwidth]
                {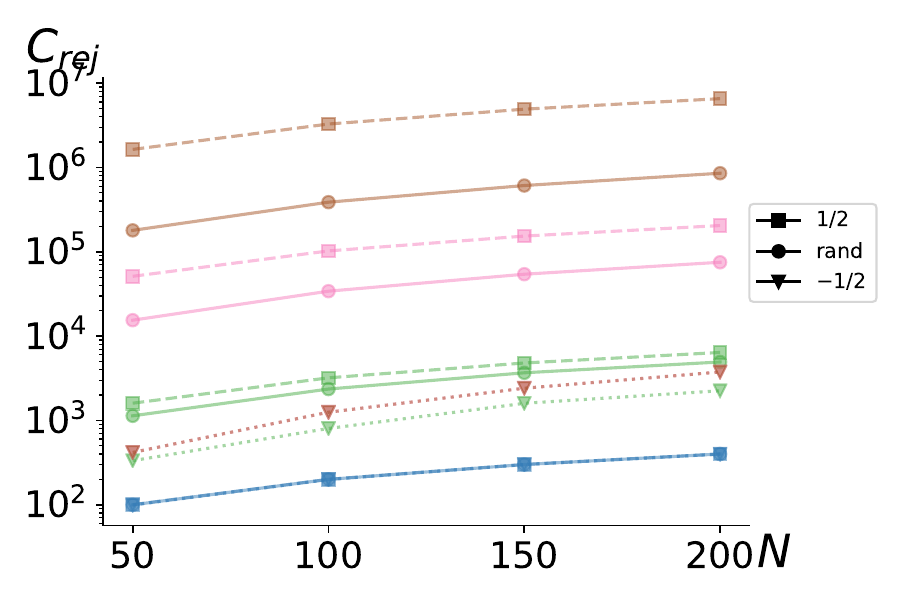}
            }
            \caption{Rejection bounds}
            \label{fig:app_rejection_bounds}
        \end{figure*}



\section{Experiments} 
\label{sec:experiments}

    \subsection{Reproducing the bump example}
    \label{sub:app_XP_bump}

        In \Secref{sub:XP_bump}, we reproduce the experiment of \citet[Section 3]{BaHa16} where they illustrate the behavior of $\Iauth{BH}_N$ on a unimodal, smooth bump function:
        \begin{equation}
            f(x)
            =
                \prod_{i=1}^d
                    \exp\lrp{-\frac{1}{1- \varepsilon - (x^i)^2}}
                    \indic_{[-1+\varepsilon, 1- \varepsilon]}(x^i).
        \end{equation}
        For each value of $N$, we sample $100$ times from the same multivariate Jacobi ensembles with i.i.d. uniform parameters on $[-1/2,1/2]$, compute the resulting $100$ values of each estimator, and plot the two resulting sample variances.
        In addition, in \Figref{fig:app_bump_KS} we test the potential hope for a CLT for $\Iauth{EZ}_N$ and compare with $\Iauth{BH}_N$ for which the CLT \eqref{eq:BH_CLT} holds, in the regime $N=300$.

        \begin{figure*}[ht]
            \subfigure[$d=1$]{
            \label{fig:app_bump_1D}
            \includegraphics[width=.49\textwidth]
                {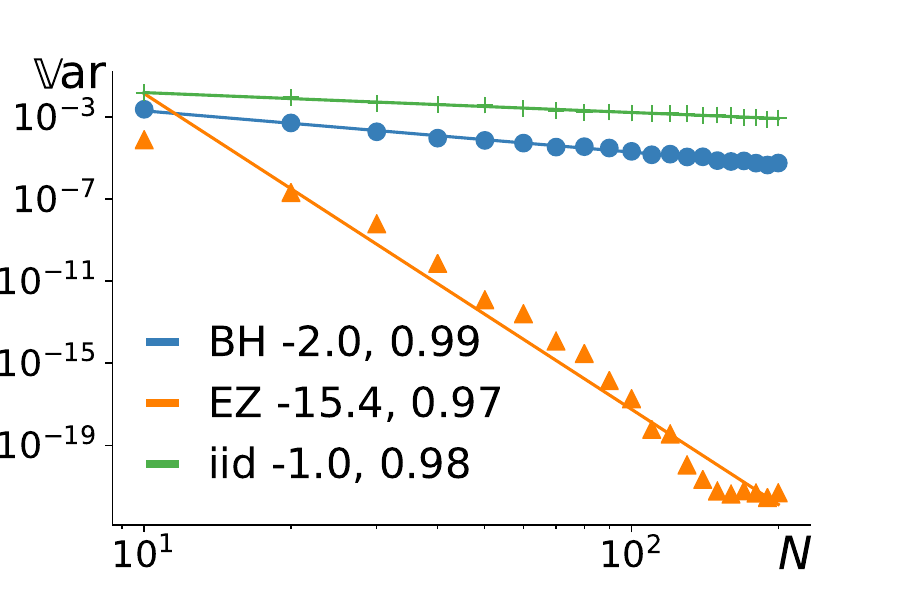}
            }
            \subfigure[$d=2$]{
            \label{fig:app_bump_2D}
            \includegraphics[width=.49\textwidth]
                {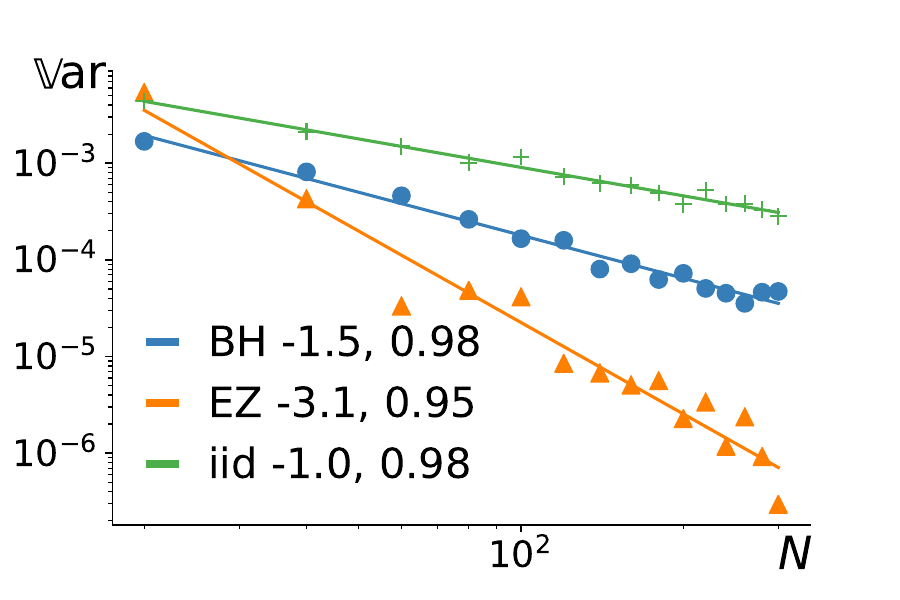}
            }\\
            \subfigure[$d=3$]{
            \label{fig:app_bump_3D}
            \includegraphics[width=.49\textwidth]
                {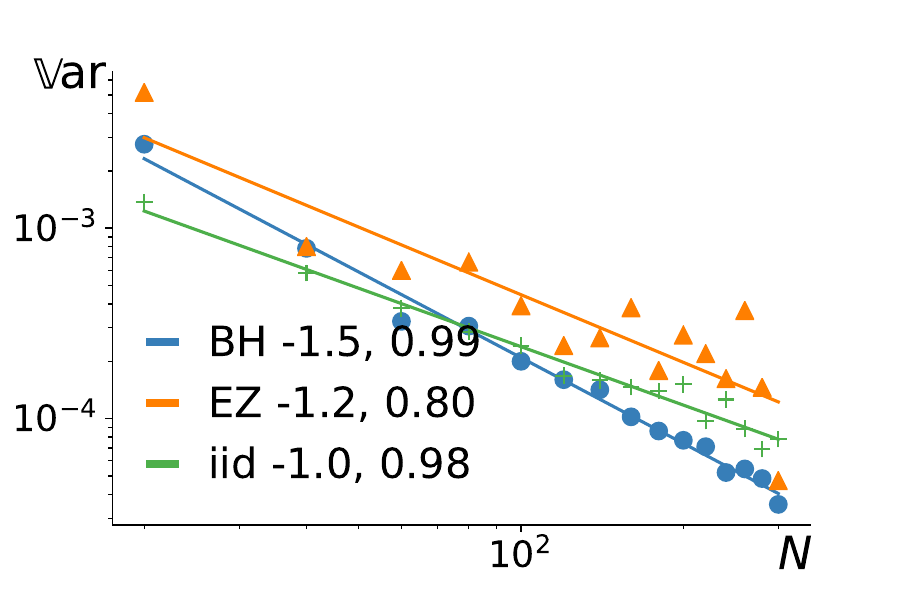}
            }
            \subfigure[$d=4$]{
            \label{fig:app_bump_4D}
            \includegraphics[width=.49\textwidth]
                {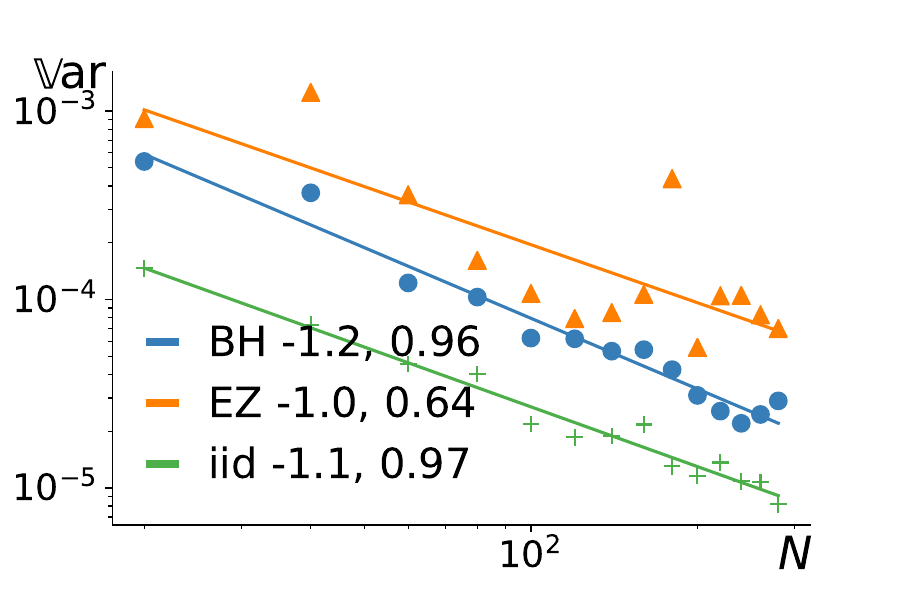}
            }
            \caption{Reproducing the bump function ($\varepsilon=0.05$) experiment of \citet{BaHa16}, cf.\,\Secref{sub:XP_bump}.
            The expected variance decay of order $1/N^{1+1/d}$ is observed for BH.
            For $d=1$, EZ has almost no variance for $N\geq 100$: the bump function is extremely well approximated by a polynomials of degree $N\geq 100$.}
            \label{fig:app_bump}
        \end{figure*}
        \begin{figure*}[ht]
            \subfigure[$d=1$]{
            \label{fig:app_bump_1D_KS}
            \includegraphics[width=.49\textwidth]
                {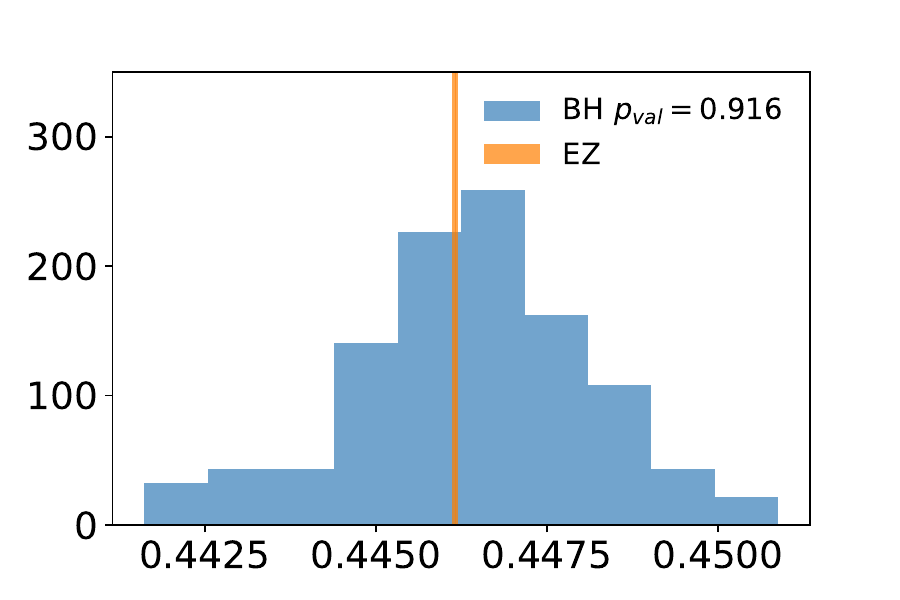}
            }
            \subfigure[$d=2$]{
            \label{fig:app_bump_2D_KS}
            \includegraphics[width=.49\textwidth]
                {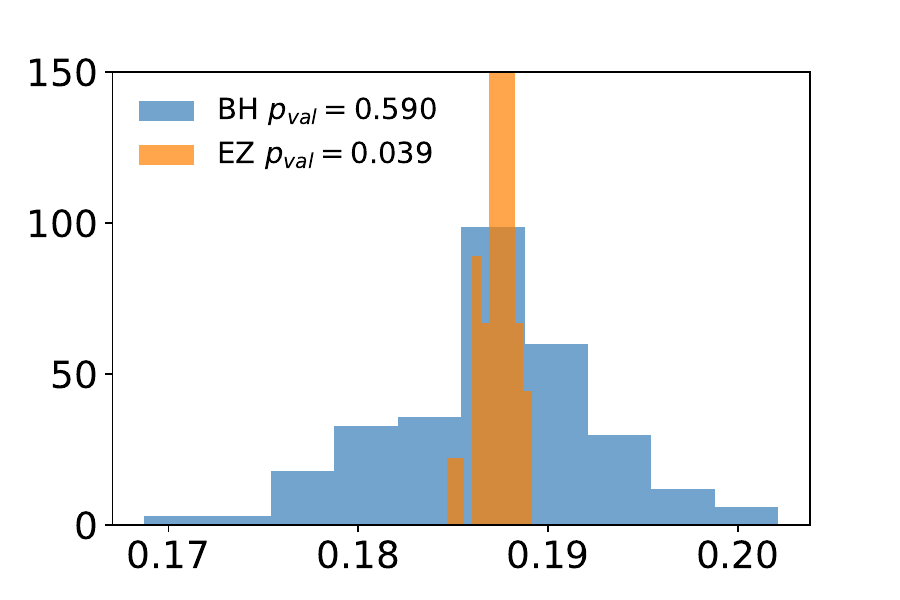}
            }\\
            \subfigure[$d=3$]{
            \label{fig:app_bump_3D_KS}
            \includegraphics[width=.49\textwidth]
                {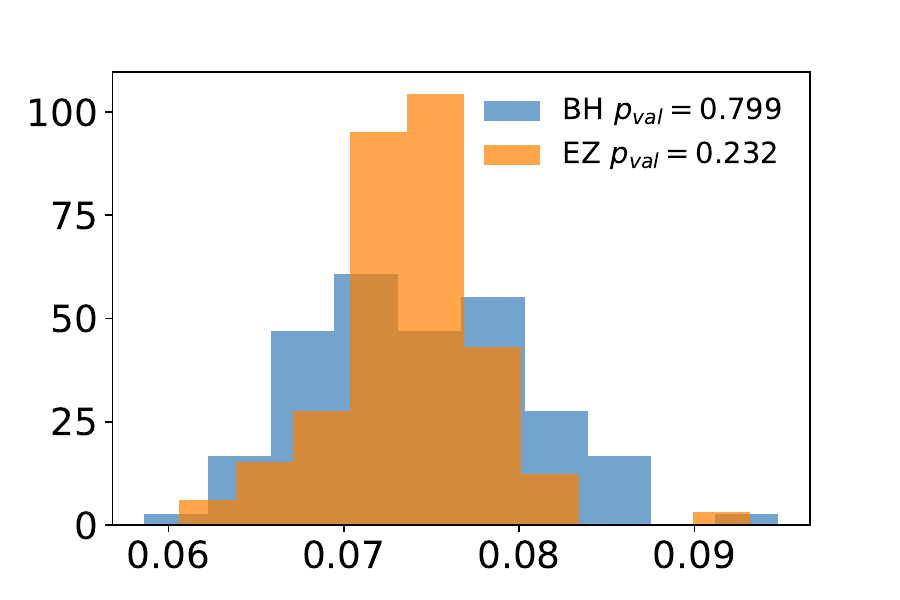}
            }
            \subfigure[$d=4$]{
            \label{fig:app_bump_4D_KS}
            \includegraphics[width=.49\textwidth]
                {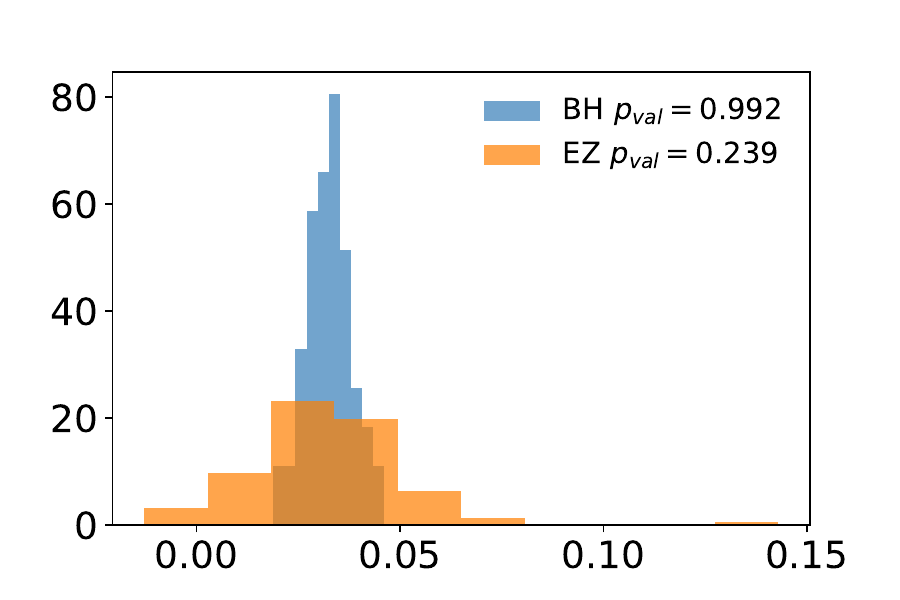}
            }
            \caption{Histogram of $100$ independent estimates $\Iauth{BH}_N$ and $\Iauth{EZ}_N$ of the integral of the bump function ($\varepsilon=0.05$) with $N=300$ and associated p-value of Kolmogorov-Smirnov test, cf.\,\Secref{sub:XP_bump}.
            The fluctuations of BH confirm to be Gaussian (cf.\,CLT \eqref{eq:BH_CLT}).
            (a) the bump function is extremely well approximated by a polynomial of degree $300$ hence $\Iauth{EZ}_N$ has almost no variance.
            (b)-(c)-(d) few outliers seem to break the potential Gaussianity of $\Iauth{EZ}_N(f)$.
            (d) $\Iauth{EZ}_N(f)$ does not preserves the sign of the integrand.}
            \label{fig:app_bump_KS}
        \end{figure*}

\newpage
\clearpage
\newpage
    \subsection{Integrating sums of eigenfunctions}
    \label{sub:app_XP_polynomials}

        \Figref{fig:app_decay_1/k_N=70_Ndpp} gives the results of the first setting set in \Secref{sub:XP_polynomials}, where we integrate a sum of $N_{\text{modes}}=70$ kernel eigenfunctions.
        \Figref{fig:app_decay_1/k_pursuit_Ndpp} illustrates the second setting, where the sum always has one more eigenfunction than there are points in the DPP samples.
        \begin{figure*}[ht]
            \subfigure[$d=1$]{
            \label{fig:app_decay_1/k_N=70_Ndpp_1D}
            \includegraphics[width=\twofig]
                {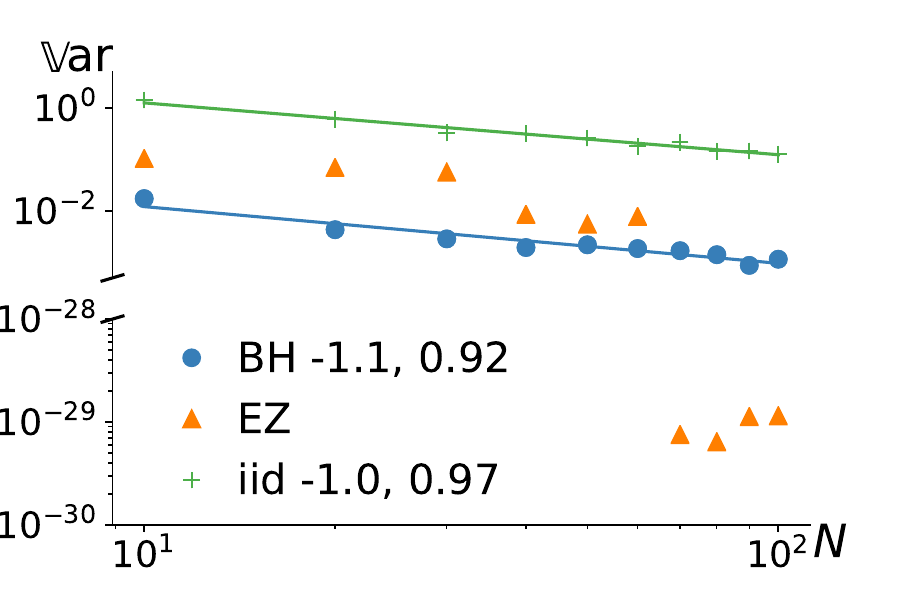}
            }
            \subfigure[$d=2$]{
            \label{fig:app_decay_1/k_N=70_Ndpp_2D}
            \includegraphics[width=\twofig]
                {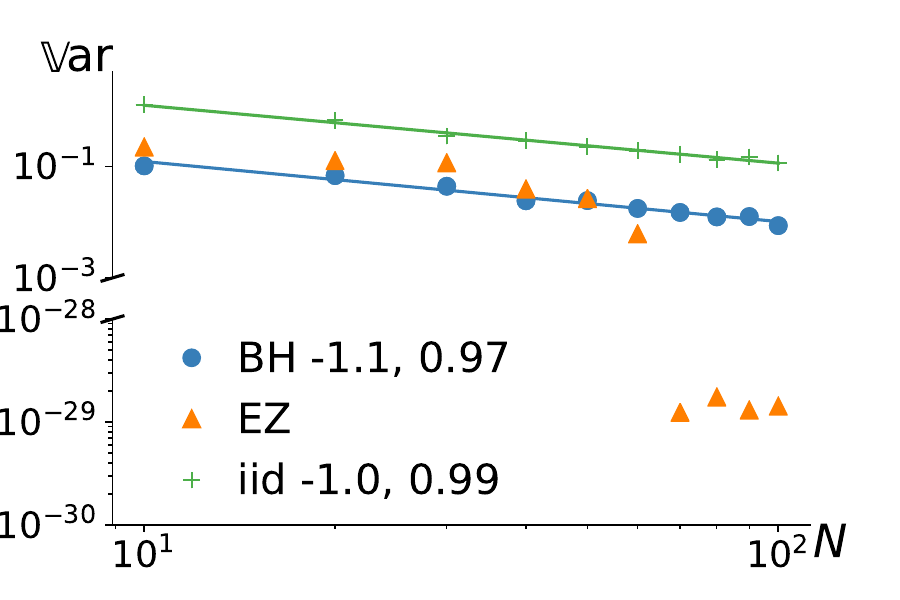}
            }\\
            \subfigure[$d=3$]{
            \label{fig:app_decay_1/k_N=70_Ndpp_3D}
            \includegraphics[width=\twofig]
                {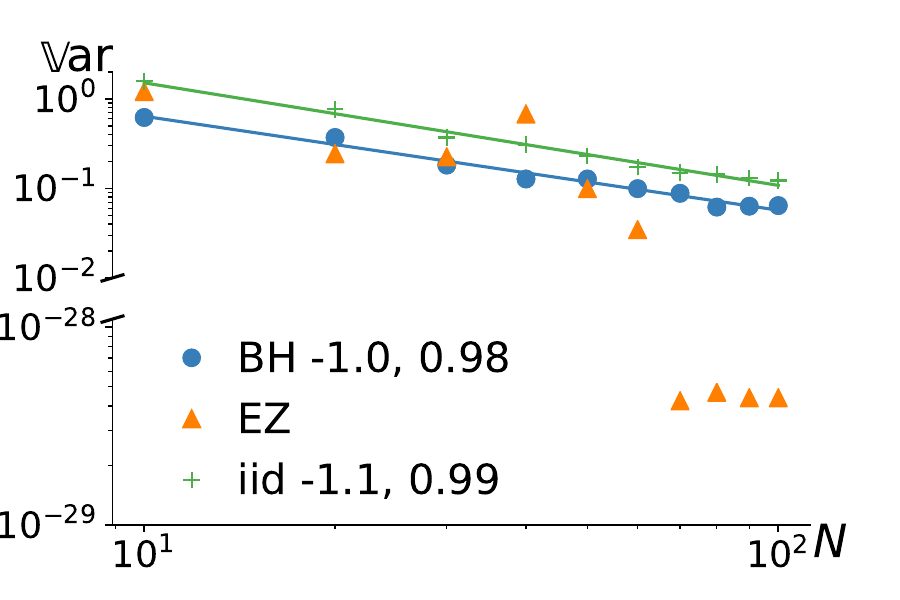}
            }
            \subfigure[$d=4$]{
            \label{fig:app_decay_1/k_N=70_Ndpp_4D}
            \includegraphics[width=\twofig]
                {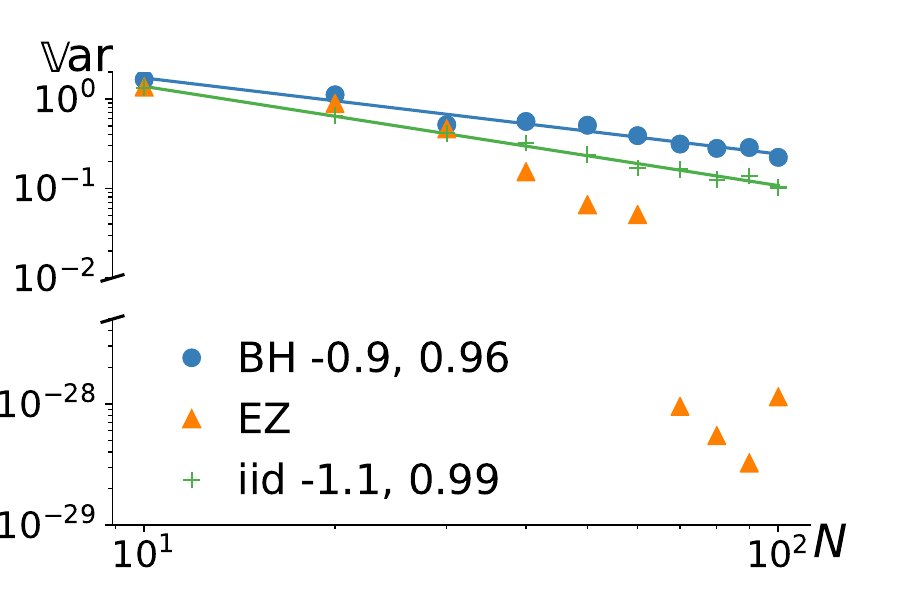}
            }
            \caption{Comparison of $\Iauth{BH}_N$ and $\Iauth{EZ}_N$ integrating a finite sum of $70$ eigenfunctions of the DPP kernel as in \eqref{eq:f_xps_polys}, cf.\,\Secref{sub:XP_polynomials}.}
            \label{fig:app_decay_1/k_N=70_Ndpp}
        \end{figure*}

        \begin{figure*}[ht]
            \subfigure[$d=1$]{
            \label{fig:app_decay_1/k_pursuit_Ndpp_1D}
            \includegraphics[width=\twofig]
                {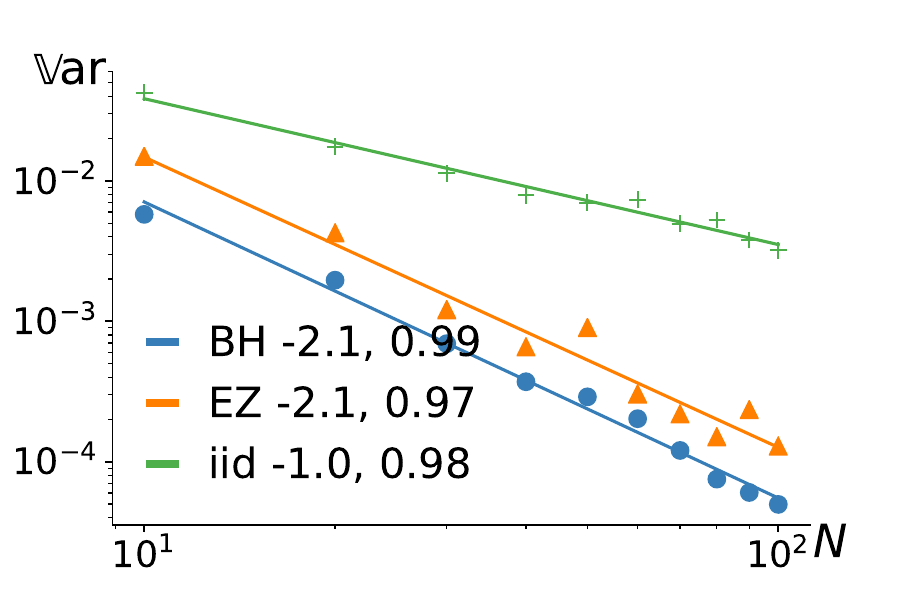}
            }
            \subfigure[$d=2$]{
            \label{fig:app_decay_1/k_pursuit_Ndpp_2D}
            \includegraphics[width=\twofig]
                {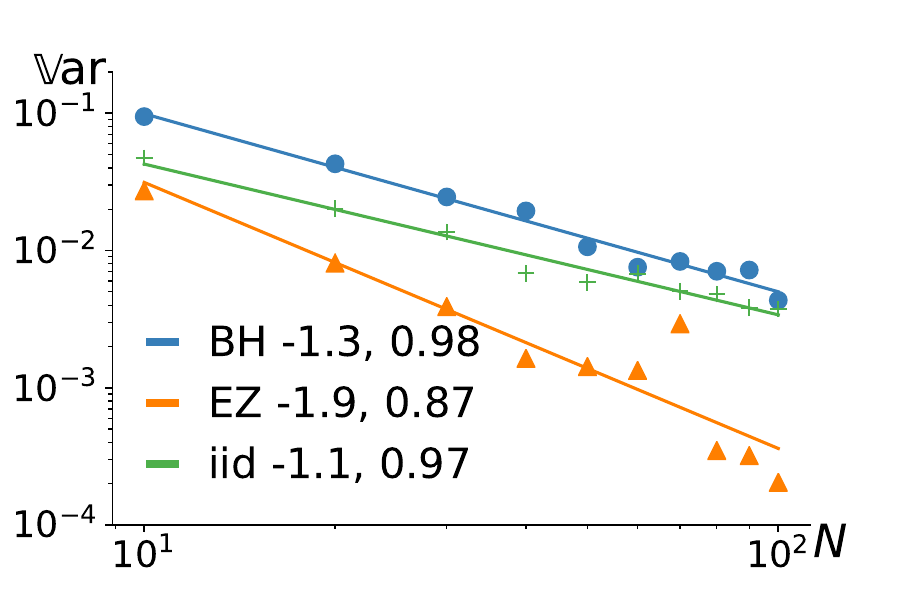}
            }\\
            \subfigure[$d=3$]{
            \label{fig:app_decay_1/k_pursuit_Ndpp_3D}
            \includegraphics[width=\twofig]
                {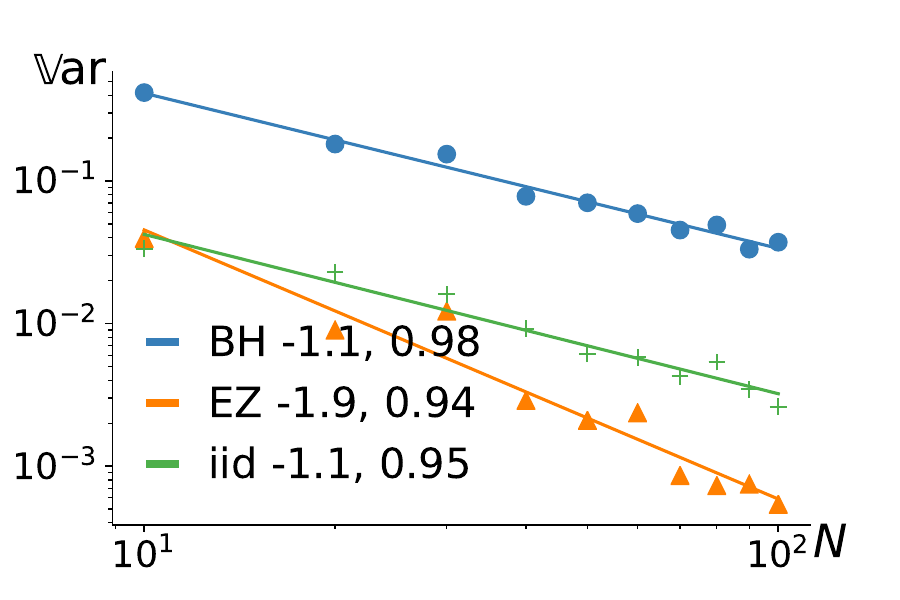}
            }
            \subfigure[$d=4$]{
            \label{fig:app_decay_1/k_pursuit_Ndpp_4D}
            \includegraphics[width=\twofig]
                {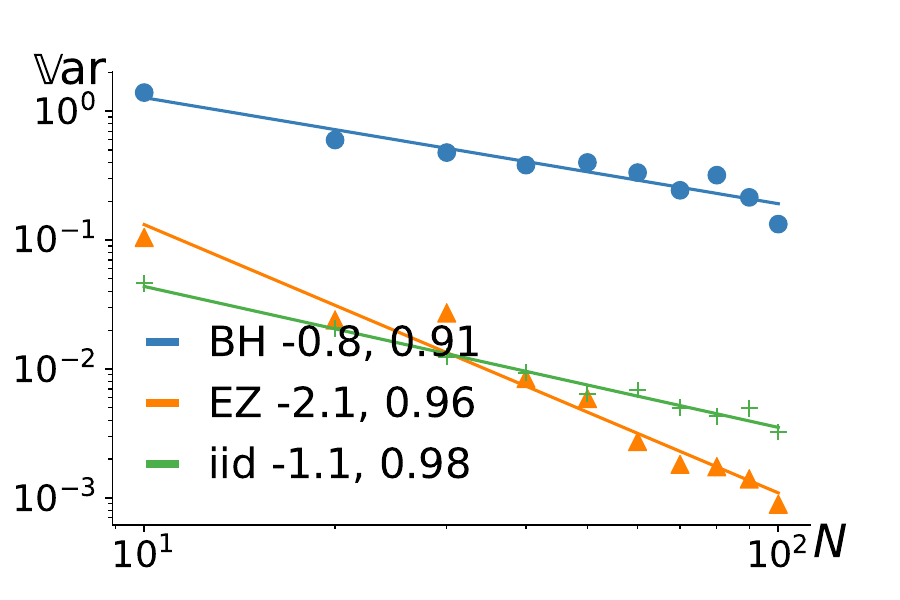}
            }
            \caption{Comparison of $\Iauth{BH}_N$ and $\Iauth{EZ}_N$ for a linear combination of $N+1$ eigenfunctions of the DPP kernel as in \eqref{eq:f_xps_polys}, cf.\,\Secref{sub:XP_polynomials}.}
            \label{fig:app_decay_1/k_pursuit_Ndpp}
        \end{figure*}

\newpage
\clearpage
\newpage
    \subsection{Integrating absolute value}
    \label{sub:app_XP_absolute}

        We consider estimating the integral
        \begin{equation}
            \int_{[-1, 1]^d}
                \prod_{i=1}^{d}
                    |x^i|
                    (1-x^i)^{a^i} (1-x^i)^{b^i}
                \diff x^i
        \end{equation}
        where $a^1, b^1=-\frac12$ and $a^i, b^i$ i.i.d.\,uniformly in $[-\frac12, \frac12]$, using BH \eqref{eq:BH_estimator} and EZ \eqref{eq:EZ_estimator} estimators.

        Results are given in \Figref{fig:app_absolute_value}.
        \begin{figure*}[ht]
            \subfigure[$d=1$]{
            \label{fig:absolute_value_1D}
            \includegraphics[width=\twofig]
                {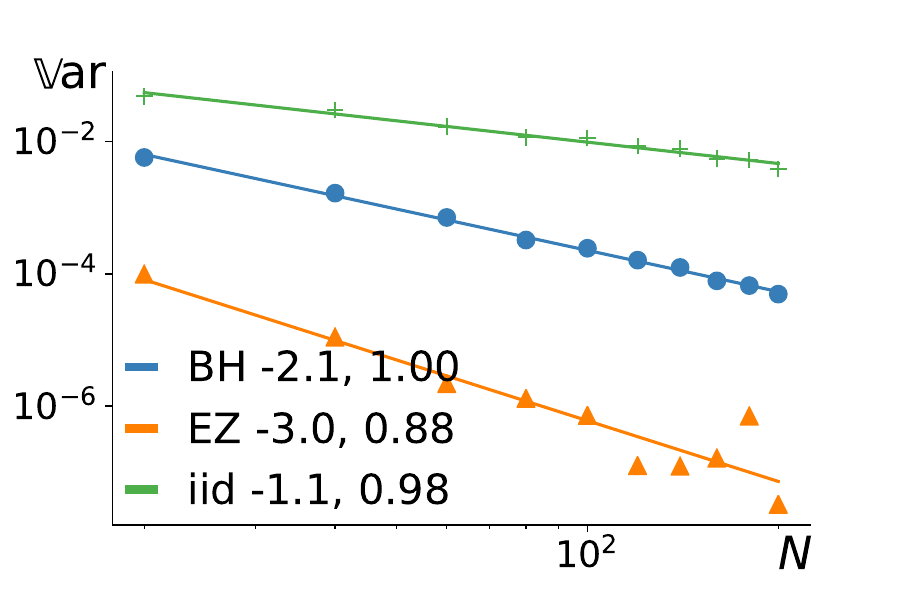}
            }
            \subfigure[$d=2$]{
            \label{fig:absolute_value_2D}
            \includegraphics[width=\twofig]
                {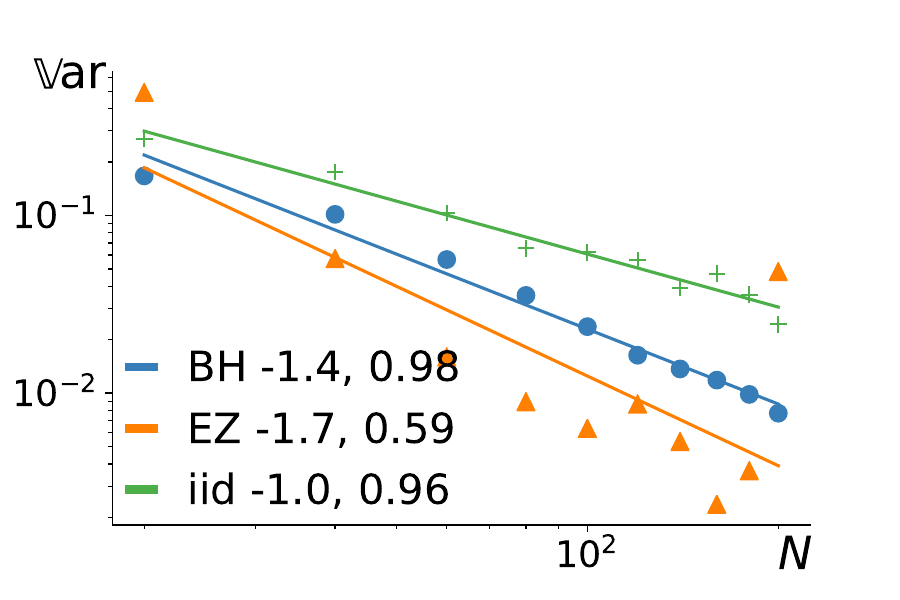}
            }\\
            \subfigure[$d=3$]{
            \label{fig:absolute_value_3D}
            \includegraphics[width=\twofig]
                {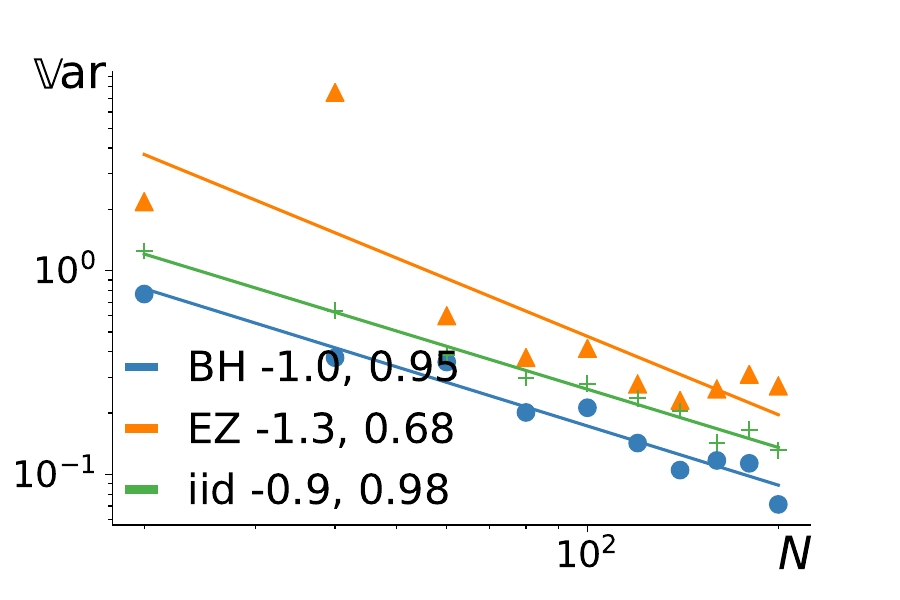}
            }
            \subfigure[$d=4$]{
            \label{fig:absolute_value_4D}
            \includegraphics[width=\twofig]
                {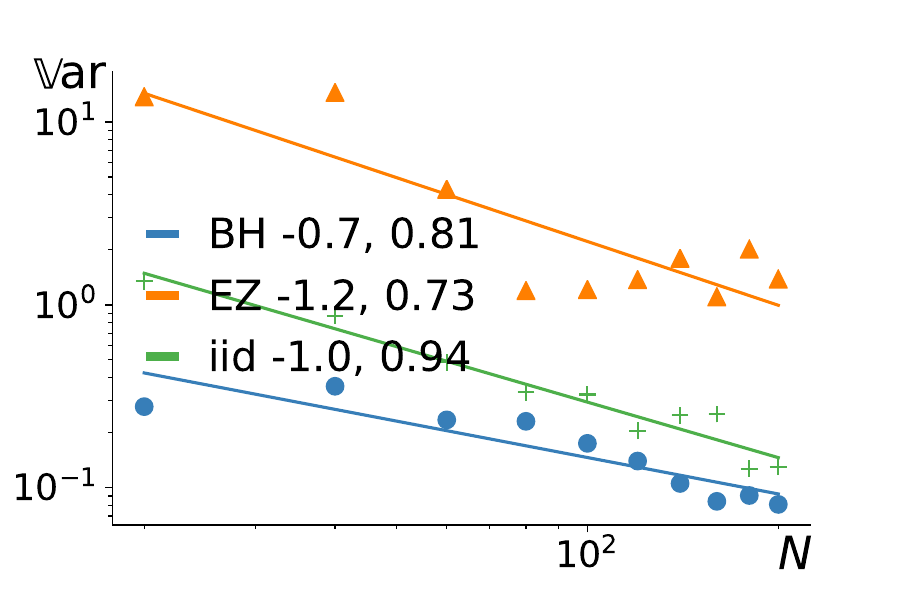}
            }
            \caption{Comparison of $\Iauth{BH}_N$ and $\Iauth{EZ}_N$ for absolute value, cf.\,\Secref{sub:further_experiments}.}
            \label{fig:app_absolute_value}
        \end{figure*}

\newpage
\clearpage
\newpage
    \subsection{Integrating Heaviside}
    \label{sub:app_XP_heaviside}

        Let $H(x)=\begin{cases}
            1, &\text{ if }x>0\\
            -1, &\text{otherwise}
        \end{cases}$.
        We consider estimating the integral
        \begin{equation}
            \int_{[-1, 1]^d}
                \prod_{i=1}^{d}
                    H(x^i)
                    (1-x^i)^{a^i} (1-x^i)^{b^i}
                \diff x^i
        \end{equation}
        where $a^1, b^1=-\frac12$ and $a^i, b^i$ i.i.d.\,uniformly in $[-\frac12, \frac12]$, using BH \eqref{eq:BH_estimator} and EZ \eqref{eq:EZ_estimator} estimators.

        Results are given in \Figref{fig:app_heaviside}.
        \begin{figure*}[ht]
            \subfigure[$d=1$]{
            \label{fig:heaviside_value_1D}
            \includegraphics[width=\twofig]
                {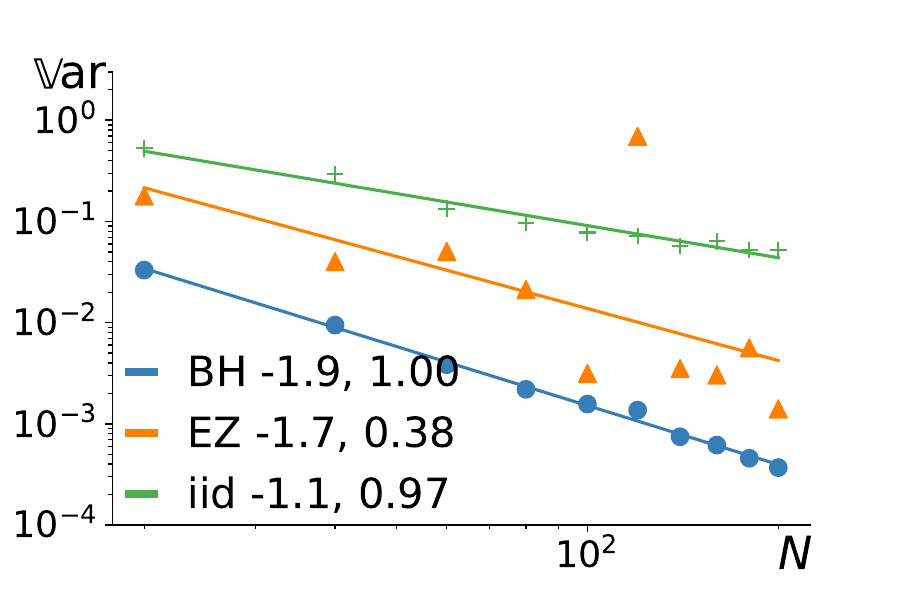}
            }
            \subfigure[$d=2$]{
            \label{fig:heaviside_value_2D}
            \includegraphics[width=\twofig]
                {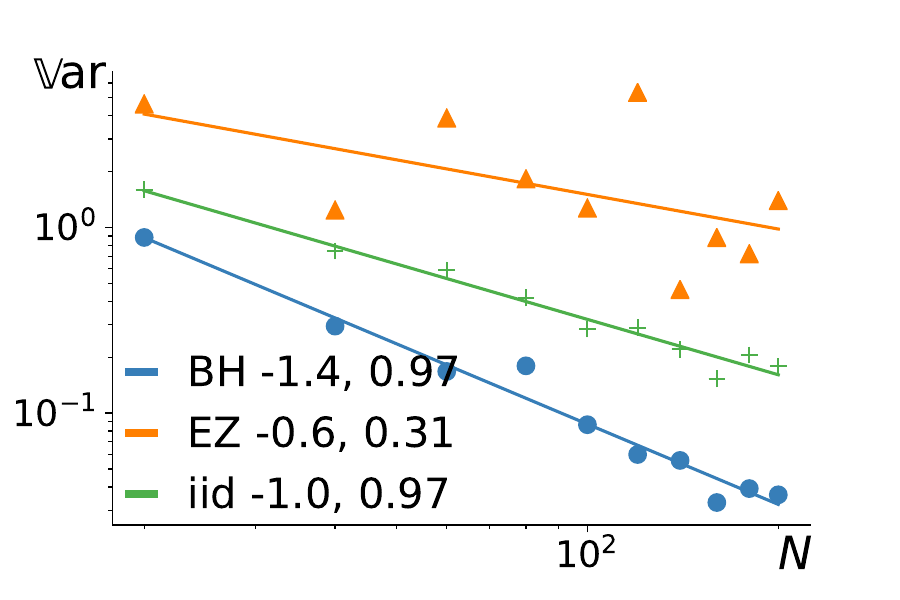}
            }\\
            \subfigure[$d=3$]{
            \label{fig:heaviside_value_3D}
            \includegraphics[width=\twofig]
                {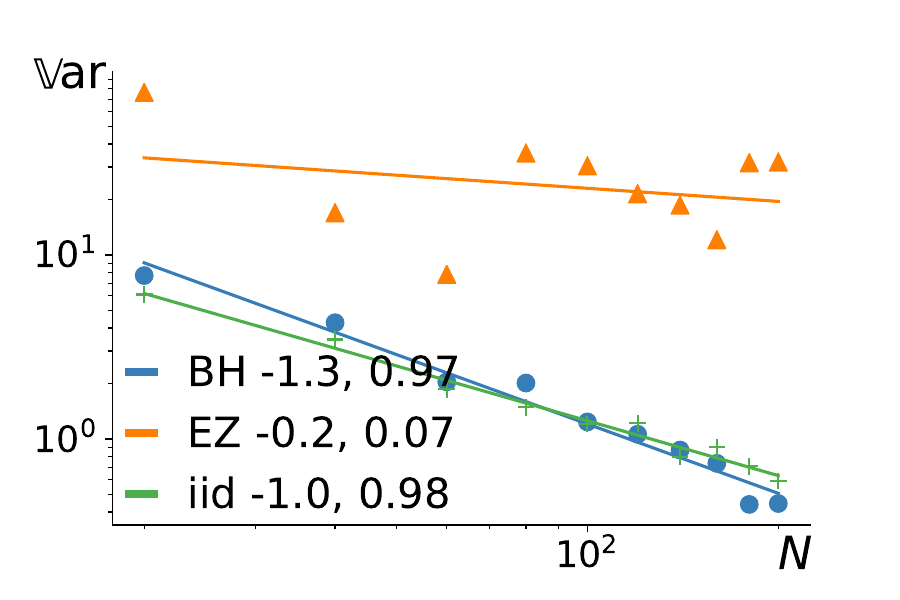}
            }
            \subfigure[$d=4$]{
            \label{fig:heaviside_value_4D}
            \includegraphics[width=\twofig]
                {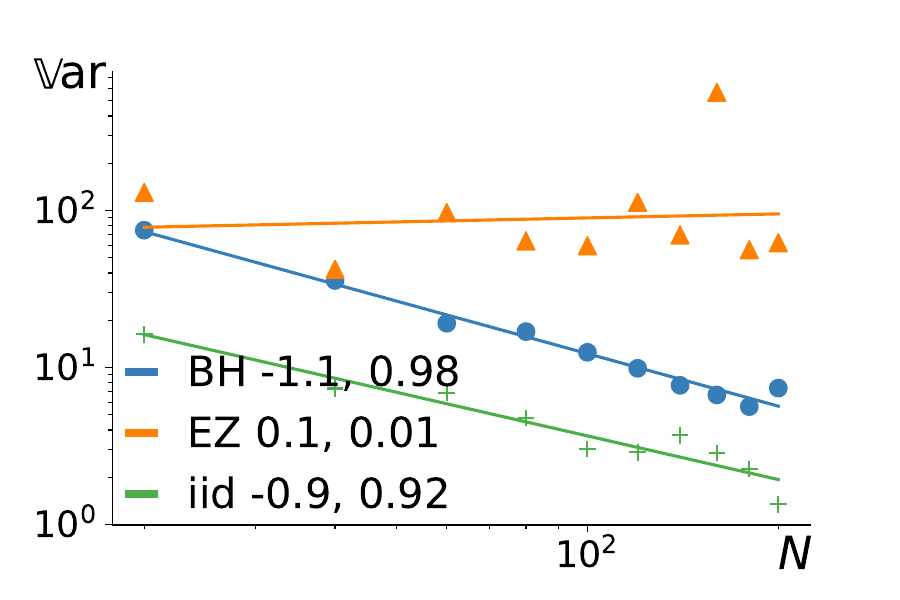}
            }
            \caption{Comparison of $\Iauth{BH}_N$ and $\Iauth{EZ}_N$ for Heaviside function, cf.\,\Secref{sub:further_experiments}.}
            \label{fig:app_heaviside}
        \end{figure*}

\newpage
\clearpage
\newpage
    \subsection{Integrating cosine}
    \label{sub:app_XP_cosine}

        We consider estimating the integral
        \begin{equation}
            \int_{[-1, 1]^d}
                \prod_{i=1}^{d}
                    \cos(\pi x^i)
                    (1-x^i)^{a^i} (1-x^i)^{b^i}
                \diff x^i
        \end{equation}
        where $a^1, b^1=-\frac12$ and $a^i, b^i$ i.i.d.\,uniformly in $[-\frac12, \frac12]$, using BH \eqref{eq:BH_estimator} and EZ \eqref{eq:EZ_estimator} estimators.

        Results are given in \Figref{fig:app_cosine}
        \begin{figure*}[ht]
            \subfigure[$d=1$]{
            \label{fig:cosine_1D}
            \includegraphics[width=\twofig]
                {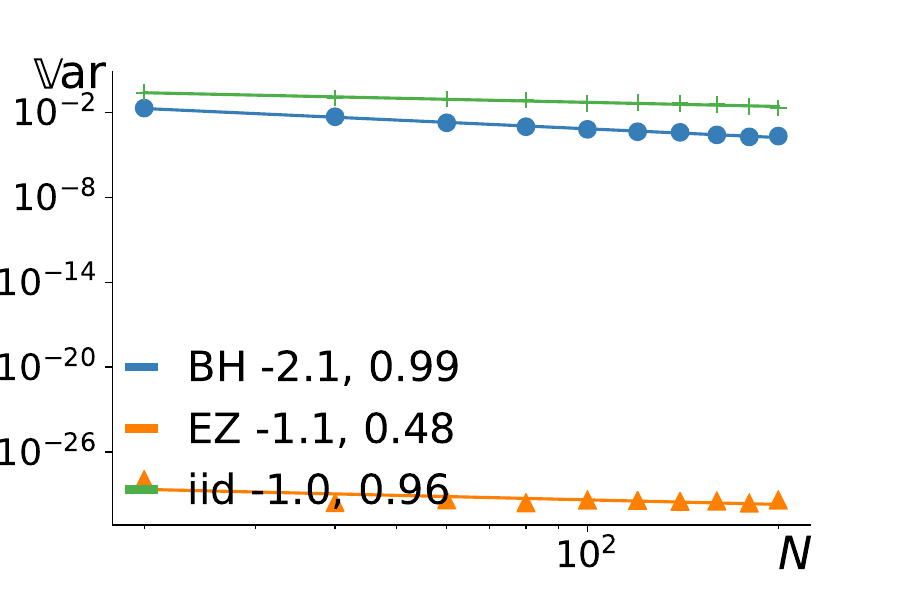}
            }
            \subfigure[$d=2$]{
            \label{fig:cosine_2D}
            \includegraphics[width=\twofig]
                {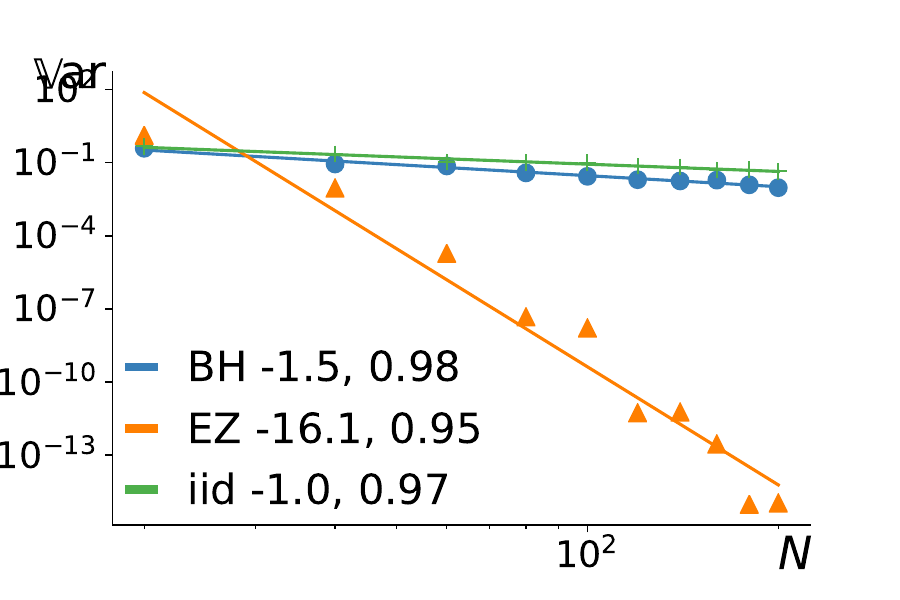}
            }\\
            \subfigure[$d=3$]{
            \label{fig:cosine_3D}
            \includegraphics[width=\twofig]
                {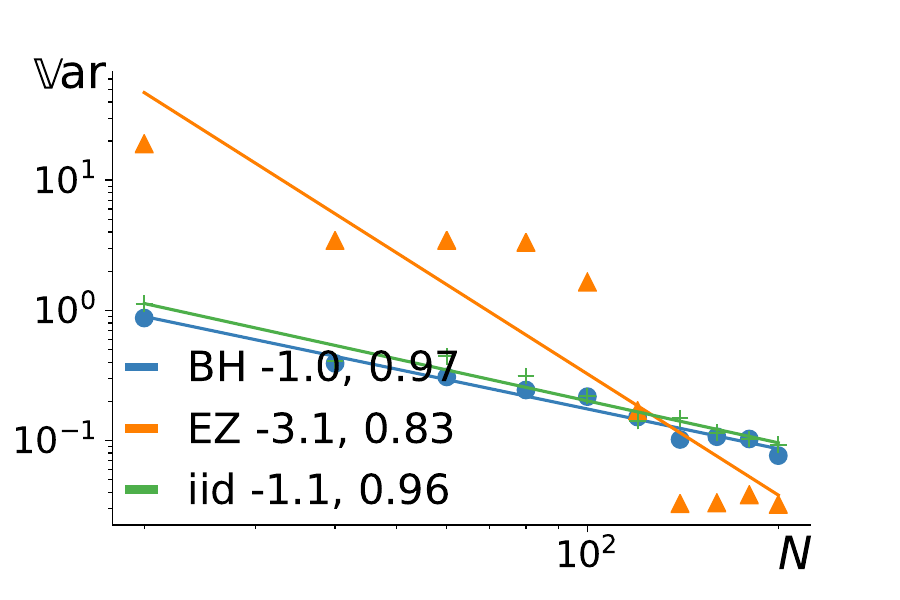}
            }
            \subfigure[$d=4$]{
            \label{fig:cosine_4D}
            \includegraphics[width=\twofig]
                {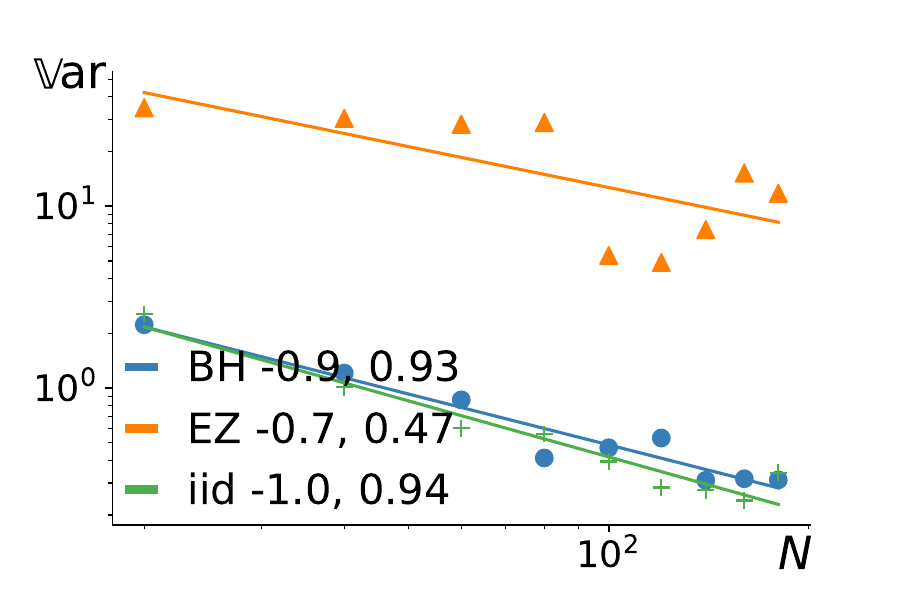}
            }
            \caption{Comparison of $\Iauth{BH}_N$ and $\Iauth{EZ}_N$ for cosine, cf.\,\Secref{sub:further_experiments}.}
            \label{fig:app_cosine}
        \end{figure*}

\newpage
\clearpage
\newpage
    \subsection{Integrating a mixture of smooth and non smooth functions} 
    \label{sub:integrating_mix}

        Let $f(x)=H(x)(\cos(\pi x) + \cos(2\pi x) + \sin(5\pi x))$.
        We consider estimating the integral
        \begin{equation}
            \int_{[-1, 1]^d}
                \prod_{i=1}^{d}
                    f(x^i)
                    (1-x^i)^{a^i} (1-x^i)^{b^i}
                \diff x^i
        \end{equation}
        where $a^1, b^1=-\frac12$ and $a^i, b^i$ i.i.d.\,uniformly in $[-\frac12, \frac12]$, using BH \eqref{eq:BH_estimator} and EZ \eqref{eq:EZ_estimator} estimators.

        \begin{figure*}[ht]
            \subfigure[$d=1$]{
            \label{fig:mix_1D}
            \includegraphics[width=\twofig]
                {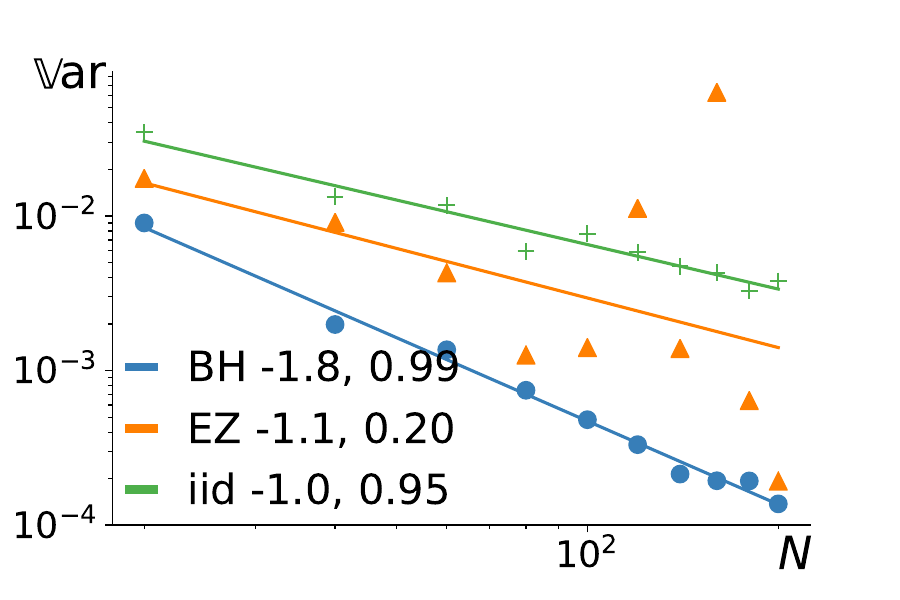}
            }
            \subfigure[$d=2$]{
            \label{fig:mix_2D}
            \includegraphics[width=\twofig]
                {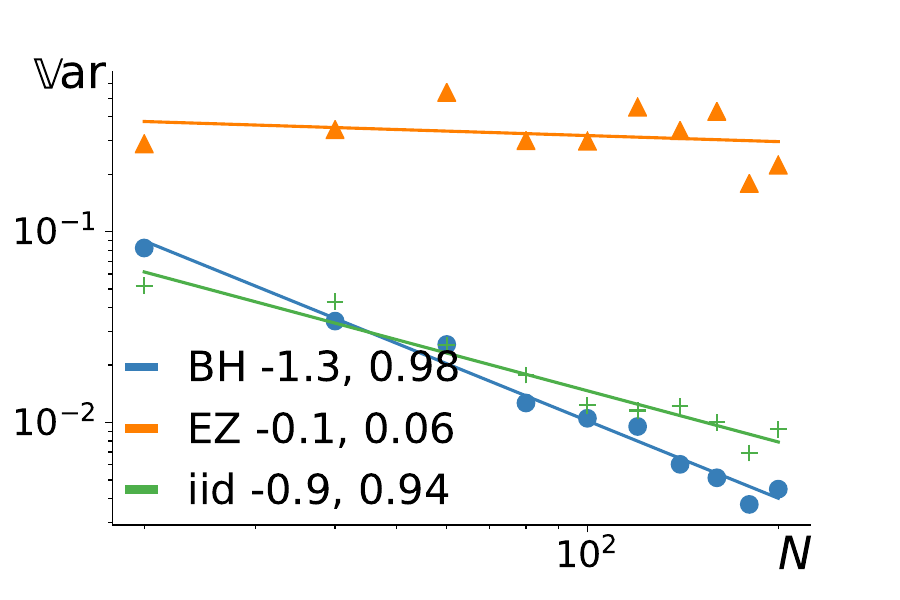}
            }\\
            \subfigure[$d=3$]{
            \label{fig:mix_3D}
            \includegraphics[width=\twofig]
                {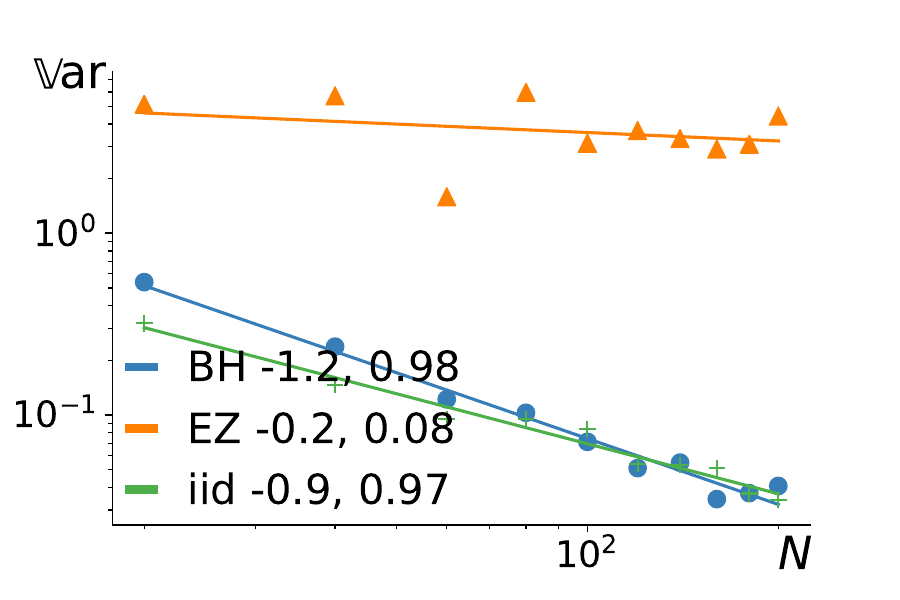}
            }
            \subfigure[$d=4$]{
            \label{fig:mix_4D}
            \includegraphics[width=\twofig]
                {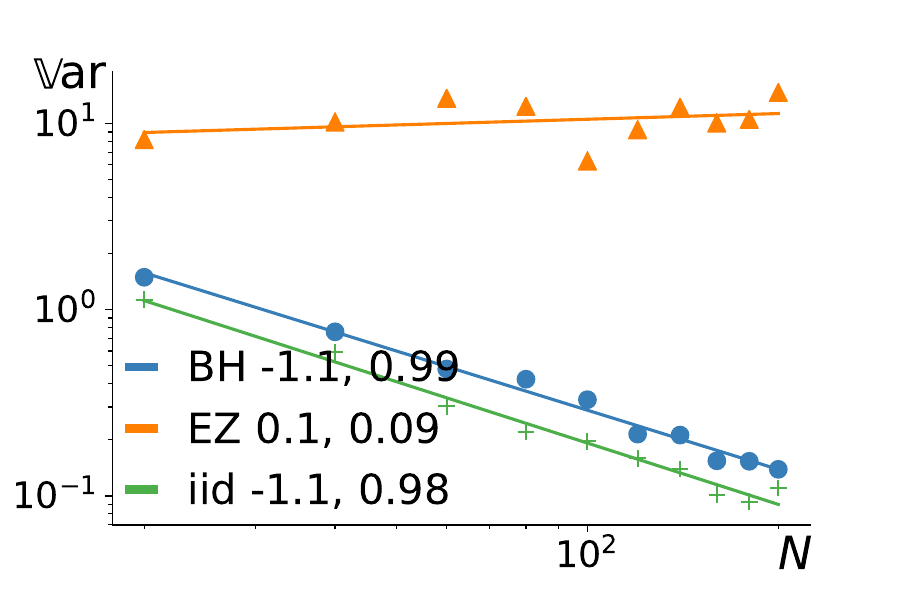}
            }
            \caption{Comparison of $\Iauth{BH}_N$ and $\Iauth{EZ}_N$, cf.\,\Secref{sub:further_experiments}.}
            \label{fig:app_mix}
        \end{figure*}

